\documentclass{article}
\usepackage[left=1in,right=1in,top=1in,bottom=1in,%
            footskip=.25in]{geometry}

\usepackage{float}
\usepackage{subfloat}

\usepackage{mathtools,amssymb,graphicx,subfigure,times,comment}
\usepackage{caption}
\usepackage{amsmath,amsthm}
\usepackage{bbm}
  {
      
  }
 
\usepackage{tcolorbox}
\usepackage[hidelinks]{hyperref}
\definecolor{darkred}{RGB}{150,0,0}
\definecolor{darkgreen}{RGB}{0,150,0}
\definecolor{darkblue}{RGB}{0,0,150}
\hypersetup{colorlinks=true, linkcolor=red, citecolor=darkgreen, urlcolor=darkblue}

\usepackage{breqn}

\newcommand{\thetao}{\overline{\theta}}

\newcommand{\go}{\overline{g}}
\newcommand{\etao}{\overline{\eta}}

\newcommand{\Scb}{\mathbf{\Sc}^{(n)}}
\newcommand{\ksit}{\widetilde{\ksi}}
\newcommand{\qt}{\widetilde{q}}

\newcommand{\Ell}{L}

\newcommand{\betah}{\widehat{\beta}}

\newcommand{\betabs}{\mathbf{\betab}_\star}
\newcommand{\yh}{\hat{y}}
\newcommand{\Rce}{\widehat{\mathcal{R}}_{\rm emp}}
\newcommand{\Rcex}{\mathcal{R}_{\rm excess}}
\newcommand{\Rct}{\mathcal{R}_{\rm train}}
\newcommand{\Ecsep}{\Ec_{\rm{sep}}}

\newcommand{\Rad}[1]{{\rm{Rad}}\left(#1\right)}

\newcommand{\mathleft}{\@fleqntrue\@mathmargin0pt}
\newcommand{\mathcenter}{\@fleqnfalse}

\newcommand{\coss}[2]{{\rm{cos}}(\,{#1}\,,\,{#2}\,)}

\newcommand{\prox}[3]{\mathrm{prox}_{{#1}}\left({#2};{#3}\right)}
\newcommand{\proxri}[3]{p_{#1}\left({#2},{#3}\right)}

\newcommand{\ellprox}[3]{L_{#1}\left({#2},{#3}\right)}

\newcommand{\ellp}{\ell^\prime}

\newcommand{\rhod}{f}

\newcommand{\R}{\mathbb{R}}

\newcommand{\al}{\alpha}

    \newcommand{\gammab}{\boldsymbol\gamma}

\newcommand{\betabh}{\boldsymbol{\widehat{\beta}}}

\newcommand{\rP}{\stackrel{{P}}{\rightarrow}}
\newcommand{\ras}{\xrightarrow[]{a.s.}}

\providecommand{\abs}[1]{\lvert#1\rvert}
\providecommand{\norm}[1]{\lVert#1\rVert}

\DeclarePairedDelimiterX{\inp}[2]{\langle}{\rangle}{#1, #2}

\newcommand{\Id}{\mathbf{I}}

\newcommand{\etab}{\boldsymbol{\eta}}

\newcommand{\ksi}{\xi}

\usepackage{dsfont}
\newcommand{\ind}[1]{{\mathds{1}}_{\{#1\}}}

\newcommand{\simiid}{\stackrel{\text{iid}}{\sim}}

\newcommand{\Pro}{\mathbb{P}}

\theoremstyle{theorem}
\newtheorem{propo}{Proposition}[section]
\newtheorem{thm}{Theorem}[section]
\newtheorem{lem}{Lemma}[section]
\newtheorem{cor}{Corollary}[section]
\theoremstyle{remark}
\newtheorem{remark}{Remark}
\theoremstyle{definition}

\newcommand{\eps}{\varepsilon}

\newcommand{\sign}{\mathrm{sign}}

\newcommand{\one}{\mathbf{1}}

\newcommand{\E}{\mathbb{E}}                    
\newcommand{\la}{{\lambda}}                     

\newcommand{\nn}{\notag}
\newcommand{\Rb}{\mathbb{R}}

\newcommand{\G}{\mathbf{G}}

\newcommand{\x}{\mathbf{x}}
\newcommand{\ub}{\mathbf{u}}

\newcommand{\w}{\mathbf{w}}

\newcommand{\g}{\mathbf{g}}
\newcommand{\vb}{\mathbf{v}}

\newcommand{\z}{\mathbf{z}}

\newcommand{\h}{\mathbf{h}}

\newcommand{\betab}{\boldsymbol{\beta}}

\newcommand{\Tc}{{\mathcal{T}}}
\newcommand{\Sc}{{\mathcal{S}}}
\newcommand{\Bc}{{\mathcal{B}}}

\newcommand{\Dc}{\mathcal{D}}

\newcommand{\Rc}{\mathcal{R}}
\newcommand{\Nn}{\mathcal{N}}

\newcommand{\Cc}{\mathcal{C}}
\newcommand{\Gc}{{\mathcal{G}}}

\newcommand{\Ec}{\mathcal{E}}

\newcommand{\beq}{\begin{equation}}
\newcommand{\eeq}{\end{equation}}
\newcommand{\bea}{\begin{align}}
\newcommand{\eea}{\end{align}}

\newcommand{\vp}{\vspace{4pt}}

\def\bea#1\eea{\begin{align}#1\end{align}}

\title{A model of Double Descent for \\High-dimensional Binary Linear Classification}

\author{Zeyu Deng \qquad Abla Kammoun \qquad Christos Thrampoulidis 
\thanks{Zeyu Deng and Christos Thrampoulidis are with the Electrical and Computer Engineering Department at the University of California, Santa Barbara, USA. Abla Kammoun is with the Electrical Engineering Department at King Abdullah University of Science and Technology, Saudi Arabia.}
}

\begin{document}

\maketitle

\begin{abstract}
We consider a model for logistic regression where only a subset of features of size $p$ is used for training a linear classifier over $n$ training samples. The classifier is obtained by running gradient descent (GD) on  logistic loss. For this model, we investigate the dependence of the classification error on the overparameterization ratio $\kappa=p/n$. First, building on known deterministic results on the implicit bias of GD, we uncover a phase-transition phenomenon for the case of Gaussian features: the classification error of GD is the same as that of the maximum-likelihood (ML) solution when $\kappa<\kappa_\star$, and that of the max-margin (SVM) solution when $\kappa>\kappa_\star$. Next, using the convex Gaussian min-max theorem (CGMT), we sharply characterize the performance of both the ML and the SVM solutions. Combining these results, we obtain curves that explicitly characterize the classification error for varying values of $\kappa$. The numerical results validate the theoretical predictions and unveil double-descent phenomena that complement similar recent findings in linear regression settings as well as empirical observations in more complex learning scenarios.
%
%
%
%
%
\end{abstract}


%
\section{Introduction}

\paragraph{Motivation.} Modern learning architectures are chosen such that the number of training parameters overly exceeds the size of the training set. Despite their increased complexity, such \emph{over-parametrized} architectures are known to generalize well in practice. In principle, this contradicts conventional  wisdom of the so-called approximation-generalization tradeoff. The latter suggests a U-shaped curve for the generalization error as a function of the number of parameters, over which the error initially decreases, but increases after some point due to overfitting. 
In contrast, several authors uncover a peculiar W-shaped curve for the generalization error of neural networks as a function of the number of parameters. After the ``classical" U-shaped curve, it is  seen that a second approximation-generalization tradeoff (hence, a second U-shaped curve) appears for large enough number of parameters. The authors of \cite{belkin2018understand,belkin2018reconciling,belkin2018overfitting} coin this the ``\emph{double-descent}" risk curve. The double-descent phenomenon has been 
%
demonstrated experimentally for decision trees, random features and two-layer neural networks (NN) in \cite{belkin2018reconciling} and, more recently, for deep NNs (including ResNets, standard CNNs and Transformers) in \cite{DDD}; see also \cite{spigler2019jamming,geiger2019scaling} for similar observations.

Recent efforts towards theoretically understanding the phenomenon of double descent 
focus on linear regression with Gaussian features \cite{hastie2019surprises,muthukumar2019harmless,belkin2019two,bartlett2019benign}; also \cite{belkin2018does,xu2019many,mei2019generalization,kobak2019optimal} for related efforts. These works investigate how the generalization error of gradient descent (GD) on square-loss depends on the overparameterization ratio $\kappa=p/n$, where $p$ number of features used for training are divided by the size  $n$ of the training set. On one hand, when $\kappa<1$, GD iterations converge to the {least-squares} solution for which the generalization performance is well-known \cite{hastie2019surprises}. On the other hand, in the overparameterized regime $\kappa>1$, GD iterations converge to the min-norm solution  for which the generalization performance is sharply characterized in \cite{hastie2019surprises,belkin2019two,muthukumar2019harmless} using random matrix-theory (RMT). Using these sharp asymptotics, these papers identify regimes (in terms of model parameters such as the signal-to-noise ratio (SNR)) for which a double-descent curves appear.

\paragraph{Contributions.}
This paper investigates the dependence of the classification error on the overparameterization ratio in binary linear classification with Gaussian features. In short, we obtain results that parallel previous studies for linear regression \cite{hastie2019surprises,belkin2019two,muthukumar2019harmless}. In more detail, we study gradient descent on logistic loss for two simple, yet popular, models: logistic model and gaussian mixtures (GM) model. Known results establish that GD iterations converge to either the support-vector machines (SVM) solution or the logistic maximum-likelihood (ML) solution, depending on whether the training data is separable or not. For the proposed learning model, we compute a phase-transition threshold $\kappa_\star\in(0,1/2)$ and show that when problem dimensions are large, then data are separable if and only if $\kappa>\kappa_\star$. Connecting the two, we redefine our task to that of studying the classification performance of the ML and SVM solutions. In contrast to linear regression, where the corresponding task can be accomplished using RMT, here we employ a machinery based on Guassian process inequalities. In particular, we obtain sharp asymptotics using the framework of the convex Gaussian min-max theorem (CGMT) \cite{StoLasso,COLT,Master}. Finally, we corroborate our theoretical findings with numerical simulations and build on the former to identify regimes where double descent occurs.  Figure \ref{fig:exp4str} contains a pictorial preview of our findings \footnote{\label{footnote:eR}The y-axis represents \emph{excess} risk $\Rcex:=\Rc-\Rc_{\rm best}$ defined as the the difference of the \emph{absolute} expected risk $\Rc$ minus the risk of the best linear classifier $\Rc_{\rm best}$ (see Eqn. \eqref{eq:eR}). Naturally both $\Rc$ and $\Rc_{\rm best}$ are decreasing functions of the SNR parameter $r$. However, $\Rc_{\rm best}$ is decreasing faster than $\Rc$. This explains why the  value of the excess risk $\Rcex$  is smaller for larger values of the signal strength $r$ in Figure \ref{fig:exp4str}. In particular, it holds $\Rc_{\rm best} = $ 0.133, 0.098 and 0.084 for $r = $ 5, 10 and 25, respectively. Contrast this to the values of the absolute risk  $\Rc = $ 0.434, 0.429 and 0.428 at (say) $\kappa = 0.05$. 
}. 

On a technical level, we provide useful extensions of the CGMT that allow: (a) studying feasibility questions (such as, when is the hard-margin SVM optimization feasible?); (b) raising certain boundedness assumption on the CGMT that were previously circumvented on a case-by-case analysis; (c) establishing almost-sure convergence results (rather than ``in probability"). While our motivation comes from the analysis of the hard-margin SVM, we believe that these extensions can be of independent interest offering a principled way for future statistical studies of related optimization-based inference algorithms. We present these extensions of the CGMT as self-contained theorems in Appendix \ref{sec:CGMT_new}.



\begin{figure*}[t]
\begin{center}
\includegraphics[width=.9\textwidth]{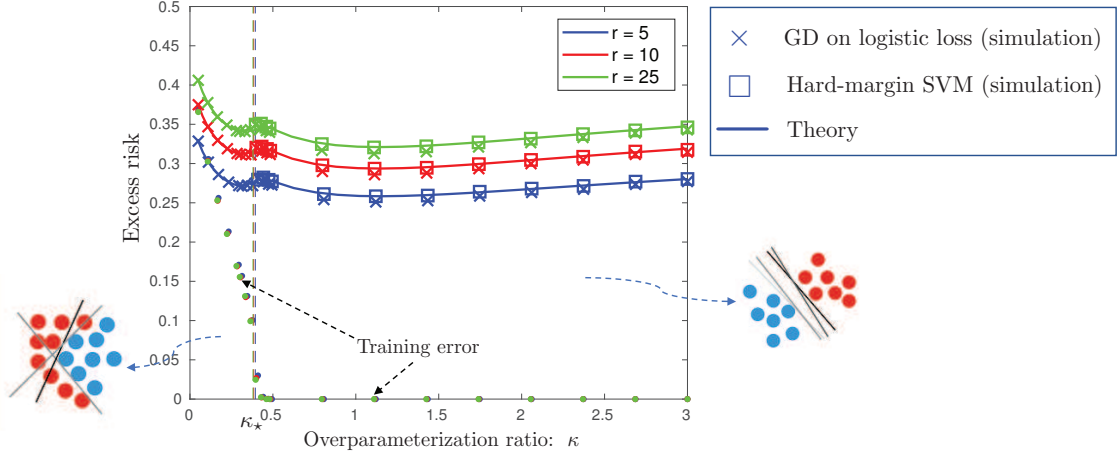}
\end{center}
\caption{Excess risk 
as a function of the overparameterization ratio $\kappa$ for binary logistic regression under a polynomial feature selection rule. Our asymptotic predictions (theory) match the simulation results for gradient descent (GD) on logistic loss and for hard-margin support-vector machines (SVM). The theory also predicts a sharp phase-transition phenomenon: data is linearly separable (eqv. training error is zero) with high-probability iff $\kappa>\kappa_\star$.  The threshold $\kappa_\star$ is depicted by dashed lines (different for different values of the SNR parameter $r$, but  almost indistinguishable due to scaling of the figure).  Observe that the risk curve experiences a ``double-descent". As such, it has two local minima: one in the underparameterized and one in the overparameterized regime, respectively. Furthermore, the global minimizer of the excess risk is achieved in the latter regime in which data are linearly separable. The minimizer corresponds to the optimal number of features that should be chosen during training for the model under consideration. Please refer to Section \ref{sec:num} for further details on the chosen simulation parameters.}
\label{fig:exp4str}
\end{figure*}

\subsection{Related works}\label{sec:rel}
Our results on the performance of logistic ML and SVM fit in the rapidly growing recent literature on \emph{sharp} asymptotics of (possibly non-smooth) convex optimization-based estimators; \cite{DMM,montanariLasso,StoLasso,OTH13,COLT,montanari13,Master,karoui15} and many references therein.  Most of these works study linear regression models, for which the CGMT framework has been shown to be powerful and applicable under several variations; see for example \cite{Master,TSP18,celentano2019fundamental,hu2019asymptotics,javanmard2020precise}. We also mention a parallel line of work that uses an alternative analysis framework based on the Approximate Message Passing (AMP) (e.g \cite{DMM,montanariLasso,wang2019does,bu2019algorithmic,rush2018finite,rangan2011generalized}); a detailed comparison between the two frameworks is out of the scope of this paper.

In contrast to the above mentioned studies of linear regression models, our results hold for binary classification. To the best of our knowledge, the first asymptotically sharp analysis of convex-based methods for binary models is due to \cite{NIPS}, which computes the squared-error of regularized least-squares under nonlinearities.  The simple, yet central idea, which allows for an application of the CGMT in this setting, is a certain ``projection trick" inspired by \cite{Ver} (also used in \cite{genzel2016high}). (A similar idea was also applied in \cite{PhaseLamp} for the analysis of the PhaseMax algorithm for phase-retrieval.) However, \cite{NIPS} is not focused on binary classification.

 In this context, the first relevant results are due to \cite{candes2018phase,sur2019modern}, who perform a detailed study of logistic regression under the logistic data model using the approximate kinematic formula \cite{TroppEdge} and the AMP. Soon after the works of Candes and Sur, \cite{salehi2019impact} and \cite{taheri2019sharp,Hossein2020} have used the CGMT to analyze the statistical properties of regularized logistic regression and of general convex empirical-risk-minimization (ERM), respectively. The paper \cite{svm_abla}, while still using the CGMT, focuses on a discriminative (rather than a generative) data model and studies the hard-margin SVM classifier. While there are a few other recent works on sharp asymptotics of binary linear classification \cite{mai2019large,logistic_regression,huang2017asymptotic}, the papers \cite{svm_abla,Hossein2020,salehi2019impact,candes2018phase,sur2019modern} are the most closely related to this work.  
 
Compared to these, our contribution differs as follows. First, we examine risk curves for all possible values of the overparameterization ratio, which gives rise to the phase-transition of Proposition \ref{propo:PT} and the double-descent phenomena investigated in Section \ref{sec:num}. Second, we propose and study the misspecified model for classification of Section \ref{sec:train_model}. Third, we derive results that treat both the logistic generative model and the GM discriminative model in a unifying way. On a technical level: (i) we properly apply the CGMT to study feasibility of the hard-margin SVM (previous results focus on feasible optimization problems); (ii) we establish convergence results that hold almost surely (the CGMT establishes convergence in probability); (iii) we prove existence and uniqueness to the solution of the equations derived in Propositions \ref{propo:ML} and \ref{propo:SVM}. Importantly, compared to \cite{svm_abla}, which also studies feasibility and almost-sure convergence of hard-margin SVM, we formulate our ideas as stand-alone results and prove them in more general settings such that they can be used beyond the context of our paper.


An early conference version of this paper appears in \cite{ICASSP2020}. After the initial submission and while preparing a long version of the paper we became aware of the parallel work \cite{montanari2019generalization}. Very similar to our setting, \cite{montanari2019generalization} investigates the classification performance of hard-margin SVM for binary linear classification under generative data models. In contrast to \cite{montanari2019generalization}, we also analyze the performance of logistic ML. This allows us to generate risk curves for all positive values of the overparametrizatio ratio $\kappa>0$ and explicitly demonstrate and study the presence of double-descent phenomena. We also extend our study to the discriminative Gaussian mixture model. On the other hand, it is worth mentioning that the authors of \cite{montanari2019generalization} further derive asymptotic predictions for correlated (not necessarily iid) Gaussian features. While both papers build their asymptotic analysis on the CGMT, there are several differences when it comes to the technical analysis of the Auxiliary Optimization (AO) problem of the CGMT \footnote{In this paper, we analyze the formulation of the hard-margin SVM in \eqref{eq:SVM}. In contrast, the authors of \cite{montanari2019generalization} consider a re-formulation of the SVM (see \cite[Eqn.~15.1]{shalev2014understanding}) that even-though is equivalent, it leads to a different AO problem.}. Overall, we believe that both contributions are of independent interest. 

We end this discussion by referring to a few more related works that have appeared since the early version of this paper \cite{ICASSP2020} and of \cite{montanari2019generalization}. In \cite{kini2020analytic}, the authors adapt the setting proposed here, but they investigate DD under gradient descent on square loss (rather than on logistic loss). Combined with the results of our paper, the authors demonstrate the impact of loss on the features of DD. Another follow-up work \cite{mignacco2020role} investigates the impact of regularization on DD curves for the GM model . Finally, \cite{liang2020precise} studies the statistical properties of min-$\ell_1$-norm interpolating classifiers. 




%

%
\section{Learning model}\label{sec:model}


\subsection{Data generation models}\label{sec:data}
We study supervised binary classification under  two popular data models, 
\begin{itemize}
\item a \emph{discriminative} model: mixtures of Gaussian.
\item a \emph{generative} model: logistic regression.
\end{itemize}
Let $\x\in\R^d$ denote the feature vector and $y\in\{\pm 1\}$ denote the class label.

\vp
\noindent\textbf{Logistic model.}~First, we consider a discriminative approach which models the marginal probability $p(y|\x)$ as follows:
$$
y =\begin{cases}
+1, &\text{w.p.}~~\rhod(\x^T\etab_0),\\
-1, &\text{w.p.}~~1-\rhod(\x^T\etab_0),
\end{cases}
$$
where $\etab_0\in\Rb^d$ is the unknown weight (or, regressor) vector and  $f(t)$ is the sigmoid function: 
$$
 \rhod(t):=(1+e^{-t})^{-1}.
 $$
\noindent 
Throughout, we assume iid Gaussian feature vectors 
$\x\sim\Nn(0,\Id_d).$
For compactness let $\Rad{p}$ denote a symmetric Bernoulli distribution with probability $p$ for the value $+1$ and probability $1-p$ for the value $-1$. We summarize the logistic model with Gaussian features:
\bea\label{eq:yi_log}
y\sim\Rad{\rhod(\x^T\etab_0)},~~~~\x\sim\Nn(0,\Id_d).
\eea


\vp
\noindent\textbf{Gaussian mixtures (GM) model.}~A common choice for the generative case is to model the class-conditional densities $p(\x|y)$ with Gaussians (e.g., \cite[Sec.~3.1]{williams2006gaussian}).  Specifically, each data point belongs to one of two classes $\Cc_{\pm1}$ with probabilities $\pi_{\pm1}$ such that $\pi_{+1} + \pi_{-1}=1$. If it comes from class $\Cc_{\pm1}$, then the feature vector $\x\in\R^d$ is an iid Gaussian vector with mean $\pm\etab_0\in\R^d$ and the response variable $y$ takes the value of the class label $\pm1$:
\bea\label{eq:yi_GM}
y = \pm 1 \Leftrightarrow \x\sim\Nn\left(\pm\etab_0,\Id_d\right).
\eea


\subsection{Feature selection model for training data}\label{sec:train_model}
 During training, we assume access to $n$ data points generated according to either \eqref{eq:yi_log} or \eqref{eq:yi_GM}. 
We allow for the possibility that only a subset $\Sc\subset[d]$ of the total number of features is known during training. Concretely, for each feature vector $\x_i,~i\in[n]$, the learner only knowns a sub-vector $\w_i:=\x_i(\Sc):=\{\x_i(j),~j\in\Sc\},$
for a chosen set $\Sc\subset[d]$.
We denote the size of the \emph{known} feature sub-vector as $p:=|\Sc|$. We will often refer to $p$ as the \emph{model size}. Onwards, we choose $\Sc=\{1,2,\ldots,p\}$ \footnote{This assumption is without loss of generality for our asymptotic model specified in Section \ref{sec:setting}.}, i.e., select the features sequentially in the order in which they appear.

 Clearly, $1\leq p\leq d$. Overall, the training set $\Tc(\Sc)$ consists of $n$ data pairs:
\begin{align}\label{eq:TS}
&\quad\Tc(\Sc):=\left\{(\w_i,y_i), i=1,\dots,n\right\},\qquad \\&\w_i=\x_i(\Sc)\in\R^p~\text{ where }~(y_i,\x_i)\sim\eqref{eq:yi_log}~\text{or}~\eqref{eq:yi_GM}\nn.
\end{align}
When clear from context, we simply write 
$$
\Tc\sim\eqref{eq:yi_log}\quad\text{or}\quad\Tc\sim\eqref{eq:yi_GM},
$$
if the training set $\Tc$ is generated according to \eqref{eq:yi_log} or \eqref{eq:yi_GM}, respectively.




\subsection{Classification rule}\label{sec:class}
After choosing a model size $p=1,\ldots,d$, the learner uses  the training set $\Tc(\Sc)$ to obtain an estimate $\betabs\in\R^p$ of the first $p$ entries of the weight vector $\etab_0\in\R^d$. Then, for a newly generated sample $(\x,y)$ (and $\w:=\x(\Sc)$), 
she forms a \emph{linear classifier} and decides a label $\yh$ for the new sample as:
\begin{align}\label{eq:yhat}
\yh = \sign(\w^T\betabs).
\end{align}
%
The estimate $\betabs$ is obtained by minimizing the empirical risk:
\bea\label{eq:Remph}
\Rce(\betab):=\frac{1}{n}\sum_{i=1}^n\ell(y_i\w_i^T\betab),
\eea
 for certain loss function $\ell:\R\rightarrow\R$. A traditional way to minimize $\Rce$ is by constructing gradient descent (GD) iterations $\betab_{(k)},~k\geq0$ via $$
\betab_{(k+1)} = \betab_{(k)} - \eta_k\nabla\Rce(\betab_{(k)}),
$$
for step-size $\eta_k>0$ and arbitrary $\betab_{(0)}$. We run GD until convergence and set
$$\betabs:=\lim_{k\rightarrow\infty}\betab_{(k)}.$$

\noindent In this paper, we focus on \emph{logistic} loss in \eqref{eq:Remph}, i.e.,  $\ell(t):=\log\big(1+e^{-t}\big).$

%
%

\vp
\noindent\textbf{Classification error.} For a new sample $(\x,y)$ we measure test error of $\hat{y}$ by the expected \emph{risk} (or, \emph{test error}):
\begin{align}\label{eq:risk}
\Rc(\betabs) := \E\left[\mathds{1}\big(\yh\neq y \big)\right].
\end{align}
Here, $\mathds{1}(\Ec)$ denotes the indicator function of an event $\Ec$. Also, the expectation here is over the distribution of the \emph{new} data sample $(\x,y)$ generated according to either \eqref{eq:yi_log}~\text{or}~\eqref{eq:yi_GM}. In particular, note that $\Rc(\betabs)$ is a function of the training set $\Tc$. Along these lines, we also define the  \emph{excess risk}:
\bea\label{eq:eR}
\Rcex(\betabs) := \Rc(\betabs) - \Rc(\etab_0),
\eea
where $\Rc(\etab_0):=\E\left[\mathds{1}\big(y\neq \sign(\x^T\etab_0) \big)\right]$ is the risk of the best linear classifier that assumes knowledge of the entire feature vector $\x_i$ and of $\etab_0$ \cite{bartlett2006convexity}.
Finally, we further consider the cosine similarity between the estimate $\betabs$ and $\etab_0$ as a measure of evaluating estimation performance: 
\bea\label{eq:cos}
\Cc(\betabs):=\coss{\betabs}{\etab_0}=\frac{\inp{\betabs}{\betab_0}}{\|\betabs\|_2\|\etab_0\|_2}.
\eea

\vp
\noindent\textbf{Training error.} 
\noindent The training error of $\betab_\star$ is given by
$$
\Rct(\betabs):=\frac{1}{n}\sum_{i=1}^n \mathds{1}\big(\yh_i\neq y_i \big) = \frac{1}{n}\sum_{i=1}^n \mathds{1}\big((\w_i^T\betabs) y_i>0 \big).
$$


\subsection{Implicit bias of GD}\label{sec:GD}
\noindent Recent literature fully characterizes the converging behavior of GD iterations for logistic loss \cite{pmlr-v99-ji19a,soudry2018implicit,telgarsky2013margins}. There are two regimes of interest: 
\begin{itemize}
\item[(i)] When data are such that $\Rce(\betab)$ is strongly convex, then standard tools show that $\betab_{(k)}$ converges to the unique bounded minimizer of $\Rce(\betab)$.
\item[(ii)] When data are linearly separable, then the normalized iterates $\betab_{(k)}/\|\betab_{(k)}\|_2$ converge to the max-margin solution.
\end{itemize}

\noindent These deterministic properties guide our approach towards studying the classification error of the converging point of GD for logistic loss in Section \ref{sec:theory}. Specifically, instead of studying GD iterations directly, we study the classification performance of the max-margin classifier and of the minimum of the empirical logistic risk. We formalize these ideas next.

\subsubsection{Separable data}
The training set is separable if and only if there exists a linear classifier achieving zero training error, i.e., 
$
\exists \betab~:~y_i\w_i^T\betab\geq 1,~\forall i\in[n].
$ When data are separable,
 the normalized iterations of GD for logistic loss converge to the maximum margin classfier \cite{soudry2018implicit,pmlr-v99-ji19a}. Precisely,  
 for $k\rightarrow\infty$ it holds $\Big\|\frac{\betab_{(k)}}{\|\betab_{(k)}\|_2}-\frac{\betabh}{\|\betabh\|_2}\Big\|_2\rightarrow 0,$
 where $\betabh$ is the solution to the {hard-margin SVM}: 
\bea\label{eq:SVM}
\betabh=\arg\min_{\betab}~\|\betab\|_2~~\text{sub. to}~ y_i\w_i^T\betab\geq 1.
\eea

\subsubsection{Non-separable data}When the separability condition does \emph{not} hold, then $\Rce(\beta)$ is coercive. Thus, its sub-level sets are closed and GD iterations will converge \cite{pmlr-v99-ji19a} to the minimizer $\betabh$ of the empirical loss: 
\bea\label{eq:ML}
\betabh=\arg\min_{\betab}~\frac{1}{n}\sum_{i=1}^{n}\ell(y_i\w_i^T\betab),
\eea
where, $\ell(t)=\log(1+e^{-t})$. For the data model in \eqref{eq:yi_log}, $\betabh$ is the {maximum-likelihood (ML)} estimator; indeed, $\ell^\prime(t) = -\rhod(t)$. Thus, we often refer to \eqref{eq:ML} as the ML estimator.



%
\section{Sharp Asymptotics}\label{sec:theory}

We present sharp asymptotic formulae for the classification performance of GD iterations for the logistic loss under the learning model of Sections \ref{sec:data} and \ref{sec:train_model} and the linear asymptotic setting presented in Section \ref{sec:setting}. Our analysis sharply predicts the limiting behavior of the test error and of the cosine similarity as a function of the model size, i.e., the number of parameters $p$ used for training.

 The first key step of our analysis, makes use of the convergence results discussed in Section \ref{sec:GD} that relate GD for logistic loss to the convex programs \eqref{eq:ML} and  \eqref{eq:SVM} depending on whether the training data $\Tc$ are linearly separable. Specifically, we show in Section \ref{sec:PT} that, in the linear asymptotic regime with Gaussian features, the convergence behavior of GD as a function of the model size, undergoes a sharp phase-transition for both  data models (logistic and Gaussian-mixtures). The boundary of the phase-transition separates the so-called under- and over-parameterized regimes. Thus, for each one of the two corresponding regimes, GD converges to an associated convex program (cf. \eqref{eq:ML} and \eqref{eq:SVM}, respectively). The second key step of our analysis, uses the Convex Gaussian min-max Theorem (CGMT) \cite{COLT,Master} to sharply evaluate the classification performance of these convex programs in Section \ref{sec:asy}. Specifically, we build on a useful extension of the CGMT, which we present in Appendix \ref{sec:CGMT_new}. The proofs of the results presented in this sections are deferred to the Appendix.



\subsection{Asymptotic setting}\label{sec:setting}

\noindent Recall the following notation:

\vp
$\bullet$~ $d$: dimension of the ambient space,\\
\indent$\bullet$~ $n$: size of the training set,\\
\indent$\bullet$~ $p$:  number of parameters used in training (aka {model size}).\\

\noindent Our asymptotic results hold in a \emph{linear asymptotic regime} where $n,d,p\rightarrow+\infty$ such that 
\bea\label{eq:linear}
{d/n}\rightarrow \zeta\geq 1 \quad \text{and}\quad  {p/n}\rightarrow \kappa\in(0,\zeta]. 
\eea
 
We call $\kappa$ the \emph{overparameterization ratio} as it determines the ratio of parameters to be trained divided by the size of the training set. To quantify its effect on the test error, we decompose the feature vector $\x_i\in\R^d$ to its \emph{known} part $\w_i\in\R^p$ and to its unknown part $\z_i$:
$
\x_i:=[\w_i^T,\z_i^T]^T.\nn
$
Then, we let $\betab_0\in\R^p$ (resp., $\gammab_0\in\R^{d-p}$) denote the vector of weight coefficients corresponding to the known (resp., unknown) features such that 
$
\etab_0:=[\betab_0^T,\gammab_0^T]^T.
$
 In this notation, we study a sequence of problems of increasing dimensions $(n,d,p)$ as in \eqref{eq:linear} that further satisfy:
\bea
&\|\etab_0\|_2 \ras r\,, \label{eq:strengths} \\
&\|\betab_0\|_2 \ras s:=s(\kappa)\, ~~\text{ and }~~~ \|\gammab_0\|_2 \ras \sigma:=\sigma(\kappa) = \sqrt{r^2-s^2(\kappa)}. \nn
\eea
%
%
Above and throughout the paper, for a sequence of random variables $\mathcal{X}_{n,p,d}$ that converges almost-surely (resp., in probability) to a constant $c$ in the limit of \eqref{eq:linear}, we write $\mathcal{X}_{n,p,d}\ras c$ (resp. $\mathcal{X}_{n,p,d}\rP c$).

The parameters $s\in[0,r]$ and $\sigma$ in \eqref{eq:strengths} can be thought of as the useful signal strength and the noise strength, respectively. Our notation specifies that $s(\kappa)$ (hence also, $\sigma(\kappa)$) is a function of $\kappa$. We are interested in functions $s(\kappa)$ that are increasing in $\kappa$ such that the signal strength increases as more features enter into the training model; see Section~\ref{sec:explicit} for explicit parameterizations. For each triplet $(n,d,p)$ in the sequence of problems that we consider, the corresponding training set $\Tc=\Tc(\{1,2,\ldots,p\})$ follows \eqref{eq:TS}.

\subsection{Regimes of learning: Phase-transition}\label{sec:PT}
As discussed in Section \ref{sec:GD}, the behavior of GD changes depending on whether the training data is separable or not. The following proposition establishes a sharp phase-transition characterizing the separability of the training data under the data model of Section \ref{sec:model}.  

We reserve the following notation for random variables $G,H, Z$ and $Y_{r,s}$
\bea
H,G,Z&\simiid\Nn(0,1),\nn\\
\text{and}\quad Y_{r,s}&\sim\Rad{\rhod(s\,G+\sqrt{r^2-s^2}\,Z)},\label{eq:HGZ}
\eea for $s$ and $r$ as defined in \eqref{eq:strengths}.  We will often drop the subscripts $r$ and $s$ from $Y_{r,s}$ and simply write $Y$ when their values are clear from context. All expectations and probabilities are with respect to the randomness of $H,G,Z$. Also, let $(x)_{-}=\min\{x,0\}$.  


\begin{propo}[Phase transition]\label{propo:PT} 
Fix total signal strength $r$ and let $s(\kappa)\in(0,r]$ be an increasing function of $\kappa\in(0,\zeta]$. For a training set $\Tc$ that is generated by either of the two models in \eqref{eq:yi_log} or \eqref{eq:yi_GM} such that  \eqref{eq:strengths} is satisfied, consider the event 
$$
\Ecsep := \left\{ ~\exists \betab\in\R^p~:~y_i\w_i^T\betab\geq 1,~\forall i\in[n]~ \right\},
$$
under which the training data is separable. Recall from \eqref{eq:TS} that for all $i\in[n]$, $\w_i = \x_i\big(\{1,\ldots,p\}\big)$ are the $p$ (out of $d$) features that are used for training. 
Recall the notation in  \eqref{eq:HGZ} and define a random variable $V_{r,s(\kappa)}$ depending on the data generation model as follows
\bea\label{eq:V}
V_{r,s(\kappa)}:= \begin{cases}
G\,Y_{r,s(\kappa)} & ,\,\text{if}~~~ \Tc \sim \eqref{eq:yi_log},\\
G+s(\kappa) & ,\,\text{if}~~~ \Tc \sim \eqref{eq:yi_GM}.
\end{cases}
\eea
Further, define the following threshold function $g:(0,\zeta]\rightarrow\R_{+}$ that also depends on the data generation model:
\bea\label{eq:threshold_func}
g(\kappa):= 
\min_{t\in\R}\E\left(H+t\,V_{r,s(\kappa)}\right)_{-}^2.
\eea
Let $\kappa_\star\in[0,1/2]$ be the unique solution to the equation $g(\kappa)=\kappa$. 
Then, the following holds regarding $\Ecsep$:
\bea
\nn \kappa>\kappa_\star~&\Rightarrow~\Pr\left(~\text{the event $\Ecsep$ holds for all large $n$}~\right)=1,\\
\nn \kappa<\kappa_\star~&\Rightarrow~\Pr\left(~\text{the event $\Ecsep$ holds for all large $n$}~\right)=0.
\eea
\end{propo}

 Put in words: the training data is separable iff $\kappa>\kappa_\star$. When this is the case, then the training error $\Rct$ can be driven to zero and we are in the \emph{interpolating regime}. 
In contrast, the training error is non-vanishing for smaller values of $\kappa$. 

In the case of the GM model, the threshold function $g(\kappa)$ takes a simple form as follows.
First, assume $\Tc\sim\eqref{eq:yi_GM}$ and substitute the value of $V=G+s$ in \eqref{eq:threshold_func}. Then, using the fact that $G$ and $H$ are independent Gaussians the threshold function simplifies to:
\begin{align}
g(\kappa) &= 
\min_{t\in\R}\E\left(G\sqrt{1+t^2}+t\,s\right)_{-}^2 \label{eq:g_GM}
\\
&=\min_{t\in\R}\Big\lbrace \left(1 + t^2 + t^2s^2\right) \cdot Q\left(\frac{ts}{\sqrt{1+t^2}}\right) -ts\sqrt{1+t^2}\cdot{\psi}\Big(-\frac{ts}{\sqrt{1+t^2}}\Big)\,\Big\rbrace. \nn
\end{align}
where $\psi(t)=\frac{1}{\sqrt{2\pi}}e^{-t^2/2}$ and $Q(x)=\int_{x}^{\infty}\psi(t)\mathrm{d}t$ are the density and tail function of the standard normal distribution, respectively. Recall from \eqref{eq:strengths} that the signal strength $s$ is a function of $\kappa$, which explains why $g$ is a function of $\kappa$.

The proposition above is an extension of the phase-transition result by Candes and Sur  \cite{candes2018phase} for the noiseless logistic model. Specifically, we extend their result to the noisy setting (to accommodate for the feature selection model in Section \ref{sec:train_model}) as well as to the Gaussian mixtures model.  Our analysis approach and proof technique are also very different to  \cite{candes2018phase}. Both approaches bare intimate connections and are motivated by corresponding results on sharp phase-transitions in compressed sensing \cite{Sto,Cha,TroppEdge}.  On the one hand, the proof in \cite{candes2018phase} is based on a purely geometric argument and specifically on an appropriate application of the approximate kinematic formula of \cite{TroppEdge}. This leads to a non-asymptotic characterization of the phase-transition. On the other hand, in this paper, we follow an approach that is based on Gaussian process inequalities \cite{rudelson2006sparse,Sto,Cha,StoLasso,OTH13} and specifically on the Convex Gaussian min-max Theorem (CGMT) \cite{COLT}. Our idea is exploiting the fact that data is linearly separable iff the solution to hard-margin SVM \eqref{eq:SVM} is bounded. Based on this, we are able to appropriately use the CGMT in order to show that the minimization \eqref{eq:SVM} is bounded almost surely iff $\kappa>\kappa_\star$. Previous applications of the CGMT have almost entirely focused on feasible minimization problems. In this work, we formulate Theorem \ref{th:feasibility} in Appendix \ref{sec:CGMT_new} as a useful corollary of the CGMT to further allow the study of feasibility questions. We anticipate that the theorem proves useful in other settings beyond the specifics considered in this paper. The detailed proof of Proposition \ref{propo:PT} is given in Appendix \ref{sec:proof_PT}.   

\subsection{High-dimensional asymptotics}\label{sec:asy}

For each increasing triplet $(n,d,p)$ and sequence of problem instances following the asymptotic setting of Section \ref{sec:setting}, let $\betabh\in\R^d$ be the sequence of vectors of increasing dimension corresponding to the convergence point of GD for logistic loss. The Propositions \ref{propo:ML} and \ref{propo:SVM} below characterize the asymptotic value of the cosine similarity and of the classification error of the sequence of $\betabh$'s for the under- and over-parameterized regimes, respectively. Recall that, that this task is equivalent to characterizing the statistical properties of the two convex optimization-based estimators \eqref{eq:ML} and \eqref{eq:SVM}.


Recall that we use $\ras$ to denote almost sure convergence in the limit of \eqref{eq:linear}. Also, we denote the proximal operator of the logistic loss as follows, 
$$
\prox{\ell}{x}{\la}:= \arg\min_{v}\frac{1}{2\la}(x-v)^2+\ell(v).
$$

%

\subsubsection{Non-separable data}
\begin{propo}[Asymptotics of ML]\label{propo:ML}
Fix total signal strength $r$ and overparameterization ratio $\kappa<\kappa_\star$. Denote $s=s(\kappa)\in(0,r]$. With these, consider a training set $\Tc$ that is generated by either of the two models in \eqref{eq:yi_log} or \eqref{eq:yi_GM}. Let $\betabh$ be given as in \eqref{eq:ML}. Recall the notation in  \eqref{eq:HGZ} and define a random variable $V=V_{r,s}$ depending on the data generation model as in \eqref{eq:V}. Let $(\mu,\alpha>0,\la>0)$ be the unique solution to the following system of three nonlinear equations in three unknowns,
\begin{align}
0 &= \mathbb{E}\left[V\cdot\ell^\prime(\prox{\ell}{\alpha H+\mu V}{\lambda})\right],\nn\\
\alpha^2\,\kappa &= \lambda^2\,\mathbb{E}\left[(\ell^\prime(\prox{\ell}{\alpha H+\mu V}{\lambda}))^2\right],\nn
\\
\kappa &= \lambda\,\E\Big[\frac{{\ell}^{\prime\prime}(\prox{\ell}{\alpha H+\mu V}{\lambda})}{1+\la{\ell}^{\prime\prime}(\prox{\ell}{\alpha H+\mu V}{\lambda})}\Big].\label{eq:ML_gen}
\end{align}
Then,  
\bea
\Cc(\betabh)\ras \frac{s\,\mu}{r\,\sqrt{\mu^2+\alpha^2}}\, \quad\text{and}\quad
\Rc(\betabh)\ras\Pro\left(\mu\,V+\alpha\,H<0\right).\nn
\eea
\end{propo}

When data is generated according to \eqref{eq:yi_GM}, then $V=G+s$. Using this and Gaussian integration by parts, the system of three equations in \eqref{eq:ML_gen} simplifies to the following:
\begin{align}
\mu\,\kappa &= -s\,\mathbb{E}\left[\lambda\,\ell^\prime\Big(\,\prox{\ell}{G_{\mu,\alpha}}{\lambda}\,\Big)\right],\nn\\
\alpha^2\,\kappa &= \mathbb{E}\left[(\lambda\,\ell^\prime(\,\prox{\ell}{G_{\mu,\alpha}}{\lambda}\,))^2\right],\nn
\\
\kappa &= \E\Big[\frac{\lambda{\ell}^{\prime\prime}(\,\prox{\ell}{G_{\mu,\alpha}}{\lambda}\,)}{1+\la{\ell}^{\prime\prime}(\,\prox{\ell}{G_{\mu,\alpha}}{\lambda}\,)}\Big],
\label{eq:ML_GM}
\end{align}
where the expectations are over a \emph{single} Gaussian random variable $G_{\mu,\alpha}\sim \Nn\left(\,\mu s \,,\, \alpha^2+\mu^2\,\right)$.
We numerically solve the system of equations via a fixed-point iteration method first suggested in \cite{Master} in a similar setting (see also \cite{salehi2019impact,taheri2019sharp}). We empirically observe that reformulating the equations as in \eqref{eq:ML_GM} is critical for the iteration method to converge. Note that the performance of logistic regression (the same is true for SVM in Section \ref{sec:sep}) does \emph{not} depend on the prior probabilities for the two classes of the GM model with isotropic Gaussian features and symmetric class centers $\pm\etab_0$ in \eqref{eq:yi_GM}. The reason behind this is that our linear classifier \eqref{eq:yhat} does not include a constant offset, say $\yh = \sign(\w^T\betabs + \beta_c).$ There is nothing fundamental changing in our analysis and results when adding an offset $\beta_c$ other than the final asymptotic formulae becoming somewhat more complicated. Thus, we have decided to study \eqref{eq:yhat} without offset to keep the exposition simpler and focused on the main points.

The performance of logistic loss under the logistic model was recently studied in \cite{sur2019modern,salehi2019impact,taheri2019sharp}. Compared to this prior work: (i) our result extends to the mismatch model of Section \ref{sec:model}; (ii) we obtain a prediction for the classification error; (iii) we prove convergence in an almost-sure sense (rather than with probability 1 as in previous works). Also, to the best of our knowledge, this work is the first to formally prove uniqueness and existence of solutions to \eqref{eq:ML_gen} under non-separable data. Our proof is based on an extension of the CGMT presented as Corollary \ref{cor:characterization} in Appendix \ref{sec:CGMT_new}. The proof of Proposition \ref{propo:ML} is given in Appendix \ref{sec:proof_ML}. 
\subsubsection{Separable data}\label{sec:sep}
Here, we characterize the asymptotic generalization performance of hard-margin SVM under both models \eqref{eq:yi_log} and \eqref{eq:yi_GM}. 
\begin{propo}[Asymptotics of SVM]\label{propo:SVM}
Fix total signal strength $r$ and overparameterization ratio $\kappa<\kappa_\star$. Denote $s=s(\kappa)\in(0,r]$. With these, consider a training set $\Tc$ that is generated by either of the two models in \eqref{eq:yi_log} or \eqref{eq:yi_GM}. Let $\betabh$ be given as in \eqref{eq:SVM}. Recall the notation in  \eqref{eq:HGZ} and define a random variable $V=V_{r,s}$ depending on the data model as in \eqref{eq:V}. With these, define
\bea
\eta(q,\rho)&:=\mathbb{E} \Big(\rho \,V+H\,\sqrt{1-\rho^2}-\frac{1}{q}\Big)_{-}^2-(1-\rho^2)\kappa.\label{eq:eta_func}
\eea
Let $q^\star$ be the unique solution of the equation $$\min_{-1\leq\rho\leq 1}\eta(q,\rho) = 0.$$ Further let $\rho^\star$ be the unique minimum of $\eta(q^\star,\rho)$ for $\rho\in[-1,1]$.
Then, 
 \bea
\nn 
\Cc(\betabh) \ras \frac{\rho^\star \,s}{r}\, \quad\text{and}\quad
\Rc(\betabh)\ras\Pro\left(\rho^\star\,V+\sqrt{1-\rho^{\star 2}} \,H<0\right).\nn
\eea
\end{propo}

In addition to the stated results in Proposition \ref{propo:SVM}, our proof further shows that the optimal cost of the hard-margin SVM \eqref{eq:SVM} converges almost surely to $q^\star$. In other words, the margin $1/\|\betabh\|_2$ of the classifier converges in probability to $1/q^\star$. The dertailed proof of Proposition \ref{propo:SVM} is given in Appendix \ref{sec:SVM_prop}.

We note that under the GM model, the function $\eta$ in Proposition \ref{propo:SVM} simplifies to:
\bea\label{eq:eta_func_1}
\eta(q,\rho):=\mathbb{E} \big(G + \rho \,s-{1}/{q}\big)_{-}^2-(1-\rho^2)\kappa,
\eea
where the expectation is over a single Gaussian random variable $G\sim\Nn(0,1)$. Moreover, the formula predicting the risk simplifies to 
$$
\Rc(\betabh) \ras Q(\rho^\star\,s).
$$

\section{Numerical results and discussion}\label{sec:num}

\subsection{Feature selection models}\label{sec:explicit}
Recall that the number of features known at training is determined by $\kappa$. Specifically, $\kappa$ enters the formulae predicting the classification performance via the signal strength $s=s(\kappa)$. In this section, we specify two explicit models for feature selection and their corresponding functions $s(\kappa)$. Similar models are considered in \cite{breiman1983many,hastie2019surprises,belkin2019two}, albeit in a linear regression setting.
%
%
%

\vspace{5pt}
\noindent\textbf{Linear model}.~We start with a uniform feature selection model characterized by the following parametrization:
\bea\label{eq:uniform}
s^2 = s^2(\kappa) = r^2\cdot({\kappa}/{\zeta}),\quad\kappa\in(0,\zeta],
\eea
for fixed $r^2$ and $\zeta>1$. This models a setting where all coefficients of the regressor $\etab_0$ have equal contribution. Hence, the signal strength $s^2=\|\betab_0\|_2^2$ increases linearly with the number $p$ of features considered in training. 

\vspace{5pt}
\noindent\textbf{Polynomial model}.~In the linear model, adding more features during training results in a linear increase of the signal strength. In contrast, our second model assumes diminishing returns:
%
%
\bea\label{eq:poly}
s^2 = s^2(\kappa) = r^2\cdot\big(1-(1+\kappa)^{-\gamma}\big),~\kappa>0,
\eea
for some $\gamma\geq 1$. As $\kappa$ increases, so does the signal strength $s$, but the increase is less significant for larger values of $\kappa$ at a rate specified by $\gamma$.

\subsection{Risk curves}\label{sec:curves}

Figure \ref{fig:exp4str} assumes the logistic data model \eqref{eq:yi_log} and polynomial feature model \eqref{eq:poly} for $\gamma=2$ and three values of total signal strength $r$. The crosses (`$\times$') are simulation results obtained by running GD on synthetic data generated according to \eqref{eq:yi_log} and \eqref{eq:poly}. Specifically, the depicted values are averages calculated over $500$ Monte Carlo realizations for $p=150$. For each realization, we ran GD on logistic loss (see Sec.~\ref{sec:class}) with a fixed step size until convergence and recored the resulting risk.  Similarly, the squares (`$\square$') are simulation results obtained by solving SVM \eqref{eq:SVM} over the same number of different realizations and averaging over the recorded performance. As expected by \cite{pmlr-v99-ji19a} (also Sec.~ \ref{sec:GD}), the performance of GD matches that of SVM when data are separable. Also, as predicted by Proposition \ref{propo:PT}, the data is separable with high-probability when $\kappa>\kappa_\star$ (the threshold value $\kappa_\star$ is depicted with dashed vertical lines). This is verified by noticing the dotted (`$\bullet$') scatter points, which record (averages of) the training error. Recall from Section \ref{sec:class} that the training error is zero if and only if the data are separable. Finally, the solid lines depict the theoretical predictions of Proposition \ref{propo:ML} ($\kappa<\kappa_\star$) and Proposition \ref{propo:SVM} {($\kappa>\kappa_\star$)}. 
The results of Figure \ref{fig:exp4str} validate the accuracy of our predictions: our asymptotic formulae predict the classification performance of GD on logistic loss for all values of $\kappa$. Note that the curves correspond to excess risk defined in \eqref{eq:eR}; see also related Footnote \ref{footnote:eR}. Corresponding results for the cosine similarity are presented in Figure \ref{fig:app1} in Appendix \ref{sec:num_more}.

\begin{figure*}[t]
\begin{center}
\includegraphics[width=0.4\textwidth]{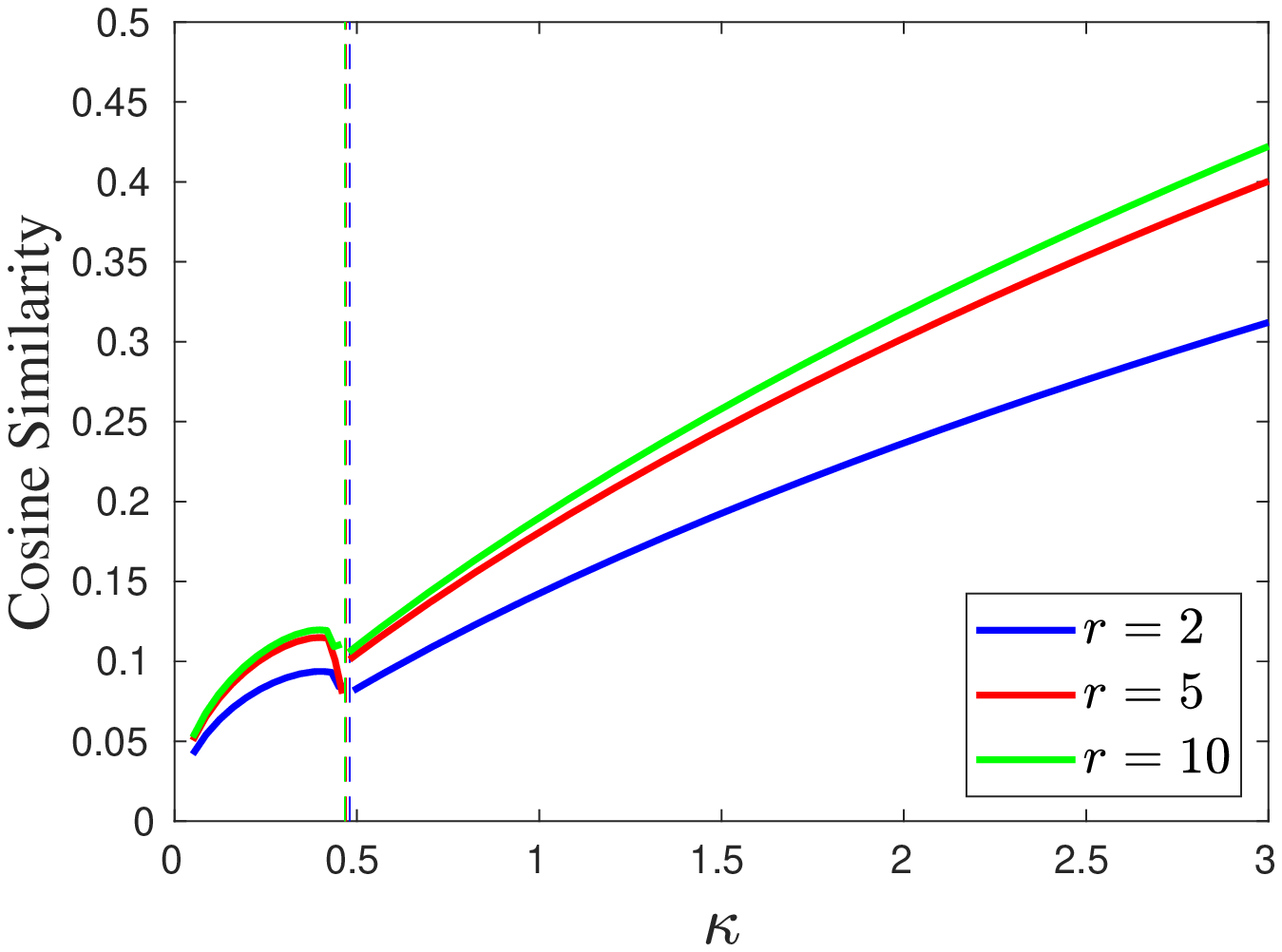}
\includegraphics[width=0.4\textwidth]{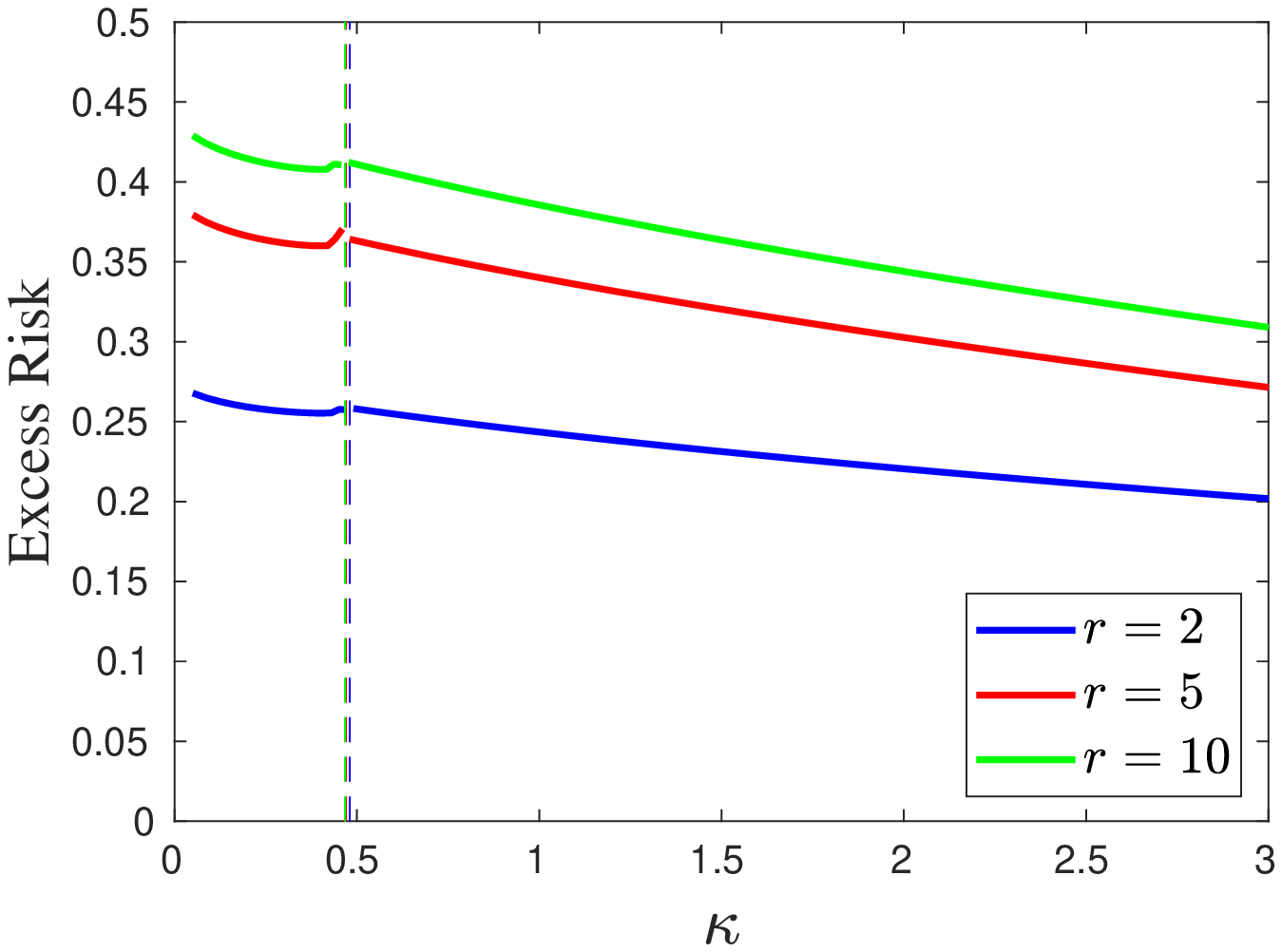}
\end{center}
\caption{Plots of the cosine similarity (left) and of the excess risk (right) as a function of $\kappa$ for the linear feature selection model (cf. \eqref{eq:uniform})  under logistic data. The curves shown are for $\zeta=3$ and three values of $r=2, 5$ and $10$.}
\label{fig:poly4str}
\end{figure*}

\begin{figure*}[t]
\begin{center}
\includegraphics[width=.3\textwidth]{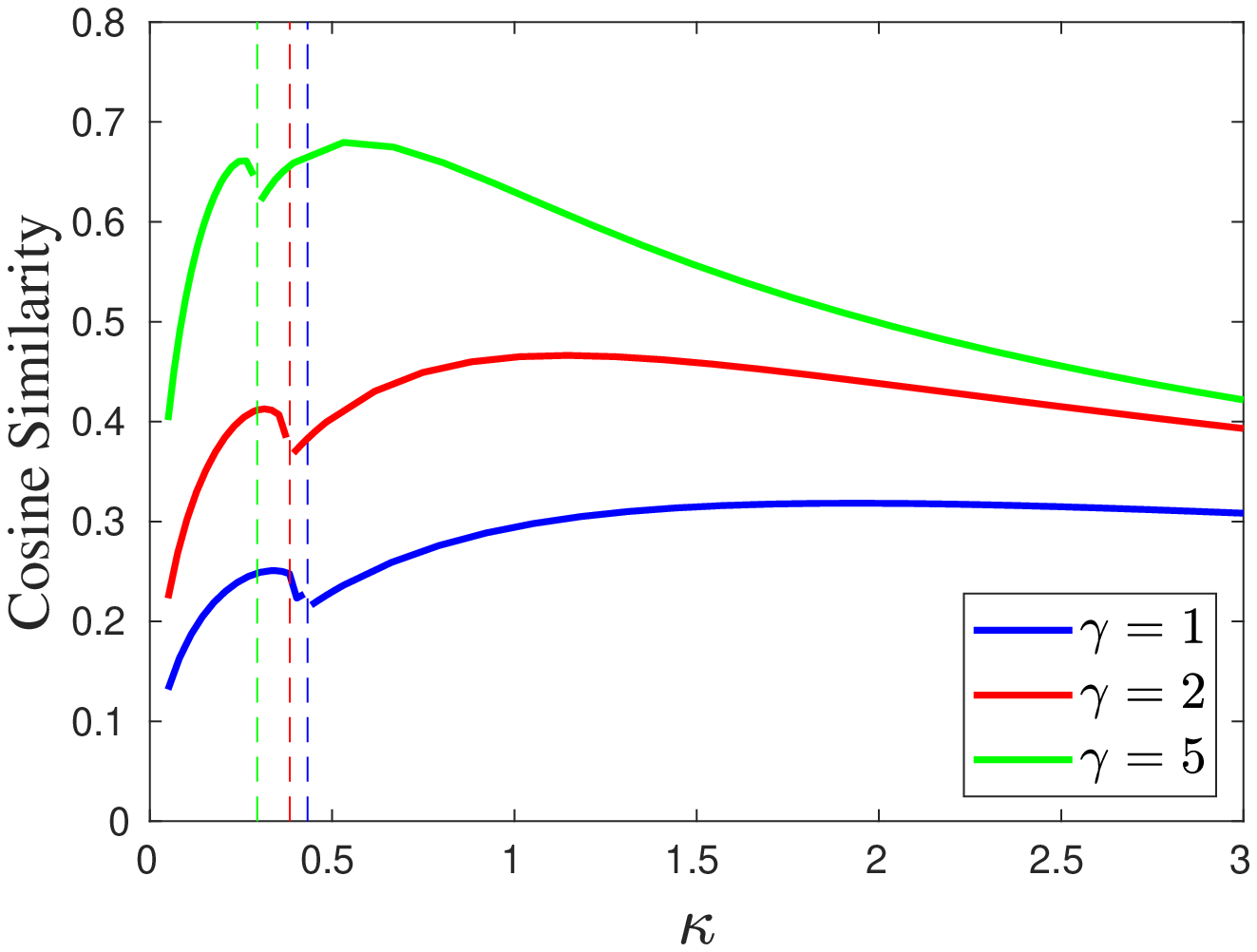}
\includegraphics[width=.3\textwidth]{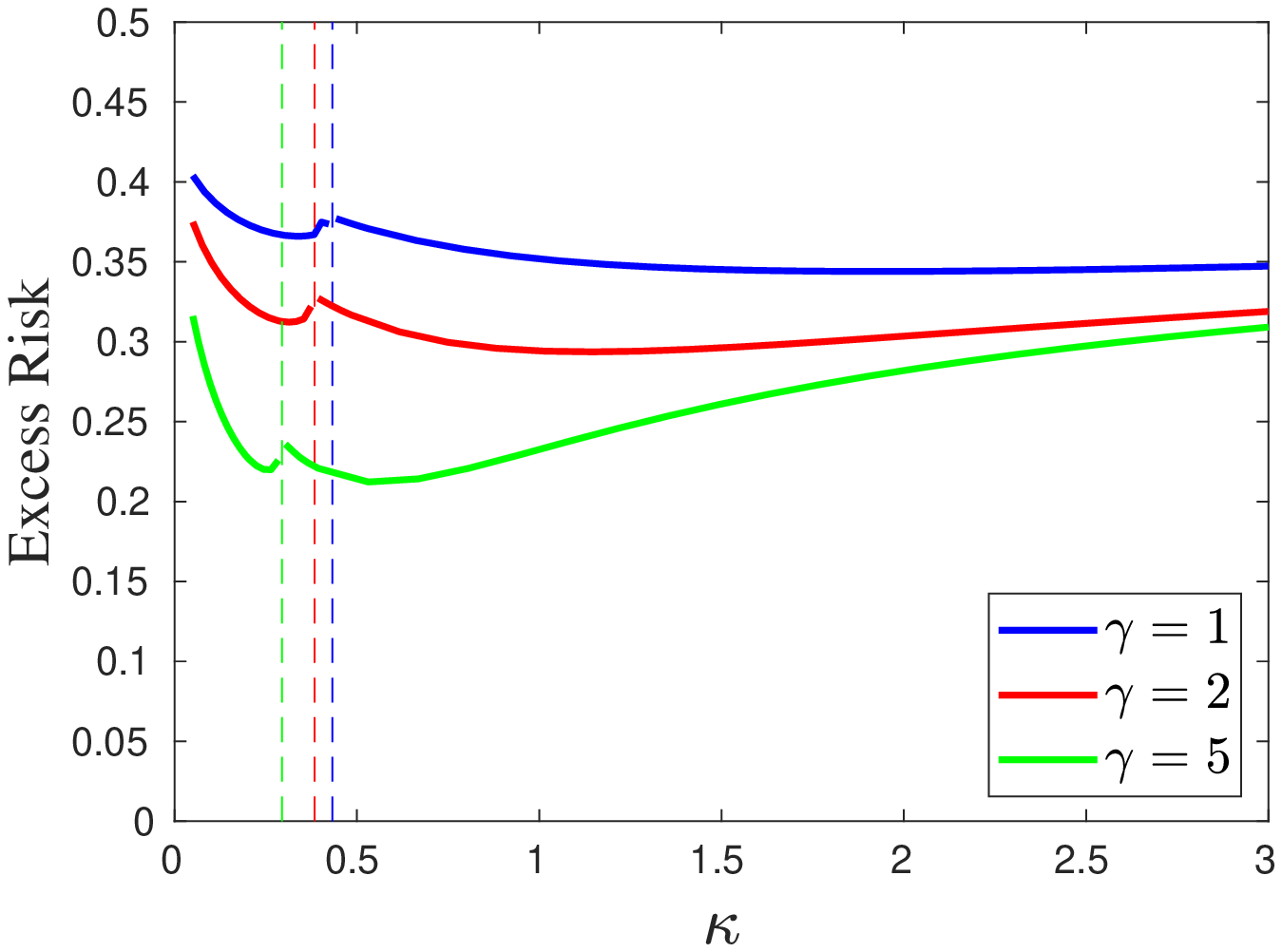}
\includegraphics[width=0.3\textwidth]{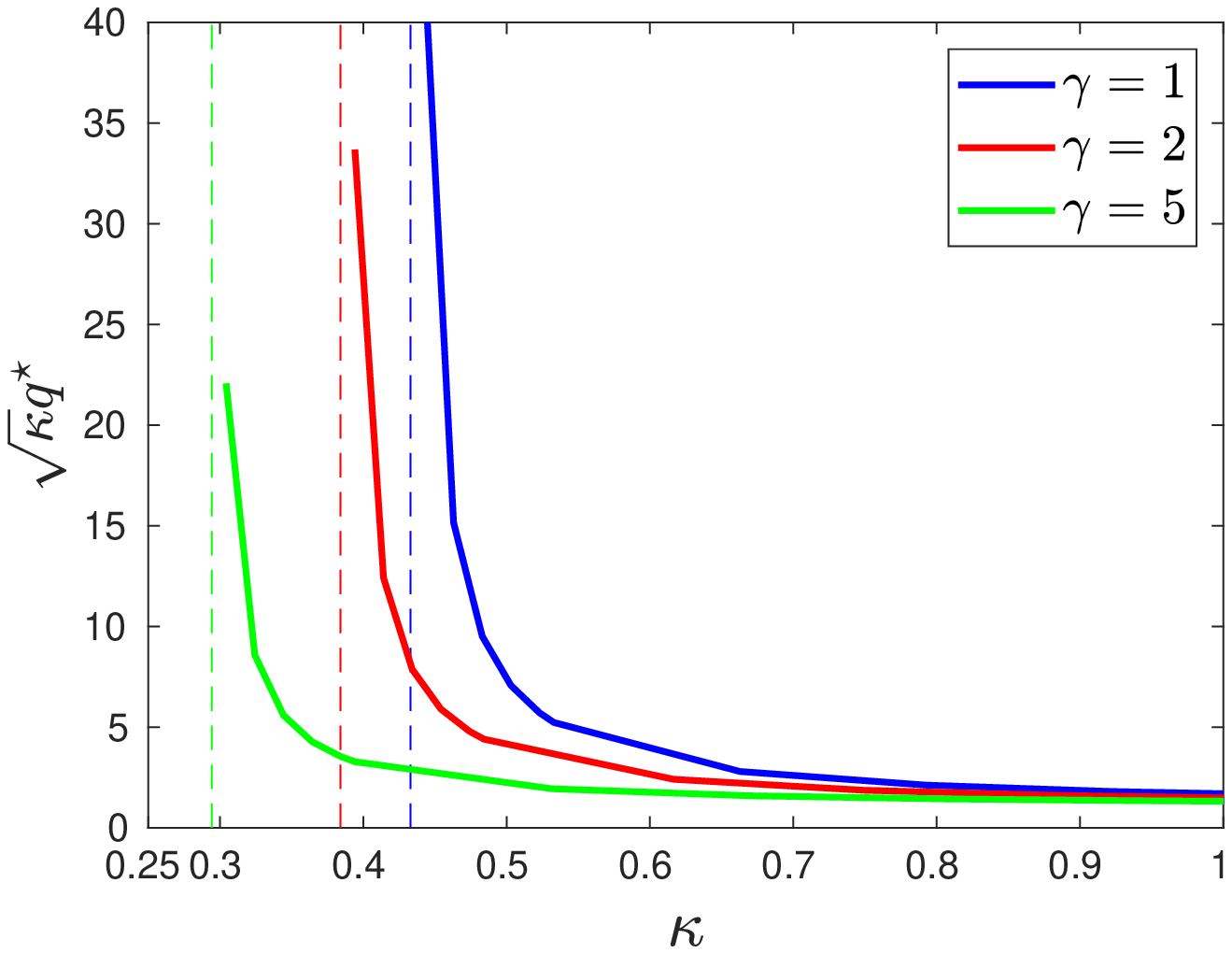}
\end{center}
\caption{Plots of the cosine similarity (left) and of the excess risk (middle) as a function of $\kappa$ for the polynomial feature selection model (cf. \eqref{eq:poly}) under logistic data. The curves shown are for $r=10$ and three values of $\gamma=1, 2$ and $5$. The rightmost plot shows the asymptotic prediction $\sqrt{\kappa}q^\star$ of Proposition \ref{propo:SVM} for the normalized margin $\sqrt{\kappa}\|\betabh\|_2$ as a function of $\kappa$ in the overparameterized regime. Note that $\sqrt{\kappa}q^*$ decreases monotonically with $\kappa$ and does not reveal the U-shape of the risk curve in the middle plot.} 
\label{fig:exp4a}
\end{figure*}

Under the logistic model \eqref{eq:yi_log}, Figures \ref{fig:poly4str} and \ref{fig:exp4a} depict
the cosine similarity and  excess risk curves of GD as a function of $\kappa$ for the linear and the polynomial model, respectively. Compared to Figure \ref{fig:exp4str}, we only show the theoretical predictions. Note that the linear model is determined by parameters $(\zeta, r)$ and the polynomial model by $(\gamma,r)$. Once these values are fixed, we compute the threshold value $\kappa_\star$. Then, we numerically evaluate the formulae of Proposition \ref{propo:ML} ($\kappa<\kappa_\star$) and Proposition \ref{propo:SVM} ($\kappa>\kappa_\star$).

\begin{figure*}[t]
\begin{center}
\includegraphics[width=0.4\textwidth]{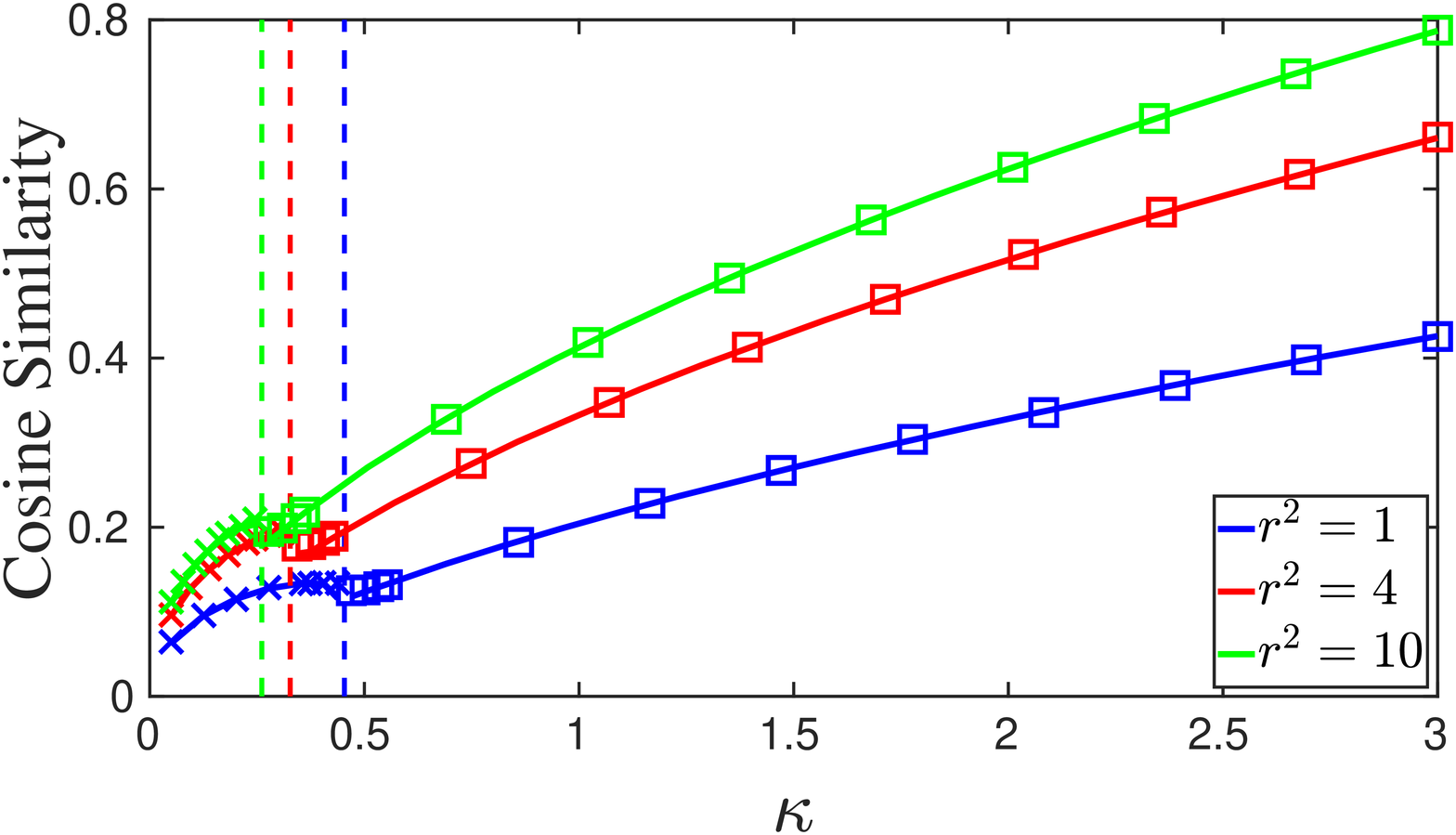}
\includegraphics[width=0.4\textwidth]{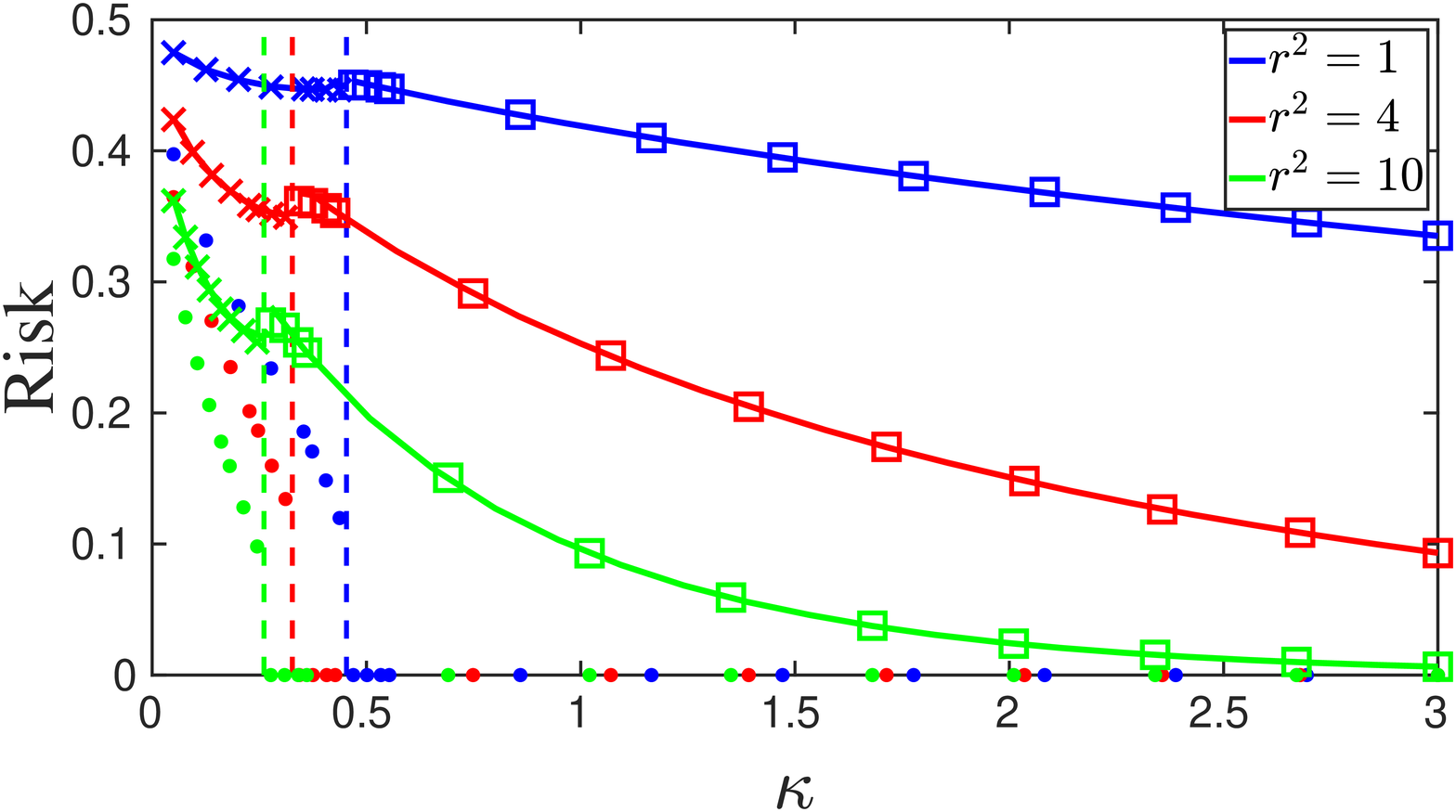}
\end{center}
\caption{Plots of the cosine similarity (left) and of the risk (right) as a function of $\kappa$ for the linear feature selection model (cf. \eqref{eq:uniform})  under data from a Gaussian mixture. The curves shown are for $\zeta=3$  and three values of the total signal strength $r^2 = 1, 4$ and $10$.} 
\label{fig:GM4r}
\end{figure*}

\begin{figure*}[t]
	\begin{center}
		\includegraphics[width=0.4\textwidth]{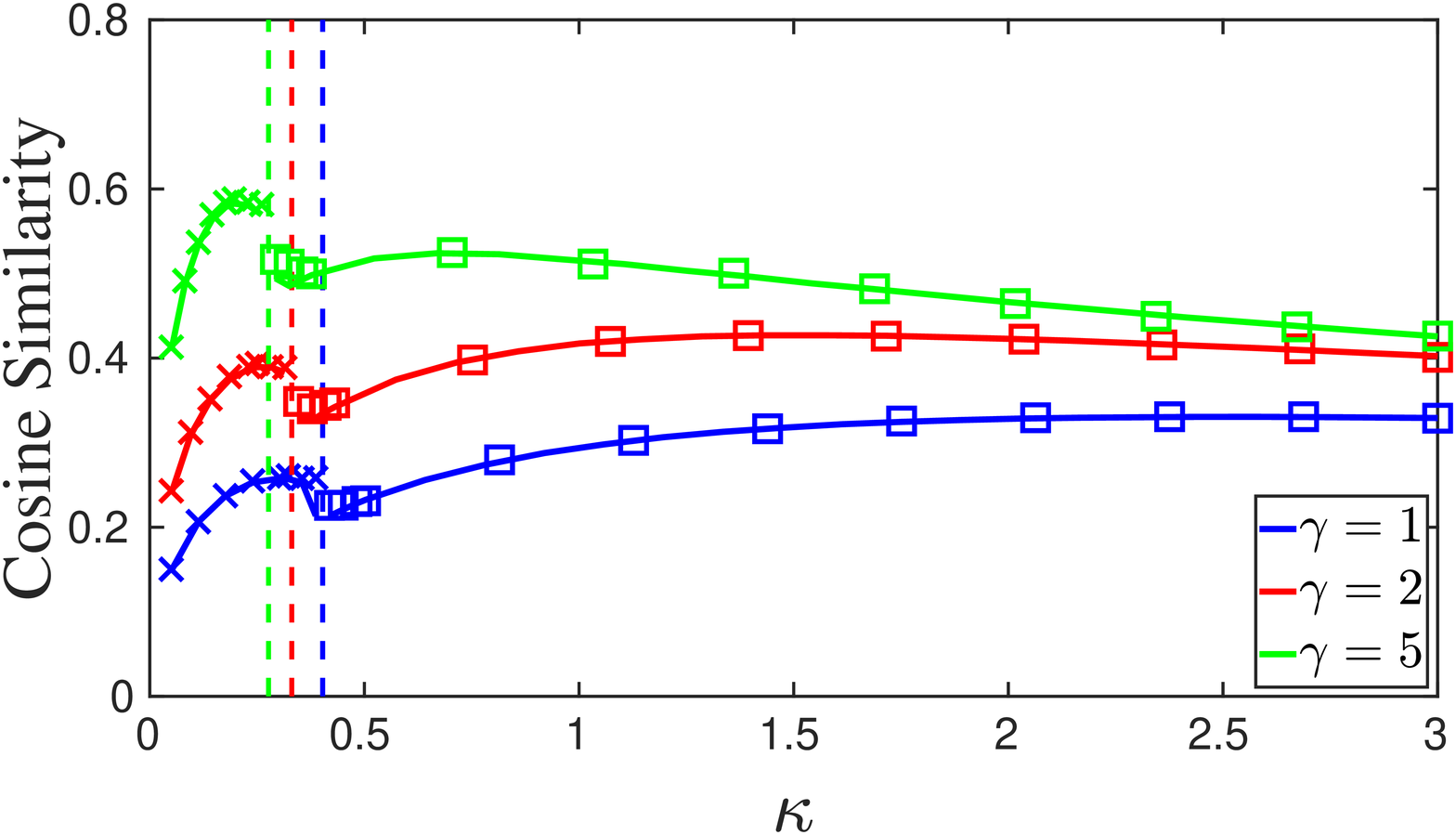}
		\includegraphics[width=0.4\textwidth]{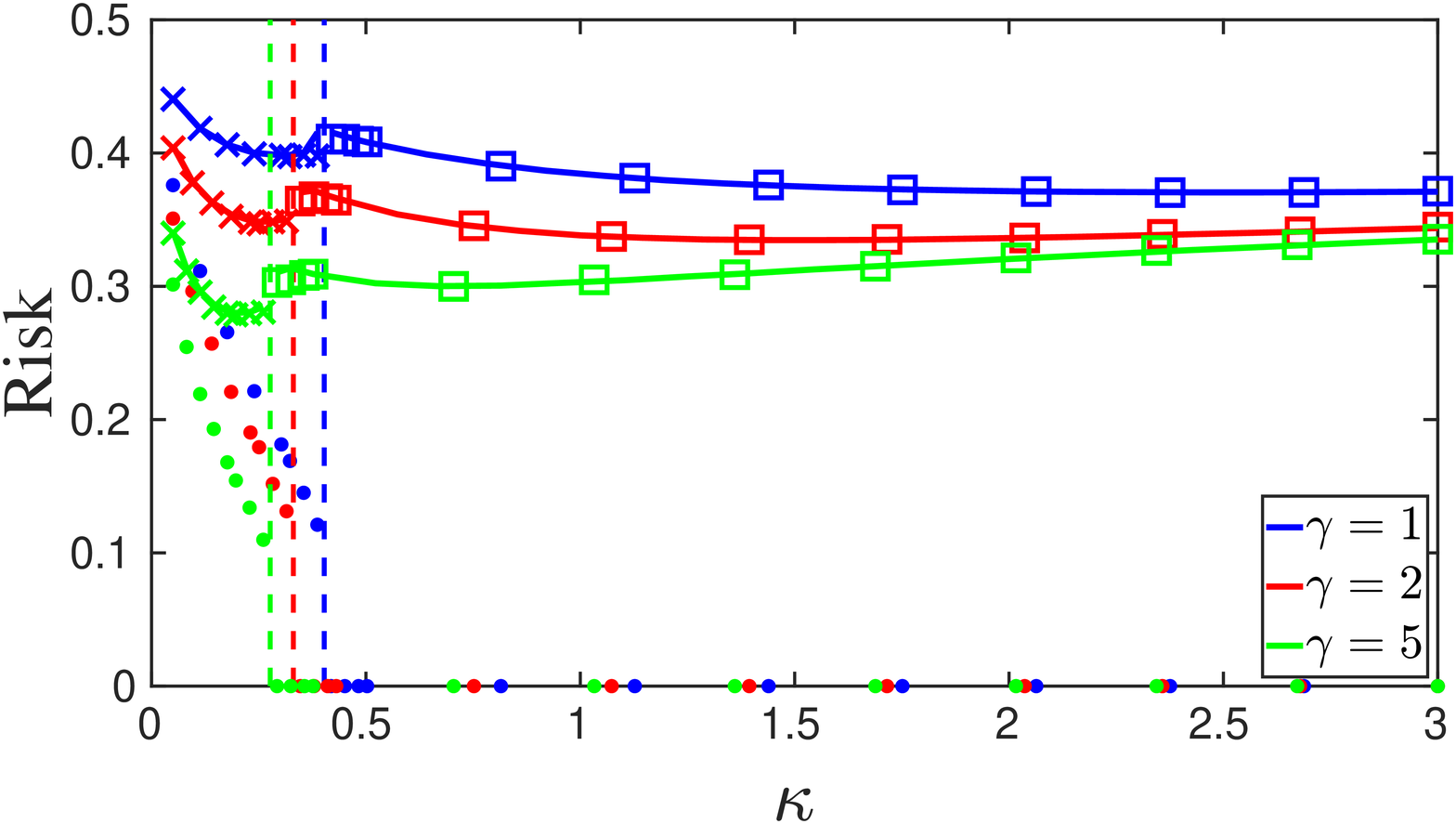}
%
		
	\end{center}
	\caption{Plots of the cosine similarity (left) and of the risk (right) as a function of $\kappa$ for the polynomial feature selection model (cf. \eqref{eq:poly}) under data from a Gaussian mixture. The three curves shown are for total signal strength $r=1$ and three values of the exponent $\gamma = 1, 2$ and $5$ in \eqref{eq:poly}.} 
	\label{fig:GMexpcos4a}
\end{figure*}

Finally, Figures \ref{fig:GM4r} and \ref{fig:GMexpcos4a} show risk and cosine-similarity curves for the Gaussian mixture data model \eqref{eq:yi_GM} with linear and polynomial feature selection rules, respectively. The figures compare simulation results to theoretical predictions similar in nature to Figure \ref{fig:exp4str}, thus validating the accuracy of the latter for the GM model. For the simulations, we generate data according to \eqref{eq:yi_GM} with $\pi_{+1}=1/2$, $n=200$ and $d=600$. The results are averages over $500$ Monte Carlo realizations.


\begin{remark}[Performance at high-SNR: logistic vs GM model] Comparing the curves for $r=2$ in Figures \ref{fig:exp4a} and \ref{fig:GMexpcos4a}, note that the cosine similarity under the GM model takes values (much) larger than those under the logistic model. While the value of $r$ plays the role of SNR in both cases, the two models \eqref{eq:yi_log} and \ref{eq:yi_GM} are rather different. The discrepancy between the observed behavior of the cosine similarity (similarly for test error) can be understood as follows. On the one hand, in the GM model (cf. \eqref{eq:yi_GM}) the features satisfy  $\x_i=y_i\etab_0+\eps_i$, $\eps_i\sim\Nn(0,1)$. Thus, learning the vector $\etab_0$ involves a \emph{linear} regression problem with SNR $r=\|\etab_0\|_2$. On the other hand, in the logistic model (cf. \eqref{eq:yi_log}) at high-SNR $r\gg 1$ it holds $y_i\approx \sign(\x_i^T\etab_0)$. Hence, learning $\etab_0$ involves solving a system of \emph{one-bit} measurements, which is naturally more challenging than linear regression.
\end{remark}

\subsection{Discussion on double descent}\label{sec:disc}

The double-descent behavior of the test error (resp. double-ascent behavior of the cosine similarity) as a function of the model complexity can be clearly observed in all the plots described above.

First, focus on the underparameterized regime  $\kappa<\kappa_\star$ (cf. area on the left of the vertical dashed lines). Here, the number $n$ of training examples is large compared to the size $p$ of the unknown weight vector $\betab_0$, which favors learning $\betab_0$ with better accuracy. However, since only a (small) fraction $p/d$ of the total number $d$ of features are used for training, the observations are rather noisy. This tradeoff manifests itself with a ``U-shaped curve" for $\kappa<\kappa_\star$. 

Next, as the size overparameterization ratio $\kappa$ increases beyond $\kappa_\star$ (cf. area on the right of the vertical dashed lines) the training data become linearly separable. The implicit bias of GD on logistic loss leads to convergence to the max-margin classifier. In all Figures \ref{fig:exp4str}--\ref{fig:GMexpcos4a}, it is clearly seen that the risk curves experience a second descent just after the interpolation threshold $\kappa_\star$. This observation analytically reaffirms similar empirical observations in more complicated learning tasks, architectures and datasets \cite{DDD}. Intuitively, the second descent can be attributed to: (a) the implicit bias of GD and (b) the fact that larger $\kappa$ implies smaller degree of model mismatch in the model of Section \ref{sec:train_model}. In fact, it can be seen that depending on the feature selection model the curve in the overparameterized regime can either be monotonically decreasing (e.g., Figure \ref{fig:poly4str}) or having a U-shape (e.g. Figure \ref{fig:exp4a}). In particular, the polynomial feature selection model \eqref{eq:poly} favors the later behavior and the ``U-shape" is more pronounced for larger values of the parameter $\gamma$ in \eqref{eq:poly}. At this point, it is worth noting that the potential U-shape in the overparameterized regime is \emph{not} predicted by classical bounds on the test error of margin-classifiers in terms of the normalized max-margin ${\sqrt{\kappa}}\big/{\|\hat\betab\|_2}$ \cite[Thm.~26.13]{shalev2014understanding}. This is demonstrated in Figure \ref{fig:exp4a} (right), which shows that the asymptotic limit $\sqrt{\kappa} q^\star$  of the normalized max-margin (cf. Proposition \ref{propo:SVM}) is monotonic in $\kappa$ rather than following a ``U-shape". 


 
Owing to the double-descent behavior, the risk curves have are two \emph{local minima} corresponding to the two regimes of learning. This observation is valid for all different data models depicted in Figures \ref{fig:exp4str}--\ref{fig:GMexpcos4a}. On the other hand, whether the \emph{global minimum} of the curves appears in the overparameterized regime or not, this depends on the nature of the training data. For example, for both the logistic and GM models under the linear feature model in Figures \ref{fig:poly4str} and \ref{fig:GM4r} the global minimum occurs at $\kappa_{\rm opt}>\kappa_\star$   in the interpolating regime for all SNR values. The value of $\kappa_{\rm opt}$ determines the optimal number of features that need to be selected during training to minimize the classification error. 
Note that for $\kappa_{\rm opt}$ the training error is zero, yet the classification performance is best.  However, for the polynomial model depending on the value of $\gamma$ it can happen that $\kappa_{\rm opt}<\kappa_\star$ (cf. Figures \ref{fig:exp4a} and \ref{fig:GMexpcos4a} for $\gamma=5$.).


The previous discussion related to Figures  \ref{fig:exp4str}--\ref{fig:GMexpcos4a} makes clear that important features of double-descent curves, such as the global minimum and the location of cusp of the curves critically depend on the data. The recent paper \cite{kini2020analytic} has extended the present analysis to GD on square-loss. When combined, \cite{kini2020analytic} and our paper demonstrate analytically that double descent behaviors can occur even in simple linear classification models. Moreover, they show analytically that the features of the curves depend both on the data (e.g., logistic vs GM) as well as on the learning algorithm (e.g., square vs logistic loss). These conclusions resemble, and thus might offer insights, to corresponding empirical findings in the literature \cite{DDD,belkin2018reconciling,spigler2019jamming}. We refer the interested reader to \cite{kini2020analytic} for further details on how the choice of loss in \eqref{eq:Remph} impacts the double descent. 

We conclude this section, with an investigation of the dependence of the double-descent curves on the size of the training set. Specifically, we fix $d$ (the dimension of the ambient space) and study double descent for three distinct values of the training-set size, namely $n=\frac{1}{3}d, \frac{2}{3}d,$ and $d$. For these three cases, we plot the test error vs the model size $p=\frac{\kappa}{\zeta} d$ in Figure \ref{fig:overpd} for the linear (Left) and polynomial (Right) feature-selection model.  First, observe that as the training size grows larger (i.e., $\zeta$ decreases), the interpolation threshold shifts to the right (thus, it also increases) under both models. Second, for the linear model,  larger $n$ improves the performance for all model sizes. On the other hand, for the polynomial model, larger $n$ does not necessarily imply that the global minima of the corresponding curve corresponds to better test error. Moreover, the curves cross each other. This implies that more training examples do not necessarily help, and could potentially hurt the classification performance, whether at the underparameterized or the overparameterized regimes. For example, for model size $p\approx 0.4 d$, the test error is lowest for $n=\frac{1}{3}d$ and largest for $n=d$. The effect of the size of the training set on the DD curve was recently studied \emph{empirically} in \cite{DDD}; see also \cite{nakkiran2019more} for an analytical treatment in a linear regression setting. Our investigation is motivated by and theoretically corroborates the observations reported therein. Also please refer to \cite{kini2020analytic} for corresponding results on square-loss. 


\begin{figure}[t]
	\begin{center}
		\includegraphics[width=0.4\textwidth]{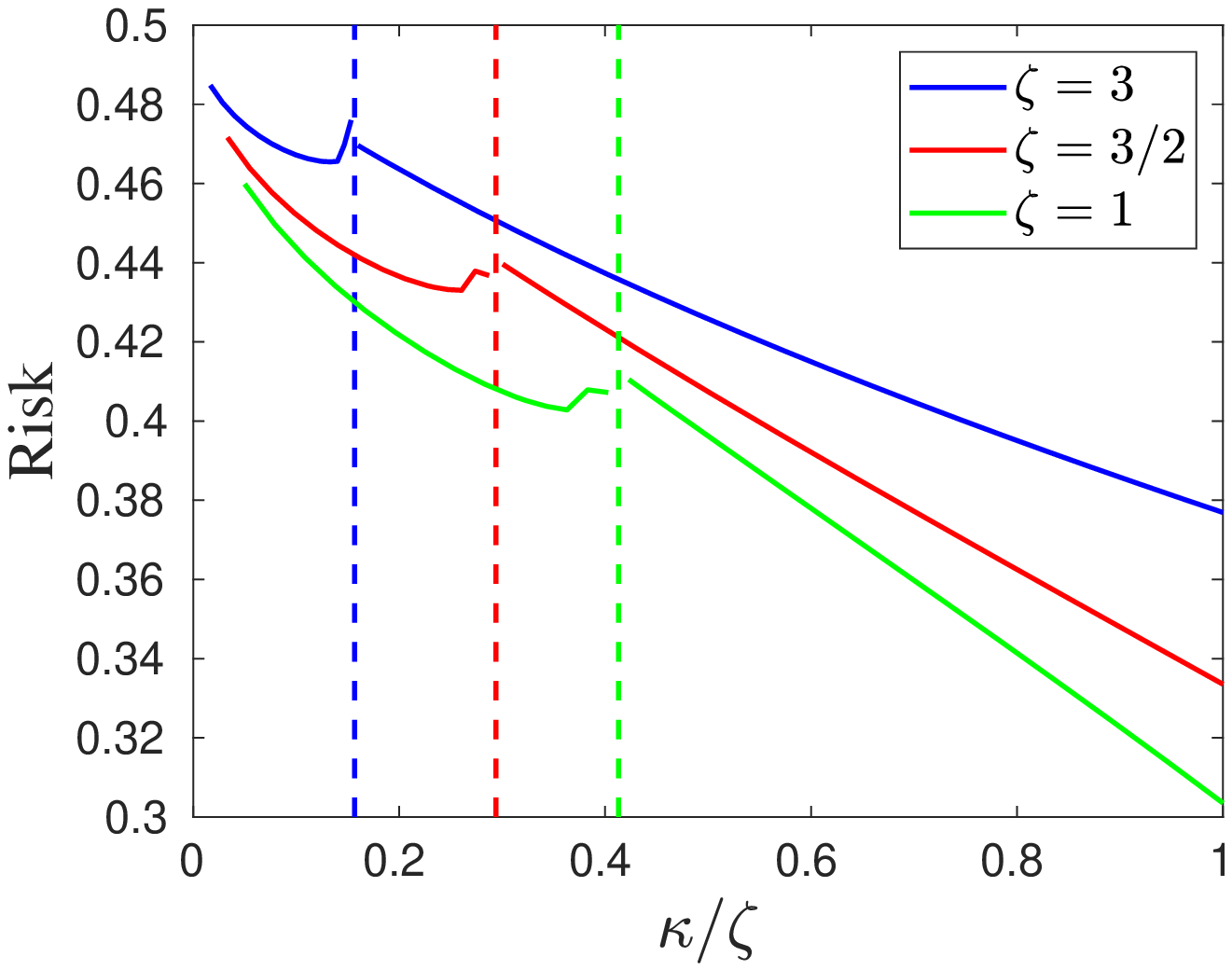}
		\includegraphics[width=0.4\textwidth]{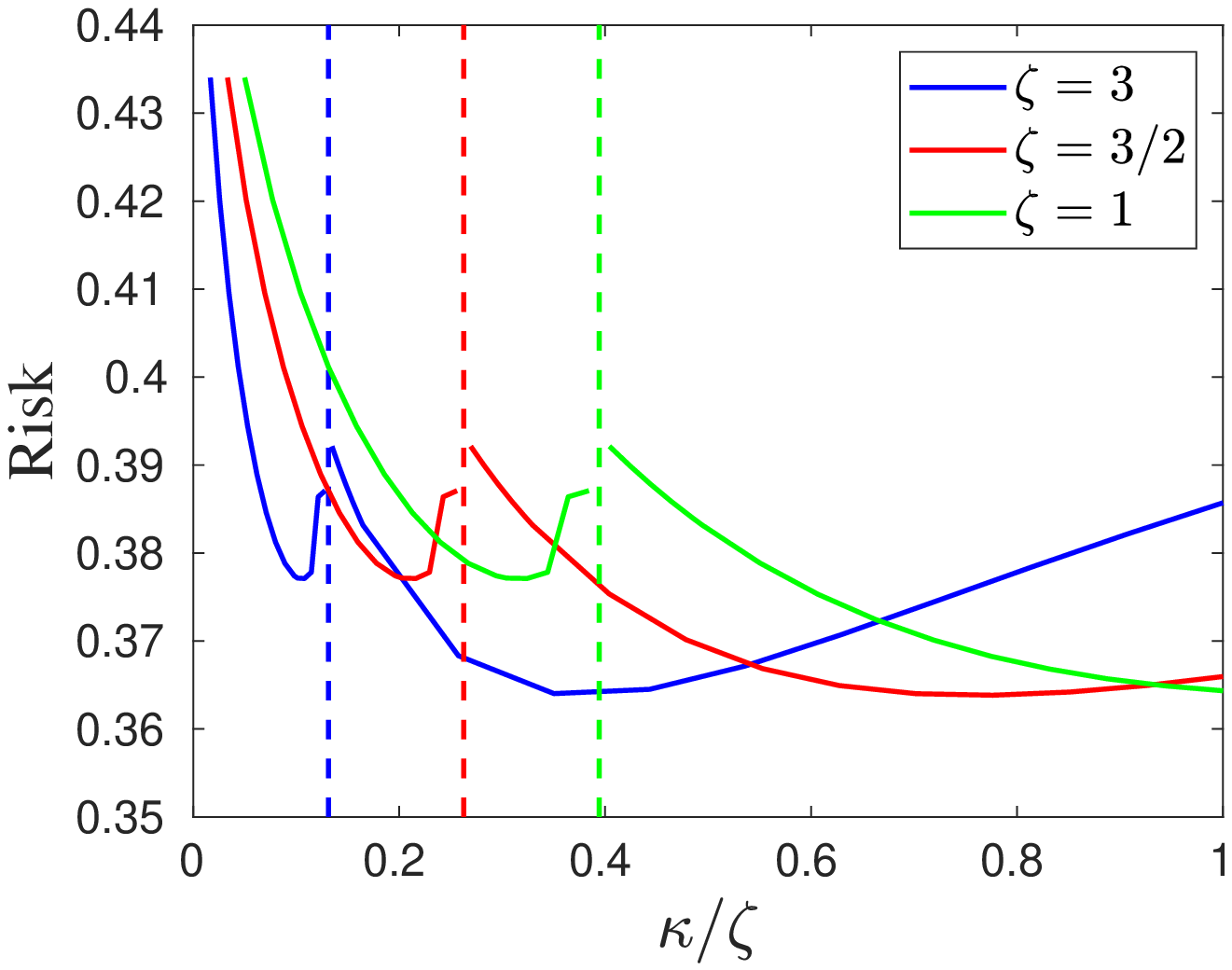}
	\end{center}
	\caption{ Studying the effect of training-set size on the double-descent curves. Smaller values of $\zeta$ correspond to larger training sets (aka larger $n$). The risk is plotted as a function of the model size (aka $p$) normalized by the space dimension $d$. Both plots are for the logistic data model. (Left) Linear feature-selection model for $r=5$. (Right) Polynomial feature-selection model for $r=5$ and $\gamma=2$.
	 }
	\label{fig:overpd}
\end{figure}

\section{Future work}
We studied the classification performance of GD on logistic loss under the logistic and GM models with isotropic Gaussian features. We further used these results to demonstrate double-descent curves in binary linear classification under such simple models. Specifically, the proposed setting is simple enough that it allows a principled analytic study of several important features of double-descent (DD) curves, such as the dependence of the curve's transition threshold and global minimum on:  (i) the learning model and SNR, (ii) the size of the training set and (iii) the loss function (in combination with the results of \cite{kini2020analytic}). 

We believe that this line of work can be extended in several aspects and we briefly discuss a few of them here. To begin with, it is important to better understand the effect of correlation $\Sigma=\E[\x\x^T]$  on the DD curve. The papers \cite{hastie2019surprises,lolas2020regularization} and \cite{montanari2019generalization} derive sharp asymptotics on the performance of min-norm estimators under correlated Gaussian features for linear regression and linear classification, respectively. While this allows to numerically plot DD curves, the structure of $\Sigma$ enters the asymptotic formulae through unwieldy expressions. Thus, understanding how different correlation patterns affect DD to a level that may provide practitioners with useful insights and easy take-home-messages remains a challenge; see \cite{bartlett2019benign} for related efforts. Another important extension is that to multi-class settings. The recent work \cite{DDD} includes a detailed empirical investigation of the dependence of DD on label-noise and data-augmentation. Moreover, the authors identify that ``DD occurs not just as a function of model size, but also as a function of training epochs". It would be enlightening to analytically study whether these type of behaviors are already present in simple models similar to those considered here.  Yet another important task is carrying out the analytic study to more complicated --nonlinear-- data models; see \cite{mei2019generalization,montanari2019generalization} for progress in this direction.

\bibliography{compbib}

\onecolumn
\appendix


\section{Main Analysis Tool}\label{sec:CGMT}

In Section \ref{sec:CGMT_new} we present our main analysis tools: Theorems \ref{th:feasibility} and \ref{th:characterization}, which are extensions of the CGMT and can potentially be used beyond the scope of this paper. For the reader's convenience, we first recall the CGMT in Section \ref{sec:CGMT}.

\subsection{The Convex Gaussian Min-Max Theorem (CGMT)}\label{sec:CGMT}

The CGMT is an extension of Gordon's Gaussian min-max inequality (GMT) \cite{Gor88}. In the context of high-dimensional inference problems, Gordon's inequality was first successfully used in the study oh sharp phase-transitions in noiseless Compressed Sensing \cite{Sto,Cha,TroppEdge,Sto,TroppUniversal}. More recently, \cite{StoLasso} (see also \cite[Sec.~10.3]{TroppEdge}) discovered that Gordon's inequality is essentially tight for certain convex problems. A concrete and general formulation of this idea was given by \cite{COLT} and was called the CGMT. 


In order to summarize the essential ideas, consider the following two Gaussian processes:
\begin{subequations}\label{eq:POAO}
\begin{align}
X_{\w,\ub} &:= \ub^T \G \w + \psi(\w,\ub),\\
Y_{\w,\ub} &:= \norm{\w}_2 \g^T \ub + \norm{\ub}_2 \h^T \w + \psi(\w,\ub),\label{eq:AO_obj}
\end{align}
\end{subequations}
where: $\G\in\mathbb{R}^{n\times d}$, $\g \in \mathbb{R}^n$, $\h\in\mathbb{R}^d$, they all have entries iid Gaussian; the sets $\mathcal{S}_{\w}\subset\R^d$ and $\mathcal{S}_{\ub}\subset\R^n$ are compact; and, $\psi: \mathbb{R}^d\times \mathbb{R}^n \to \mathbb{R}$. For these two processes, define the following (random) min-max optimization programs, which are refered to as the \emph{primary optimization} (PO) problem and the \emph{auxiliary optimization} AO:
\begin{subequations}
\begin{align}\label{eq:PO_loc}
\Phi(\G)&=\min\limits_{\w \in \mathcal{S}_{\w}} \max\limits_{\ub\in\mathcal{S}_{\ub}} X_{\w,\ub},\\
\label{eq:AO_loc}
\phi(\g,\h)&=\min\limits_{\w \in \mathcal{S}_{\w}} \max\limits_{\ub\in\mathcal{S}_{\ub}} Y_{\w,\ub}.
\end{align}
\end{subequations}

According to the first statement of the CGMT\footnote{In fact, this is only a slight reformulation of Gordon's original comparison inequality \cite{Gor88}; see \cite{COLT}}, for any $c\in\R$, it holds:
\begin{equation}\label{eq:gmt}
\mathbb{P}\left( \Phi(\G) < c\right) \leq 2 \mathbb{P}\left(  \phi(\g,\h) < c \right).
\end{equation}
In other words, a high-probability lower bound on the AO is a high-probability lower bound on the PO. The premise is that it is often much simpler to lower bound the AO rather than the PO. 

However, the real power of the CGMT comes in its second statement, which asserts that if the PO is \emph{convex} then the AO in  can be used to tightly infer properties of the original PO, including the optimal cost and the optimal solution.
More precisely, if the sets $\mathcal{S}_{\w}$ and $\mathcal{S}_{\ub}$ are convex and \emph{bounded}, and $\psi$ is continuous \emph{convex-concave} on $\mathcal{S}_{\w}\times \mathcal{S}_{\ub}$, then, for any $\nu \in \mathbb{R}$ and $t>0$, it holds
\begin{equation}\label{eq:cgmt}
\mathbb{P}\left( \abs{\Phi(\G)-\nu} > t\right) \leq 2 \mathbb{P}\left(  \abs{\phi(\g,\h)-\nu} > t \right).
\end{equation}
In words, concentration of the optimal cost of the AO problem around $q^\ast$ implies concentration of the optimal cost of the corresponding PO problem around the same value $q^\ast$.  Asymptotically, if we can show that $\phi(\g,\h)\rP q^\ast$, then we can conclude that $\Phi(\G)\rP q^\ast$. Moreover, starting from \eqref{eq:cgmt} and under appropriate strict convexity conditions, the CGMT shows that concentration of the optimal solution of the AO problem implies concentration of the optimal solution of the PO around the same value. For example, if minimizers of \eqref{eq:AO_loc} satisfy $\norm{\w_\phi(\g,\h)}_2 \rP \alpha^\ast$ for some $\alpha^\ast>0$, then, the same holds true for the minimizers of \eqref{eq:PO_loc}: $\norm{\w_\Phi(\G)}_2 \rP \alpha^\ast$ \cite[Theorem 6.1(iii)]{Master}. Thus, one can analyze the AO to infer corresponding properties of the PO, the premise being of course that the former is simpler to handle than the latter. In \cite{Master}, the authors introduce a principled machinery that allows to (a) express general convex empirical-risk minimization (ERM) problems for noisy linear regression in the form of the PO and (b) simplify the AO from a (random) optimization over vector variables to an easier optimization over only few scalar variables, termed the ``scalarized AO".


\subsection{Novel extensions of the CGMT}\label{sec:CGMT_new}

As mentioned in the previous section, the CGMT applies to optimization problems in which the optimization sets are convex and \emph{compact} sets. While convexity is frequently met in practice, the compactness condition is not always satisfied (this is specifically true for compactness of the ``dual" variable $\ub$ in \eqref{eq:PO_loc}). Some existing ideas allowing to circumvent this problem are developed in \cite{Master} but, they are essentially suitable to high-dimensional linear regression problems and cannot be directly extended to handle any optimization problem. The difficulty becomes even more pronounced in problems like the hard-margin SVM that are not always feasible. 

In this work, we extend the CGMT in two directions. 
First, we show that the behavior of any convex optimization problem in the form of the CGMT can be characterized through that of  \emph{a sequence of} AO problems. Particularly, we  illustrate in Theorem \ref{th:feasibility} 
 how these sequences can help identify whether the underlying optimization problem is feasible or not. When feasibility conditions are met, we show in Theorem \ref{th:characterization} that the precise characterization of the PO boils down  to an analysis of the optimal costs of the sequence of AO problems.
While our motivation stems from the analysis of the  hard-margin SVM, we believe that the generalization of the CGMT presented here can be of independent interest offering a principled way for future statistical performance studies of related optimization-based inference algorithms.

The proofs of Theorems \ref{th:feasibility} and \ref{th:characterization} are presented in Appendix \ref{sec:CGMT_new_proofs}.


\paragraph{Notation.} Before stating our main results, let us first introduce some necessary notation and some useful facts. As previously (cf. \eqref{eq:PO_loc}),
let $\Phi({\bf G})$ be the primary optimization problem:
\begin{equation}
\Phi({\bf G})= \inf_{{\bf w}\in \mathcal{S}_{\bf w}}\sup_{{\bf u}\in \mathcal{S}_{\bf u}} X_{{\bf w},{\bf u}},
	\label{eq:pr}
\end{equation}
where $X_{\w,\ub}$ is defined in \eqref{eq:POAO}. We assume that  $X_{\w,\ub}$  is convex-concave and that the constraint sets $\Sc_\w\subset\Rb^d$ and $\Sc_\ub\subset\Rb^n$ are convex. However, in contrast to the assumption of the CGMT, the sets $\Sc_\w$ and $\Sc_\ub$ are no longer necessarily bounded. As such, for fixed $R$ and $\Gamma$, we consider the following ``$(R,\Gamma)$-bounded" version of \eqref{eq:pr}:
\bea\label{eq:PO_b}
\Phi_{R,\Gamma}({\bf G})=\min_{\substack{{\bf w}\in \mathcal{S}_{\bf w}\\ \|{\bf w}\|_2\leq R}} \max_{\substack{{\bf u}\in\mathcal{S}_{\bf u}\\ \|{\bf u}\|_2\leq \Gamma}} X_{{\bf w},{\bf u}}.
\eea
Clearly, 
\begin{equation}
	\Phi({\bf G})=\inf_{R\geq 0}\min_{\substack{{\bf w}\in\mathcal{S}_{\bf w}\\ \|{\bf w}\|_2\leq R}} \sup_{\Gamma \geq 0} \max_{\substack{{\bf u}\in\mathcal{S}_{\bf u}\\ \|{\bf u}\|_2\leq \Gamma}} X_{{\bf w},{\bf u}}.
	\label{eq:unbounded}
\end{equation}
Since the function $\max_{\substack{{\bf u}\in\mathcal{S}_{\bf u}\\ \|{\bf u}\|\leq \Gamma}} X_{{\bf w},{\bf u}}$ is convex in ${\bf w}$ and concave in $\Gamma$, the order of min-sup in \eqref{eq:unbounded} can be flipped \cite{Sion} leading to the following connection between the PO and its $(R,\Gamma)$-bounded version:
\begin{equation}
\Phi({\bf G})=\inf_{R\geq 0}\sup_{\Gamma \geq 0}\Phi_{R,\Gamma}({\bf G}).
	\label{eq:PO_original}
\end{equation}
We associate with $\Phi_{R,\Gamma}({\bf G})$ the following AO problem:
\begin{equation}
	\phi_{R,\Gamma}({\bf g},{\bf h})= \min_{\substack{{\bf w}\in \mathcal{S}_{\bf w}\\ \|{\bf w}\|\leq R}} \max_{\substack{{\bf u}\in\mathcal{S}_{\bf u}\\ \|{\bf u}\|_2\leq \Gamma}} Y_{{\bf w},{\bf u}}.
	\label{eq:AOpr}
\end{equation}
Note in the definitions of $\Phi_{R,\Gamma}(\G)$ and $\phi_{R,\Gamma}(\g,\h)$ in \eqref{eq:PO_b} and \eqref{eq:AOpr} that the involved optimization problems are over \emph{bounded} convex sets. Thus, the CGMT directly establishes a link between the two. Nevertheless, what is of interest to us is the PO problem $\Phi(\G)$ in \eqref{eq:pr} that optimizes over (possibly) unbounded sets. 

Theorems \ref{th:feasibility} and \ref{th:characterization} below establish a connection between $\Phi(\G)$ and a sequence (over $R$ and $\Gamma$) of the AO problems $\phi_{R,\Gamma}(\g,\h)$. In order to show this, it is convenient to further define $\phi_{R}({\bf g},{\bf h})$ as follows:
\begin{align}
		\phi_{R}({\bf g},{\bf h})&:=\sup_{\Gamma \geq 0} \ \  \phi_{R,\Gamma}({\bf g},{\bf h}) \label{eq:2o}
	\end{align}
Using the fact that $\Gamma\mapsto \phi_{R,\Gamma}(\g,\h)$ is increasing it can be further shown that (see Section \ref{sec:limsupproof}):
\bea\label{eq:2u}
\phi_{R}({\bf g},{\bf h}) = \lim_{\Gamma \rightarrow \infty} \ \  \phi_{R,\Gamma}({\bf g},{\bf h}).
\eea
Finally, starting from \eqref{eq:2o} and applying the min-max inequality \cite[Lem.~36.1]{Roc70} we see that
\bea\label{eq:supmin}
\phi_{R}(\g,\h) \leq \min_{\substack{\w\in\Sc_\w \\ \|\w\|_2\leq R}} \sup_{\ub\in\Sc_\ub} Y_{\w,\ub}.
\eea
Note that equality does not necessarily hold in \eqref{eq:supmin} since $Y_{\w,\ub}$ is not necessarily convex-concave.

As a last remark, we note that $\Phi(\G), \Phi_{R,\Gamma}(\G)$ and $\phi_{R,\Gamma}(\G)$ are all indexed by the dimensions $n$ and $d$. We have not explicitly accounted for this dependence in our notation for simplicity. 

\paragraph{Feasibility.} First, Theorem \ref{th:feasibility} shows how studying feasibility of the sequence of ``$(R,\Gamma)$-bounded" AO problems in \eqref{eq:AOpr} can determine the feasibility of the original PO problem (in which the constraint sets are potentially unbounded).

\begin{thm}[Feasibility]
	\label{th:feasibility}$~$ Recall the definitions of $\Phi(\G), \phi_{R,\Gamma}(\g,\h)$ and $\phi_R(\g,\h)$ in \eqref{eq:pr}, \eqref{eq:AOpr} and \eqref{eq:2u}, respectively. Assume that $(\w,\ub)\mapsto X_{\w,\ub}$ in \eqref{eq:POAO} is convex-concave and that the constraint sets $\Sc_\w,\Sc_\ub$ are convex (but not necessarily bounded). The following two statements hold true.
	
	\begin{enumerate}
		\item[(i)] Assume that { for any fixed $R,\Gamma\geq 0$} there exists a  positive constant $C$ (independent of $R$ and $\Gamma$) and a continuous increasing function $f:\R_+\rightarrow\R$ tending to infinity such that, for $n$ sufficiently large (independent of $R$ and $\Gamma$):
			\begin{equation}
				\phi_{R,\Gamma}({\bf g},{\bf h})\geq C\,f(\Gamma).
				\label{eq:cond_inf}
			\end{equation}
	Then, with probability $1$, for $n$ sufficiently large,
	$			\Phi({\bf G})=\infty$. 
\item[(ii)] Assume that there exists $k_0\in\mathbb{N}$ and a positive constant $C$ such that:
	\begin{equation}
		\mathbb{P}\left[ \left\{\phi_{k_0}({\bf g},{\bf h})\geq C\right\}, \ \text{i.o.}\right]=0.
			\label{eq:bnd_cond}
	\end{equation}
			Then, 
$			\mathbb{P}\left[\Phi({\bf G})=\infty, \ \text{i.o}\right]=0.$
	\end{enumerate}
\end{thm}

\begin{remark}\label{rem:feas} Condition \eqref{eq:bnd_cond} of Statement (ii) asks for a constant high-probability upper bound of $\phi_{k_0}(\g,\h)$ for some $k_0\in\mathbb{N}$. In view of \eqref{eq:supmin} it suffices to upper bound the problem on the RHS, i.e., to show that:
\bea\label{eq:bnd_cond_easy}
\mathbb{P}\left[ \Big\{ \min_{\substack{\w\in\Sc_\w \\ \|\w\|_2\leq R}} \sup_{\ub\in\Sc_\ub} Y_{\w,\ub} \geq C \Big\}, \ \text{i.o.}\right]=0.
\eea
This observation can be useful as it is often the case that the ``unbounded" optimization over $\ub$ is easy to perform. 
\end{remark}

\paragraph{Optimal cost and optimal solution.} Theorem \ref{th:characterization} below shows that if the sequence of AO problems are all feasible, then their limiting cost determines the optimal cost of the original PO problem. Compared to the CGMT, note that it is not required that $\Sc_\w$ and $\Sc_\ub$ are compact and also the convergence results hold in almost-sure sense.

\begin{thm}
	\label{th:characterization}
	Assume the same notation as in Theorem \ref{th:feasibility}. Further 
	let $\mathcal{S}$ be an  open subset of $\mathcal{S}_{\bf w}$ and $\mathcal{S}^{c}=\mathcal{S}_{\bf w}\backslash \mathcal{S}$. Denote by $\widetilde{\phi}_{R,\Gamma}({\bf g},{\bf h})$ the optimal cost of  \eqref{eq:AOpr} when the minimization over ${\bf w}$ is now further constrained over ${\bf w}\in\mathcal{S}^c$. Define $\widetilde{\phi}_{R}({\bf g},{\bf h})$ in a similar way to \eqref{eq:2o}.
	Assume that there exists $\overline{\phi}$ and $k_0\in\mathbb{N}$ such that  the following statements hold true:
 \begin{subequations}\label{eq:conditions}
\begin{align}
	&\text{For any}~ \eps>0:\quad\mathbb{P}\left[\cup_{k=k_0}^\infty\Big\{ \phi_{k}({\bf g},{\bf h}) \leq \overline{\phi}-\epsilon\Big\}, ~~\text{i.o.}\right]=0,\label{eq:condition1}\\
	&\text{For any}~ \eps>0:\quad\mathbb{P}\left[\Big\{\phi_{k_0}({\bf g},{\bf h}) \geq \overline{\phi}+\epsilon\Big\}, ~~\text{i.o.}\right]=0\label{eq:condition2},\\
	&\text{There exists}~ \zeta>0:\quad\mathbb{P}\left[\cup_{k=k_0}^\infty\Big\{ \widetilde\phi_{k}({\bf g},{\bf h}) \leq \overline{\phi}+\zeta\Big\}, ~~\text{i.o.}\right]=0.\label{eq:condition3}
\end{align}
\end{subequations}
	Then,
	\begin{equation}
	\Phi({\bf G})\ras  \overline{\phi}.
		\label{eq:almost_sure}
	\end{equation}
Furthermore, letting $\w_\Phi$ denote a minimizer of the PO in \eqref{eq:pr}, it also holds that
	\begin{equation}
	\mathbb{P}\left[{\bf w}_{\Phi}\in \mathcal{S}, \ \ \text{for sufficiently large} \  n\right]=1.
		\label{eq:property}
	\end{equation}
\end{thm}

\vp
While the constraint set $\Sc_\ub$ is often unbounded, it is common that the constraint set $\Sc_\w$ can be assumed to be bounded. When this is the case, the conditions of Theorem \ref{th:characterization} can be simplified as shown in the corollary below, which is an immediate consequence of Theorem \ref{th:characterization}.

\begin{cor}
	\label{cor:characterization}
	Consider the same setting as in Theorem \ref{th:characterization} and further assume that there exists $k_0\in\mathbb{N}$ such that $\mathcal{S}_w\subset \left\{{\bf w} \  | \ \|{\bf w}\|_2\leq k_0 \right\}$. 
	Assume that there exists $\overline{\phi}$ such that  the following statements hold true:
 \begin{subequations}\label{eq:conditions1}
\begin{align}
	&\text{For any}~ \eps>0:\quad\mathbb{P}\left[\Big\{\phi_{k_0}({\bf g},{\bf h}) \leq \overline{\phi}-\epsilon\Big\}, ~~\text{i.o.}\right]=0,\label{eq:condition1_1}\\
	&\text{For any}~ \eps>0:\quad\mathbb{P}\left[\Big\{\phi_{k_0}({\bf g},{\bf h}) \geq \overline{\phi}+\epsilon\Big\}, ~~\text{i.o.}\right]=0\label{eq:condition2_11},\\
	&\text{There exists}~ \zeta>0:\quad\mathbb{P}\left[\Big\{ \widetilde\phi_{k_0}({\bf g},{\bf h}) \leq \overline{\phi}+\zeta\Big\}, ~~\text{i.o.}\right]=0.\label{eq:condition3_1}
\end{align}
\end{subequations}
	Then,
	\begin{equation}
	\Phi({\bf G})\ras  \overline{\phi}.
		\label{eq:almost_sure_1}
	\end{equation}
Furthermore, letting $\w_\Phi$ denote a minimizer in \eqref{eq:pr}, it also holds that
	\begin{equation}
	\mathbb{P}\left[{\bf w}_{\Phi}\in \mathcal{S}, \ \ \text{for sufficiently large} \  n\right]=1.
		\label{eq:property_1}
	\end{equation}
 
\end{cor}

\begin{remark}[A  meta-theorem: connection to the ``unbounded" AO]\label{rem:minsup} Similar to Remark \ref{rem:feas},  it suffices for  \eqref{eq:condition2_11} to hold that the following (often easier to check) condition is true:
\begin{align}
\mathbb{P}\Big[ \Big\{ \min_{\substack{\w\in\Sc_\w \\ \|\w\|_2\leq k_0}} \sup_{\ub\in\Sc_\ub} Y_{\w,\ub} \geq \overline{\phi} + \eps \Big\}, \ \text{i.o.}\Big]=0. \label{eq:minsupAO}
\end{align}
	Further as noted in Remark \ref{rem:feas} this formulation can be useful as it is often the case that the ``unbounded" optimization over $\ub$ in \eqref{eq:minsupAO} is easy to perform. What allows us to replace \label{eq:condition2_1} with \eqref{eq:minsupAO} is the min-max inequality in \eqref{eq:supmin}. If equality were true when ${\bf w}$ is constrained to either $\Sc_w$ or the (possibly) non-convex set $\mathcal{S}^c:=\mathcal{S}_{\bf w}\backslash \mathcal{S}$,   
	then, \eqref{eq:condition1_1} and \eqref{eq:condition3_1} would be equivalent to the following:
	\begin{subequations}\label{eq:minsupAO2}
\begin{align}
&	\mathbb{P}\Big[ \Big\{ \min_{\substack{\w\in{\mathcal{S}_\w} \\ \|\w\|_2\leq k_0}} \sup_{\ub\in\Sc_\ub} Y_{\w,\ub} \leq \overline{\phi} + \eps \Big\}, \ \text{i.o.}\Big]=0, \label{eq:unbounded_AO_rel}\\
&		\mathbb{P}\Big[ \Big\{ \min_{\substack{\w\in\Sc^c \\ \|\w\|_2\leq k_0}} \sup_{\ub\in\Sc_\ub} Y_{\w,\ub} \leq \overline{\phi} + \zeta \Big\}, \ \text{i.o.}\Big]=0. \label{eq:unbounded_AO_rel2}
\end{align}
\end{subequations}
Similar to \eqref{eq:minsupAO}, this formulation is often preferred in practice as it relates to the ``easier" unbounded AO. It turns out that a sufficient condition for the equivalence described above is that
the objective function $Y_{\w,\ub}$ of the AO is convex in $\w$. To see this, note that under this convexity condition, $$\mathcal{G}_\Gamma(\w):= \displaystyle{\max_{{{\bf u}\in\mathcal{S}_{\bf u},\, \|{\bf u}\|_2\leq \Gamma}}} Y_{\w,\ub},$$ is convex in ${\bf w}$. Moreover, $\Gamma\mapsto\mathcal{G}_\Gamma(\w)$ converges (point-wise in $\w$) to $\mathcal{G}(\w):= {\sup_{\substack{{\bf u}\in\mathcal{S}_{\bf u}}}} Y_{\w,\ub}$. But point-wise convergence of convex functions implies uniform convergence over compact sets \cite[Thm.~10.8]{Roc70}, i.e., $\mathcal{G}_\Gamma$ converges \emph{uniformly}  to $\mathcal{G}$ over the bounded sets  $\tilde{\mathcal{S}}\cap \left\{\w \ \ | \ \ \|\w\|_2\leq k_0\right\}$, where we denote by $\tilde{\mathcal{S}}$ either of the two sets $\Sc_\w$ or $\Sc^c$ in \eqref{eq:minsupAO2}. Based on this and recalling \eqref{eq:2u}, it follows that 
$$
\phi_{k_0}(\g,\h)= \lim_{\Gamma\to \infty} \min_{\substack{\w\in\tilde{\mathcal{S}} \\ \|\w\|_2\leq k_0}} \max_{\substack{\ub\in\Sc_\ub\\ \|{\bf u}\|_2\leq \Gamma}} Y_{\w,\ub}\,
$$
 is equal to 
 $$
 \min_{\substack{\w\in\tilde{\mathcal{S}} \\ \|\w\|_2\leq k_0}} \sup_{{\ub\in\Sc_\ub}} Y_{\w,\ub}\,,
 $$
provided that $Y_{\w,\ub}$ is convex. 
	Unfortunately, the form of the AO in \eqref{eq:AO_obj} is \emph{not} convex in $\w$ as is.  
	Nevertheless, the analysis of the AO (as prescribed in \cite{Master}) often leads to an equivalent ``scalarized version" (aka optimization over only a few scalar variables) which possesses the desired convexity property. From the discussion above, one may then use \eqref{eq:minsupAO2} towards applying Corollary \ref{cor:characterization} to the convex scalarized AO. 
\end{remark}

\begin{remark}[On applying Corollary \ref{cor:characterization}]\label{rem:recipeCor}
In practice Corollary \ref{cor:characterization} can be used in conjunction with Lemma \ref{lemma:5} in Appendix \ref{app:tech} to handle the unboundedness of the set $\mathcal{S}_{{\bf w}}$. At a high level, the recipe is as follows. Start with the following ``single-bounded version" of the PO
	\begin{equation}
	\Phi_{R}({\bf G}):=\min_{\substack{{\bf w}\in\mathcal{S}_{\bf w}\\ \|{\bf w}\|_2\leq R}} \sup_{{\bf u}\in\mathcal{S}_{\bf u}} X_{{\bf w},{\bf u}},
		\label{eq:single_PO}
	\end{equation}
	which allows the use of Corollary \ref{cor:characterization}. Next, suppose that applying the corollary allows to show that any  minimizer $\hat{\bf w}_R$ of \eqref{eq:single_PO} satisfies $\|{\bf w}_R\|\ras \theta^\star$ for some $\theta^\star<R$ (and independent of $R$). Then, by Lemma \ref{lemma:5}, it follows that any solution $\hat{\bf w}_\Phi$ of \eqref{eq:pr} must also satisfy $\|\hat{\bf w}_\Phi\|\ras \theta^\star$. Hence, this validates the equivalence of \eqref{eq:single_PO} for sufficiently large $R$ (e.g., any $R\geq2\theta^\star$) to the unbounded PO in \eqref{eq:pr}. We apply this recipe in Appendix \ref{sec:proof_ML}, which studies the statistical properties of logistic-loss minimization in the underparameterized regime  $\kappa<\kappa_\star$. 
\end{remark}

\section{Proofs for Appendix \ref{sec:CGMT_new}}\label{sec:CGMT_new_proofs}

\subsection{Proof of \eqref{eq:2o} = \eqref{eq:2u}}\label{sec:limsupproof}

We prove \eqref{eq:2u} using the definition in \eqref{eq:2o}. For convenience, we drop here the arguments $(\g,\h)$ to lighten notation. We need to consider separately the cases $\phi_R=\infty$ and $\phi_R\neq \infty$. If $\phi_{{R}}=\infty$, then necessarily $\lim_{\Gamma\to\infty} \phi_{R,\Gamma}=\infty$, since otherwise, $\phi_{{R},\Gamma}$ would be bounded and so would be $\sup_{\Gamma\geq 0} \phi_{R,\Gamma}$, contradicting the fact that $\phi_{R}=\infty$. On the other hand, if $\phi_R\neq \infty$, from the ``$\epsilon$-definition'' of the supremum, for $\tilde{\epsilon}>0$, there exists $\Gamma_0\geq 0$ such that $\phi_R\leq \phi_{R,{\Gamma_0}}+\tilde{\epsilon}$. As $\Gamma\mapsto \phi_{R,\Gamma}$ is increasing, for any $\Gamma\geq \Gamma_0$,
$$
\phi_{R,\Gamma}\leq \phi_{R}\leq \phi_{R,\Gamma_0}+\tilde{\epsilon}\leq \phi_{R,\Gamma}+\tilde{\epsilon}.
$$
Letting $\Gamma \to\infty$ in the left and right-hand sides of the above inequality, we obtain:
$$
\lim_{\Gamma\to\infty}\phi_{R,\Gamma}\leq \phi_{R}\leq \lim_{\Gamma\to\infty}\phi_{R,\Gamma}+\tilde{\epsilon}.
$$
As $\tilde{\epsilon}$ may be chosen arbitrarily small, we conclude that $\phi_R=\lim_{\Gamma\to\infty} \phi_{R,\Gamma}$, as desired. 

\vp
For the rest of the proofs, it is also convenient to note that, due to the convex-concave property of the objective in the PO,   the  single-bounded version of the PO introduced in \eqref{eq:single_PO} writes also as follows:
	\begin{align}
		\Phi_R({\bf G})&:=\sup_{\Gamma\geq 0} \ \  \Phi_{R,\Gamma}({\bf G}). \label{eq:1o}
	\end{align}
Similar to what was shown above, we can equivalently write:
\begin{align}
	\Phi_R({\bf G})&=\lim_{\Gamma\to\infty}\min_{\substack{{\bf w}\in\mathcal{S}_{\bf w}\\ \|{\bf w}\|_2\leq R}}  \max_{\substack{{\bf u}\in\mathcal{S}_{\bf u}\\ \|{\bf u}\|_2\leq \Gamma}} X_{{\bf w},{\bf u}}\label{eq:1u}\,.
\end{align}

\subsection{Proof of Theorem \ref{th:feasibility}}
	\subsubsection{Proof of the statement (i)}  
	From \eqref{eq:PO_original} and the ``$\epsilon$-definition" of the infimum, if $\Phi({\bf G})\neq \infty$, for $\epsilon>0$ sufficiently small, there exists $k\in\mathbb{N}$ such that:
	$$
	\Phi({\bf G})\geq \Phi_{k,m}({\bf G})-\epsilon,  \ \ \forall \ \ m\in\mathbb{N}.
	$$
	Hence, for any fixed $x>0$,
	\bea
	\mathbb{P}\left[\left\{\Phi({\bf G})\leq x\right\}\right]\leq \mathbb{P}\left[\cup_{k=1}^\infty\cap_{m=1}^\infty \left\{\Phi_{k,m}({\bf G})\leq x+\epsilon\right\}\right],\nn
	\eea
	For $k\in\mathbb{N}^\star$, the events $\mathcal{E}_k=\left\{\cap_{m=1}^\infty\left\{\Phi_{k,m}({\bf G})\leq x+\epsilon\right\}\right\}$ form an increasing sequence of events, hence, 
	$$
	\mathbb{P}\left[\cup_{k=1}^\infty \mathcal{E}_k \right]=\lim_{k\to\infty} \mathbb{P}(\mathcal{E}_k).
	$$
	In a similar way, $\Phi_{k,m}({\bf G})\geq \Phi_{k,m-1}({\bf G})$, hence the sequence of events:
	$$
	\mathcal{E}_{k,m}=\left\{\Phi_{k,m}({\bf G})\leq x+\epsilon\right\},
	$$
	is decreasing, thus yielding,
	$$
	\mathbb{P}\left[\cap_{m=1}^\infty \mathcal{E}_{k,m}\right]=\lim_{m\to\infty} \mathbb{P}(\mathcal{E}_{k,m})\,.
	$$
	Hence, we have:
	$$
	\mathbb{P}\left[\left\{\Phi({\bf G})\leq x\right\}\right]\leq \lim_{k\to\infty}\lim_{m\to\infty} \mathbb{P}\left[\Phi_{k,m}({\bf G})\leq x+\epsilon\right]\,.
	$$
	We may now use Gordon's inequality \eqref{eq:gmt}
	$$
\mathbb{P}\left[\Phi_{k,m}({\bf G})\leq x+\epsilon\right]\leq 2\mathbb{P}\left[\phi_{k,m}({\bf g},{\bf h})\leq x+\epsilon\right],
	$$
	to conclude that
	\bea\label{eq:Ann}
	\lim_{k\to\infty}\lim_{m\to\infty}\mathbb{P}\left[\Phi_{k,m}({\bf G})\leq x+\epsilon\right] \leq 2\lim_{k\to\infty}\lim_{m\to\infty} \mathbb{P}\left[\phi_{k,m}({\bf g},{\bf h})\leq x+\epsilon\right]=2\mathbb{P}\left[\cup_{k=1}^\infty\left\{\cap_{m=1}^\infty\left\{\phi_{k,m}({\bf g},{\bf h})\leq x+\epsilon\right\}\right\}\right].
	\eea
Consider now the event $\mathcal{E}$ defined as follows:
	$$
	\mathcal{E}:=\left\{\cup_{k=1}^\infty \cap_{m=1}^\infty \left\{\phi_{k,m}({\bf g},{\bf h})\geq Cf(m)\right\}, \ \ \text{for} \ \ n\  \ \text{sufficiently large}\right\}.
	$$
	From \eqref{eq:cond_inf}, it holds that $\mathbb{P}[\mathcal{E}]=1$, since by assumption: when  $n$ is sufficiently large, for every $k$ and $m$,  $\phi_{k,m}({\bf g},{\bf h})\geq Cf(m)$.  
Therefore, the sequence of events $$A_n:=\left\{\cup_{k=1}^\infty\left\{\cap_{m=1}^\infty\left\{\phi_{k,m}({\bf g},{\bf h})\leq x+\epsilon\right\}\right\}\right\}\,,$$ in \eqref{eq:Ann} does not occur infinitely often. To see this, note that since $\lim_{m\to\infty}f(m)=\infty$, there exists $m_0$ such that for all $m\geq m_0$, $Cf(m)\geq x+\epsilon$.  
	Moreover, since $A_n$ are independent, each event being generated by independent realizations of ${\bf h}$ and ${\bf g}$, by the converse of Borel Cantelli lemma, we find that $\sum_{n=1}^\infty  \mathbb{P}(A_n)<\infty$.  Using this fact in combination with the inequalities above, we obtain that:
	$
	\sum_{n=1}^\infty \mathbb{P}\left[\Phi({\bf G})\leq x\right]<\infty.
	$
	Therefore, using the Borel-Cantelli lemma, we conclude for any $x>0$, $\Phi({\bf G})>x$ for $n$ sufficiently large. Hence $\Phi({\bf G})=\infty$, a.s.
\vp

\subsubsection{Proof of statement (ii)}
Recall that 
 $
 \Phi_{k_0}({\bf G})=\lim_{m\to\infty} \Phi_{k_0,m}({\bf G}) = \sup_{m\to\infty}  \Phi_{k_0,m}({\bf G}),
$
and so, $\Phi({\bf G})\leq \Phi_{k_0}({\bf G})$. Therefore, for any $1>\epsilon>0$ we have:
	\begin{align}
		\mathbb{P}\left[\Phi({\bf G})\geq C+1\right]\leq \mathbb{P}\left[\Phi_{k_0}({\bf G})\geq C+1\right]&=\mathbb{P}\left[\lim_{m\to\infty}\Phi_{k_0,m}({\bf G})\geq C+1 \right]\nn\\
		&\leq \mathbb{P}\left[\cup_{m=1}^\infty \left\{\Phi_{k_0,m}({\bf G})\geq C+1-\epsilon\right\}\right]\label{eq:whyeps}\\
		&=\lim_{m\to\infty}\mathbb{P}\left[\left\{\Phi_{k_0,m}({\bf G})\geq C+1-\epsilon)\right\}\right],\label{eq:last}
	\end{align}
	where \eqref{eq:last} follows from the fact that $\left\{\left\{\Phi_{k_0,m}({\bf G})\geq C+1-\epsilon\right\}\right\}_{m\in\mathbb{N}}$ forms an increasing sequence of events. Also, \eqref{eq:whyeps} is true as follows. If $\lim_{m\to\infty}\Phi_{k_0,m}({\bf G})<\infty$, then, $\lim_{m\to\infty}\Phi_{k_0,m}({\bf G})\leq \Phi_{k_0,m_0}({\bf G})+\epsilon$ for sufficiently large $m_0$. Else, if $\lim_{m\to\infty}\Phi_{k_0,m}({\bf G})=\infty$, then necessarily $\Phi_{k_0,m_0}({\bf G})\geq C+1-\epsilon$ sufficiently large $m_0$.
	
	 But, from the CGMT \eqref{eq:cgmt},
	$$
	\mathbb{P}\left[\left\{\Phi_{k_0,m}({\bf G})\geq C+1-\epsilon)\right\}\right]\leq 2\mathbb{P}\left[\phi_{k_0,m}({\bf g},{\bf h})\geq C+1-\epsilon\right].
	$$
	Hence, 
	\begin{align*}
		\mathbb{P}\left[\left\{\Phi({\bf G)}\geq C+1)\right\}\right]\leq 2\lim_{m\to\infty}\mathbb{P}\left[\phi_{k_0,m}({\bf g},{\bf h})\geq C+1-\epsilon\right]&=2\mathbb{P}\left[\cup_{m=1}^\infty \left\{\phi_{k_0,m}({\bf g},{\bf h})\geq C+1-\epsilon\right\}\right]\\
		&\leq 2\mathbb{P}\left[\cup_{m=1}^\infty \left\{\phi_{k_0,m}({\bf g},{\bf h})\geq C\right\}\right]\\
		&\leq 2\mathbb{P}\left[\lim_{m\to\infty} \phi_{k_0,m}({\bf g},{\bf h})\geq C\right]= 2\mathbb{P}\left[\left\{\phi_{k_0}({\bf g},{\bf h})\geq C\right\}\right].
	\end{align*}
 The equality in the first line above follows because $\left\{\left\{\phi_{k_0,m}({\g,\h})\geq C+1-\epsilon\right\}\right\}_{m\in\mathbb{N}}$ forms an increasing sequence of events. For the inequality in the second line, recall that $0<\eps<1$. The last inequality is again because $\phi_{k_0,m}({\g,\h})$ is increasing in $m$.
 
It follows from the inequality above together with \eqref{eq:bnd_cond} and applying he converse of the Borel-Cantelli lemma that $\sum_{n=1}^\infty \mathbb{P}\left[\lim_{m\to\infty} \left\{\phi_{k_0,m}({\bf g},{\bf h})\geq C\right\}\right]<\infty$.  Hence, $\sum_{n=1}^\infty\mathbb{P}\left[\left\{\Phi({\bf G})\geq C+1)\right\}\right]<\infty$ and with probability $1$ for sufficiently large $n$, 
	$
	\Phi({\bf G})\neq \infty,
	$
	as desired.

\subsection{Proof of Theorem \ref{th:characterization}}
	The proof consists in showing that each one of \eqref{eq:condition1}, \eqref{eq:condition2} and \eqref{eq:condition3} imply one of the following three statements respectively:
\begin{subequations}\label{eq:conditions}
\begin{align}
	&\text{For any}~ \eps>0:\quad\mathbb{P}\left[\Big\{\Phi({\bf G}) \leq \overline{\phi}-2\epsilon\Big\}, ~~\text{i.o.}\right]=0,\label{eq:condition1_pr}\\
	&\text{For any}~ \eps>0:\quad\mathbb{P}\left[\Big\{\Phi({\bf G}) \geq \overline{\phi}+{2\epsilon}\Big\}, ~~\text{i.o.}\right]=0\label{eq:condition2_pr},\\
	&\text{There exists}~ \zeta>0:\quad\mathbb{P}\left[\Big\{\widetilde{\Phi}({\bf G}) \leq \overline{\phi}+2\zeta\Big\}, ~~\text{i.o.}\right]=0,\label{eq:condition3_pr}
\end{align}
\end{subequations}
	where $\widetilde{\Phi}({\bf G})$ is the optimal cost of the optimization in \eqref{eq:pr} where the minimization over ${\bf w}$ is constrained over ${\bf w}\in\mathcal{S}^{c}$. 
	Clearly the combination of \eqref{eq:condition1_pr} and \eqref{eq:condition2_pr} yields \eqref{eq:almost_sure}. To prove how \eqref{eq:property} follows from \eqref{eq:condition3_pr}, consider the event:
	$$
	\mathcal{E}=\left\{\widetilde{\Phi}({\bf G})\geq \overline{\phi}+2\zeta, \ \ {\Phi}({\bf G})\leq \overline{\phi}+\zeta\right\}.
	$$
	In this event, it is easy to see that $\widetilde{\Phi}({\bf G}) > \Phi({\bf G})$. Therefore, ${\bf w}_{\Phi}\in\mathcal{S}$. From \eqref{eq:condition1_pr} and \eqref{eq:condition3_pr}, $\mathbb{P}(\mathcal{E}, \text{i.o.})=1$, which yields $\mathbb{P}\left[{\bf w}_\Phi\in\mathcal{S}, \ \text{i.o} \right]=1$.
	\noindent	 Thus, it remains to prove that \eqref{eq:condition1_pr}, \eqref{eq:condition2_pr} and \eqref{eq:condition3_pr} are by-products of \eqref{eq:condition1}, \eqref{eq:condition2} and \eqref{eq:condition3}, respectively . 

	\vp
\noindent {\bf Proof of \eqref{eq:condition1_pr}:}
	In the event $\left\{\Phi({\bf G})\leq \overline{\phi}-2\epsilon\right\}$, $\Phi({\bf G})\neq \infty$. 
	From the ``$\epsilon$-definition: of the infimum, there exists $\tilde{k}\in\mathbb{N}$ such that:
	$$
	\Phi({\bf G})\geq \Phi_{k_0}({\bf G})-\epsilon\geq \Phi_{k}({\bf G})-\epsilon,  \ \ \forall \ \text{integer} \  \  k\geq \tilde{k}.
	$$
	Since $\Phi_{k}({\bf G})\geq \Phi_{k,m}({\bf G})$ for all $m\in\mathbb{N}$, we thus have, 
	$$
	\Phi({\bf G})\geq \Phi_{k,m}({\bf G})-\epsilon  \ \ \forall \ \text{integer} \  \  k\geq \tilde{k} \ \text{and} \ m\in\mathbb{N}.
	$$
	Hence,
	$$
	\mathbb{P}\left[\Phi({\bf G})\leq \overline{\phi}-2\epsilon  \ \right]\leq \mathbb{P}\left[\cup_{k=1}^\infty \cap_{m=1}^\infty \left\{\Phi_{k,m}({\bf G})\leq \overline{\phi}-\epsilon\right\}\right]=\lim_{k\to\infty}\lim_{m\to\infty} \mathbb{P}\left[\Phi_{k,m}({\bf G})\leq \overline{\phi}-\epsilon \right],
	$$
	where the last equality follows from the fact that the sequence $\left\{\mathcal{E}_k\right\}_{k\in\mathbb{N}^*}$ with $\mathcal{E}_k:=\left\{\cap_{m=1}^\infty  \Phi_{k,m}({\bf G})\leq \overline{\phi}-\epsilon\right\}$ forms an increasing sequence of events, while the sequence of events $\left\{\mathcal{E}_{k,m}\right\}_{m\in\mathbb{N}^*}$ with $\mathcal{E}_{k,m}:=\left\{\Phi_{k,m}({\bf G})\leq \overline{\phi}-\epsilon\right\}$ is decreasing.
	Using Gordon's inequality \eqref{eq:gmt}, 
	\begin{align}
		\lim_{k\to\infty}\lim_{m\to\infty}\mathbb{P}\left[\left\{\Phi_{k,m}({\bf G})\leq \overline{\phi}-\epsilon\right\}\right]\leq 2\lim_{k\to\infty}\lim_{m\to\infty}\mathbb{P}\left[\phi_{k,m}({\bf g},{\bf h})\leq \overline{\phi}-\epsilon\right]&=2\mathbb{P}\left[\cup_{k=1}^\infty \cap_{m=1}^{\infty} \left\{\phi_{k,m}({\bf g},{\bf h})\leq \overline{\phi}-\epsilon \right\}\right] \nn \\
		&\leq2\mathbb{P}\left[\cup_{k=1}^\infty  \left\{\lim_{m\to\infty}\phi_{k,m}({\bf g},{\bf h})\leq \overline{\phi}-\epsilon\right\} \right]\nn\\
		&=2\mathbb{P}\left[\cup_{k=k_0}^\infty \left\{ \lim_{m\to\infty}\phi_{k,m}({\bf g},{\bf h})\leq \overline{\phi}-\epsilon \right\}\right],\nn
	\end{align}
	where the last equality follows from the fact that $\left\{\left\{\lim_{m\to\infty}\phi_{k,m}({\bf g},{\bf h})\leq \overline{\phi}-\epsilon\right\} \right\}_{k\in\mathbb{N}^*}$ forms an increasing sequence of events.

\vp
	Now from \eqref{eq:condition1}, the sequence of events $A_n:=\left\{\cup_{k=k_0}^\infty \left\{\lim_{m\to\infty}\phi_{k,m}({\bf g},{\bf h})\leq \overline{\phi}-\epsilon \right\}\right\}_{n\in\mathbb{N}}$ does not occur infinitely often. Since $\left\{A_n\right\}$ are independent, each event being generated by independent realizations of ${\bf h}$ and ${\bf g}$, by the converse of Borel-Cantelli lemma, $\sum_{n=1}^\infty \mathbb{P}(A_n) <\infty$. Keeping track of the inequalities above and using this fact, we find that:
	$$
	\sum_{n=1}^\infty \mathbb{P}\left[\left\{\Phi({\bf G})\leq \overline{\phi}-\epsilon\right\}\right]<\infty. 
	$$
	Hence \eqref{eq:condition1_pr} holds true.
	\vp

	\noindent {\bf Proof of \eqref{eq:condition2_pr}:} We start by noticing that 
	$$\Phi({\bf G})\geq \overline{\phi}+\epsilon \Longrightarrow \Phi_{k_0}({\bf G})\geq \overline{\phi}+\epsilon.$$
Also, recall that 	$
 \Phi_{k_0}({\bf G})= \lim_{m\to\infty}\Phi_{k_0,m}({\bf G}).
	$ 
	With this at hand, we obtain
	\begin{align}
		\mathbb{P}\left[\Phi({\bf G})\geq \overline{\phi}+2\epsilon\right]\leq \mathbb{P}\left[\Phi_{k_0}({\bf G}) \geq \overline{\phi}+2\epsilon\right]&= \mathbb{P}\left[\lim_{m\to\infty}\Phi_{k_0,m}({\bf G}) \geq \overline{\phi}+2\epsilon\right]\nn\\
		&\leq \mathbb{P}\left[\cup_{m=1}^\infty \left\{\Phi_{k_0,m}({\bf G})\geq \overline{\phi}+\epsilon\right\}\right] \label{eq:why2eps}\\
		&=\lim_{m\to\infty} \mathbb{P}\left[\left\{\Phi_{k_0,m}({\bf G})\geq \overline{\phi}+\epsilon\right\}\right]\,, \label{eq:exp1}
	\end{align}
		where: \eqref{eq:why2eps} follows by the same argument explaining \eqref{eq:whyeps}; and, \eqref{eq:exp1} follows from the fact that $\left\{\left\{\Phi_{k_0,m}({\bf G})\geq \overline{\phi}+\epsilon\right\}\right\}_{m\in\mathbb{N}^*}$ is an increasing sequence of events. 
%
%
	
	But, from the CGMT \eqref{eq:cgmt},
	$$
\mathbb{P}\left[\left\{\Phi_{k_0,m}({\bf G})\geq \overline{\phi}+\epsilon\right\}\right]\leq 2\mathbb{P}\left[\phi_{k_0,m}({\bf g},{\bf h})\geq \overline{\phi}+\epsilon\right].
	$$ Hence,
	\begin{align}
		\mathbb{P}\left[\left\{\Phi({\bf G})\geq \overline{\phi}+\epsilon\right\}\right]\leq 2\lim_{m\to\infty} \mathbb{P}\left[\phi_{k_0,m}({\bf g},{\bf h})\geq \overline{\phi}+\epsilon\right]&=2\mathbb{P}\left[\cup_{m=1}^\infty\left\{\phi_{k_0,m}({\bf g},{\bf h})\geq \overline{\phi}+\epsilon\right\}\right]\nonumber\\
		&\leq 2\mathbb{P}\left[\Big\{\lim_{m\to\infty}\phi_{k_0,m}({\bf g},{\bf h})\geq \overline{\phi}+\epsilon\Big\}\right],
		\label{eq:pr22}
	\end{align}
where we used once more the fact that	$\left\{\left\{\phi_{k_0,m}({\g,\h})\geq \overline{\phi}+\epsilon \right\}\right\}_{m\in\mathbb{N}}$ forms an increasing sequence of events.

	It follows from \eqref{eq:condition2} along with the converse of Borel-Cantelli lemma that $\sum_{n=1}^\infty \mathbb{P}\left[\left\{\lim_{m\to\infty}\phi_{k_0,m}({\bf g},{\bf h})\geq \overline{\phi}+\epsilon\right\}\right]<\infty$. Combining this with \eqref{eq:pr22} yields to the desired, namely \eqref{eq:condition2_pr}.

\vp
	\noindent {\bf Proof of \eqref{eq:condition3_pr}:}  The proof of \eqref{eq:condition3_pr} is identical to the proof of  \eqref{eq:condition1_pr}; thus, it is omitted for brevity.    


\section{Hard-margin SVM}\label{sec:SVM}


In this section we analyze the statistical properties of hard-margin SVM. The analysis is based on the two theorems of Section \ref{sec:CGMT_new}. First, in Section \ref{sec:SVM_AO} we show how to bring the SVM minimization in the form of a PO, we write the equivalent AO and we simplify it to a scalar optimization problem. Next, in Section \ref{sec:proof_PT} we apply Theorem \ref{th:feasibility} to study the feasibility of SVM. As mentioned, the SVM problem is feasible iff the data are linearly separable. Hence, this section proves Proposition \ref{propo:PT}. In the regime $\kappa>\kappa_\star$ where the problem is feasible, we apply Theorem \ref{th:characterization} to prove Proposition \ref{propo:SVM} in Section \ref{sec:SVM_prop}. Throughout, we focus on the logistic data model \eqref{eq:yi_log}. Treatment for the GM model \eqref{eq:yi_GM} is almost identical (if not simpler). To avoid repetitions, we only sketch the main part where the two models need different treatment, which is the formulation of the PO and of the equivalent AO. We present this in Section \ref{sec:sketch_GM}.

\subsection{The Auxiliary Problem of Hard-margin SVM}\label{sec:SVM_AO}
\noindent\textbf{Identifying the PO.} The max-margin solution is obtained by solving the following problem:
\begin{equation}\label{eq:SVM_app}
\betabh = \arg\min_{\betab} \|\betab\|_2 \qquad \text{subject to}\quad y_i\w_i^T\betab \geq 1,
\end{equation}
and is feasible only when the training data is separable. The minimization in  \eqref{eq:SVM_app} is equivalent to solving: 
\begin{equation}\label{eq:origprob}
\min_{\betab} \max_{u_i \leq 0} ~\|\betab\|_2^2 + \frac{1}{n}\sum_{i=1}^n u_i\left( y_i \w_i^T\betab - 1 \right).
\end{equation}
Based on the rotational invariance of the Gaussian measure, we may assume without loss of generality that $\betab_0=[s,0,\ldots,0]^T,$
where only the first coordinate is nonzero. For convenience, we also write
 \bea\label{eq:notation_proof}
 \w_i = [w_i,\mathbf{v}_i^T]^T\quad\text{ and }\quad\betab = [\mu, \tilde{\betab}^T]^T.
 \eea
 In this new notation, the response variables are expressed as follows:
\bea\label{eq:y_pf_SVM}
y_i \sim \Rad{\rhod(s\,w_i+\sigma z_i)},~~~z_i\sim\Nn(0,1).
\eea
Further using this notation and considering the change of variable $\tilde{u}_i=\frac{{u}_i}{\sqrt{n}}$, \eqref{eq:origprob} becomes
\begin{equation}
	\min_{\mu, \tilde{\betab}}~ \max_{\tilde{u}_i \leq 0}~ \mu^2 + \|\tilde{\betab}\|_2^2 + \frac{1}{\sqrt{n}}\sum_{i=1}^n \tilde{u}_iy_iw_i\mu + \frac{1}{\sqrt{n}}\sum_{i=1}^{n} \tilde{u}_iy_i\mathbf{v}_i^T\tilde{\betab} - \frac{1}{\sqrt{n}}\sum_{i=1}^n \tilde{u}_i.
\end{equation}
Letting $\mathbf{1}$ be an $n$-dimensional vector with elements all ones,	$\mathbf{D}_y = \text{diag}(y_1,\ldots,y_n)$,  $\tilde{\mathbf{u}} = [\tilde{u}_1,\ldots,\tilde{u}_n]^T$, and   writing  $$\mathbf{W} = \left[\begin{matrix} \w_1^T \\ \vdots \\ \w_n^T \end{matrix}\right] = \left[ \w_{n\times 1} | \mathbf{V}_{n\times (p-1)} \right],$$
leads to the following optimization problem
\begin{equation}
	\Phi^{(n)}:=\min_{\mu, \tilde{\betab}}~ \max_{\tilde{u}_i \leq 0}~ \frac{1}{\sqrt{n}}\tilde{\mathbf{{u}}}^T\mathbf{D}_y\mathbf{V}\tilde{\betab} + \frac{1}{\sqrt{n}}\left(\tilde{\mathbf{{u}}}^T\mathbf{D}_y\w\right)\mu - \frac{1}{\sqrt{n}}\tilde{\mathbf{u}}^T\mathbf{1} + \mu^2 + \|\tilde{{\betab}}\|_2^2,
	\label{eq:primary}
\end{equation}
which has the same form of a primary optimization (PO) problem as required  by the CGMT (cf. \eqref{eq:PO_loc}), with the single exception that the feasibility sets are not compact.  To solve this issue, we pursue the approach presented in Section \ref{sec:CGMT_new}.


\vp
\noindent\textbf{Forming the AO.} 
 In view of \eqref{eq:AOpr}, we associate the PO in \eqref{eq:primary} with the following ``$(R,\Gamma)$'--bounded" auxiliary optimization problems:
\bea
\phi_{R,\Gamma}^{(n)}:=&\min_{\substack{\mu, \tilde{\betab}\\ \mu^2+\|\tilde{\betab}\|^2\leq R^2}}~\max_{\substack{\tilde{\bf u}\leq 0\\ \|\tilde{\bf u}\|\leq \Gamma}}\frac{1}{\sqrt{n}}\|\tilde{\betab}\|_2\g^{T}{\bf D}_y{\bf u}+\frac{1}{\sqrt{n}}\|{\bf D}_y\tilde{\bf u}\|_2\h^{T}\tilde{\betab}+\frac{1}{\sqrt{n}}\left(\tilde{\bf u}^{T}{\bf D}_y{\bf w}\right)\mu-\frac{1}{\sqrt{n}}\tilde{\bf u}^{T}{\bf 1}+\mu^2+\|\tilde{\betab}\|_2^2\nonumber\\
&=\min_{\substack{\mu, \tilde{\betab}\\ \mu^2+\|\tilde{\betab}\|^2\leq R^2}}~\max_{0\leq \theta\leq \Gamma}\max_{\substack{\tilde{\bf u}\leq 0\\ \|\tilde{\bf u}\|= \theta}}\frac{1}{\sqrt{n}}\|\tilde{\betab}\|_2\g^{T}{\bf D}_y{\bf u}+\frac{1}{\sqrt{n}}\|{\bf D}_y\tilde{\bf u}\|_2\h^{T}\tilde{\betab}+\frac{1}{\sqrt{n}}\left(\tilde{\bf u}^{T}{\bf D}_y{\bf w}\right)\mu-\frac{1}{\sqrt{n}}\tilde{\bf u}^{T}{\bf 1}+\mu^2+\|\tilde{\betab}\|_2^2,
\eea
where $\mathbf{g}\sim \mathcal{N}(0, {\bf I}_n)$ and $\mathbf{h} \sim \mathcal{N}(0,{\bf I}_{p-1})$. Having identified the AO problem, the next step is to simplify them so as to reduce them to problems involving optimization only over a few number of scalar variables. This will facilitate in the next step  of inferring their asymptotic behavior. 

\vp
\noindent\textbf{Scalarization of the AO.} To simplify the AO problem, we start by optimizing over the direction of the optimization variable  $\tilde{\ub}$. In doing so, we obtain:
\begin{equation}
\phi_{R,\Gamma}^{(n)}=\min_{\substack{\mu,\tilde{\betab}\\ \mu^2+\|\tilde{\betab}\|_2^2\leq R^2}} \max_{0\leq \theta\leq \Gamma} \theta \left\|\left(\frac{1}{\sqrt{n}}\|\tilde{\betab}\|_2{\bf D}_y \g+\frac{\mu}{\sqrt{n}}{\bf D}_y{\bf w}-\frac{{\bf 1}}{\sqrt{n}}\right)_{-}\right\|_2 +\theta \ \frac{1}{\sqrt{n}}  \h^{T}\tilde{\betab} +\mu^2+\|\tilde{\betab}\|_2^2,~
\label{eq:AO_inter}
\end{equation}
where we used the fact that $\|\mathbf{D}_y\tilde{\ub}\|_2=\|\tilde{\ub}\|_2$.
To proceed, note that for any $\theta$, the direction of $\tilde{\betab}$ that minimizes the objective in \eqref{eq:AO_inter} is given by $\frac{-\h}{\|\h\|}$. We can use this fact to show (see Lemma \ref{lem:Abla_help1} in Appendix \ref{app:tech}) that $\phi_{R,\Gamma}^{(n)}$ simplifies to:
$$
\phi_{R,\Gamma}^{(n)}=\min_{\substack{\mu,\tilde{\betab}\\ \mu^2+\|\tilde{\betab}\|_2^2\leq R^2}} \max_{0\leq \theta\leq \Gamma} \theta \left\|\left(\frac{1}{\sqrt{n}}\|\tilde{\betab}\|_2{\bf D}_y \g+\frac{\mu}{\sqrt{n}}{\bf D}_y{\bf w}-\frac{{\bf 1}}{\sqrt{n}}\right)_{-}\right\|_2 -\theta \  \frac{\|\tilde{\betab}\|_2}{\sqrt{n}} \|\h\|_2 +\mu^2+\|\tilde{\betab}\|_2^2.
$$
In the above optimization problem, it is easy to see that $\tilde{\betab}$ appears only through its norm, which we henceforth denote $\alpha:=\|\tilde{\betab}\|_2$. Further optimizing over $\theta$, we obtain the following scalar optimization problem:
\bea
\phi_{R,\Gamma}^{(n)}=\min_{\substack{\mu,\alpha \geq 0\\ \mu^2+\alpha^2\leq R^2}} \Gamma\, \max\left(\left\|\left(\frac{\alpha}{\sqrt{n}}{\bf D}_y \g+\frac{\mu}{\sqrt{n}}{\bf D}_y{\bf w}-\frac{{\bf 1}}{\sqrt{n}}\right)_{-}\right\|_2 - \  \frac{\alpha}{\sqrt{n}} \|\h\|_2 ,0\right) +\mu^2+\alpha^2 \label{eq:similar1}.
\eea
It is important to note that the new formulation of the AO problem is obtained from a deterministic analysis that did not involve any asymptotic approximation. 

\vp
\noindent\textbf{Asymptotic behavior of the AO.} In view of \eqref{eq:similar}, let us consider the sequence of functions $$\Ell_n(\alpha,\mu):= \left\|\left(\frac{\alpha}{\sqrt{n}}{\bf D}_y \g+\frac{\mu}{\sqrt{n}}{\bf D}_y{\bf w}-\frac{{\bf 1}}{\sqrt{n}}\right)_{-}\right\|_2 - \  \frac{\alpha}{\sqrt{n}} \|\h\|_2 ,$$ for $\alpha\geq 0$ and $\alpha^2+\mu^2\leq R^2$. It is easy to see that $\Ell_n$ is jointly convex in its arguments $(\alpha,\mu)$ and converges almost surely in the limit of \eqref{eq:linear} to:
\begin{align}\label{eq:oell}
\overline{\Ell}:(\alpha,\mu)\mapsto \sqrt{\mathbb{E}\left(\alpha\,H+\mu\,G\,Y-1\right)_{-}^2}-\alpha \sqrt{\kappa},
\end{align}
where, as in \eqref{eq:HGZ},
$$
H,G,Z \overset{\text{i.i.d.}}{\sim} \mathcal{N}(0,1),\qquad Y \sim \text{Rad}\left( f(s\,G + \sigma Z) \right).
$$

The convergence above holds point-wise in $(\alpha,\mu)$. But, convergence of convex functions is uniform over compact sets \cite[Cor.~II.1]{AG1982}. Therefore, for an arbitrary fixed compact set $\Cc$ and any $\eps>0$, for $n$ sufficiently large it holds that
\bea\label{eq:uni}
\overline{\Ell}(\alpha,\mu)-\epsilon \leq \Ell_n(\alpha,\mu)\leq \overline{\Ell}(\alpha,\mu)+\epsilon,\qquad\forall(\alpha,\mu)\in\Cc.
\eea

%




\subsection{Proof of Proposition~\ref{propo:PT}}\label{sec:proof_PT}

\noindent\textbf{Organization of the proof.} We organize the proof of Proposition \ref{propo:PT} in three parts. In the first two parts, which are both presented in this section, we prove the following two statements in the order in which they {appear}:
\bea
 \kappa<g(\kappa)~&\Rightarrow~\Pr\left(~\text{The event} \ \ \Ecsep~ \text{holds for all large} \ n\right)=0,\label{eq:proof_PT_step2}\\
 \kappa>g(\kappa)~&\Rightarrow~\Pr\left(~\text{The event} \ \  \Ecsep~ \text{holds for all large} \  n\right)=1.\label{eq:proof_PT_step1}
\eea
In the third part, we prove that the function $g(\kappa)$ is decreasing. Hence, the equation $\kappa=g(\kappa)$ admits a unique solution $\kappa_*$ and  one of the following two statements holds true for any $\kappa$:
$
\kappa > \kappa_* \Rightarrow  \kappa>g(\kappa)$ or $\kappa < \kappa_* \Rightarrow  \kappa<g(\kappa)$.
This third part of the proof is deferred to Appendix \ref{sec:g(kappa)_proof}.


\vp
\subsubsection{Part I: Proof of \eqref{eq:proof_PT_step2}}\label{sec:Part I}
Here, we prove that when 
\begin{equation}
\kappa < \inf_{t\in\mathbb{R}} \mathbb{E}\left(H+tGY\right)_{-}^2 =: g(\kappa),
	\label{eq:cond247}
\end{equation}
holds, then, $\Phi^{(n)}=\infty$, for large enough $n$. Equivalently, when \eqref{eq:cond247} holds, then, the hard-margin SVM program \eqref{eq:SVM_app} is infeasible with probability one. To prove the desired we will apply Theorem \ref{th:feasibility}(i). Specifically, in view of \eqref{eq:cond_inf}, it suffices  to prove that under \eqref{eq:cond247}, for any fixed $R,\Gamma$, there exists constant $C>0$ (independent of $\Gamma$ {and $R$}) such that for sufficiently large $n$ {(that can be taken independently of $R$ and $\Gamma$):}
\begin{equation}
\phi_{R,\Gamma}^{(n)}\geq C\Gamma.
		\label{eq:unboundedness}
	\end{equation}

In the remaining of the proof, we show that \eqref{eq:unboundedness} holds.
{Recall  from \eqref{eq:similar1} that,
$$
\phi_{R,\Gamma}^{(n)}= \min_{\substack{\mu,\alpha \geq 0\\ \mu^2+\alpha^2\leq R^2}} \Gamma \max({\Ell}_n(\alpha,\mu),0)+\alpha^2+\mu^2.
$$
Hence, 
\begin{equation}
	\phi_{R,\Gamma}^{(n)}\geq  \Gamma \max(\min_{\mu,\alpha\geq 0}{\Ell}_n(\alpha,\mu),0). \label{eq:lb}
\end{equation}
The function $(\alpha,\mu) \mapsto {\Ell}_n(\alpha,\mu)$ is jointly convex in the variables $(\alpha,\mu)$ and converges pointwise to $(\alpha,\mu) \mapsto\overline{\Ell}(\alpha,\mu)$. Thus, for any $\alpha>0$, $\mu\mapsto L_n(\alpha,\mu)$ is convex and converges to $\mu\mapsto \overline{\Ell}(\alpha,\mu)$. Moreover, it is easy to see that $\lim_{\mu\to\pm\infty}\overline{\Ell}(\alpha,\mu)\to\infty$. Hence, using \cite[Lemma 10]{Master}, $\min_{\mu}\Ell_n(\alpha,\mu)$ converges to $\min_{\mu}\overline{\Ell}(\alpha,\mu)$. Similarly, $\alpha\mapsto \min_{\mu} {\Ell}_n(\alpha,\mu)$ is convex and converges pointwise to $\alpha\mapsto \min_{\mu} \overline{\Ell}(\alpha,\mu)$. Moreover, 
\begin{align}
	\lim_{\alpha\to\infty} \overline{\Ell}(\alpha,\mu)&\geq \alpha \big(\inf_{t\in\R} \sqrt{\mathbb{E}(H+tGY)_{-}^2}-\sqrt{\kappa}\big) \nn \\
	&\stackrel{\alpha\to \infty}{\longrightarrow} \infty \label{eq:pret}
\end{align}
because of \eqref{eq:cond247}.
Hence, we can use again \cite[Lemma 10]{Master}, to find that
$$
\min_{\mu,\alpha\geq 0}{\Ell}_n(\alpha,\mu)\, {\ras}\,  \min_{\substack{\mu,\alpha\geq 0}}\overline{\Ell}(\alpha,\mu).
$$
Thus, in view of \eqref{eq:lb}, for any $\epsilon>0$, it holds for $n$ sufficiently large (independent of $R$ and $\Gamma$) that
%
\bea\label{eq:uni_lower}
\phi_{R,\Gamma}^{(n)} \geq  \Gamma \max(\min_{\substack{\mu,\alpha\geq 0}}\overline{\Ell}(\alpha,\mu)-\epsilon,0).\eea
Thus, the desired inequality \eqref{eq:unboundedness} holds provided {that there exists $C$ such that }
\begin{equation}
\inf_{\substack{ \alpha\geq 0\\ \mu\in\mathbb{R}} }  \overline{\Ell}(\alpha,\mu)>C.\label{eq:condEq}
\end{equation}
To see this, note that in this case we can choose $\epsilon$ sufficiently small (say, smaller than $\inf_{\mu,\alpha\geq 0} 0.5\overline{\Ell}(\alpha,\mu)$) in \eqref{eq:uni_lower}. 

Hence, it suffices to prove that \eqref{eq:uni_lower} holds under \eqref{eq:cond247}.
To show the desired, we first prove that  the optimization over $\mu$ in \eqref{eq:uni_lower} when $\alpha$ is constrained to be in the vicinity of $0$ can be assumed over a compact set. This can be seen as follows. 
For any $0\leq \alpha\leq 1$, it holds
$$
\overline{\Ell}(\alpha,\mu)\geq \sqrt{\mathbb{E}\left(\alpha\,H+\mu\,G\,Y-1\right)_{-}^2{1}_{\{H<0\}}}-\alpha\sqrt{\kappa}\geq \sqrt{\mathbb{E}\left(\mu GY-1\right)_{-}^2{1}_{\{H<0\}}} -\sqrt{\kappa}.
$$
But, $\sqrt{\mathbb{E}\left(\mu GY-1\right)_{-}^2{1}_{\{H<0\}}}$ grows unboundedly large as $\mu\to \infty$ or $\mu\to -\infty$. Hence, for any $0<\tilde{\eta}\leq 1$,
$$
\lim_{|\mu|\to\infty }\min_{0\leq \alpha \leq \tilde{\eta}} \overline{\Ell}(\alpha,\mu) =\infty,
$$
thus proving that the minimum over $\mu$ in \eqref{eq:uni_lower} is bounded. To be concrete, we can (and we do) assume henceforth that $\mu$ belongs to some compact set  of $\mathbb{R}$ when $\alpha$ is constrained in $[0,\tilde{\eta}]$. For fixed $\mu\in\mathcal{A}$,  note that $\lim_{\alpha\downarrow 0} \overline{\Ell}(\alpha,\mu)=\sqrt{\mathbb{E}\left(\mu GY-1\right)_{-}^2}.$ Furthermore, this convergence is uniform over compact sets due to convexity of the function $\overline{\Ell}$, hence,
$$
\lim_{\alpha \downarrow 0} \min_{\mu \in \mathcal{A}} \overline{\Ell}(\alpha, \mu)=\min_{\mu\in\mathcal{A}}\sqrt{\mathbb{E}\left(\mu GY-1\right)_{-}^2}>\inf_{\mu\in\mathbb{R}} \sqrt{\mathbb{E}\left(\mu GY-1\right)_{-}^2}>0.
$$
Since $\alpha\mapsto \min_{\mu\in\mathcal{A}} \overline{L}(\alpha,\mu)$ is continuous and thus uniformly continuous over compacts, there exists $1>\eta>0$, such that
\begin{equation}
\min_{0\leq \alpha\leq \eta} \inf_{\mu\in\mathcal{A}} \overline{\Ell}(\alpha, \mu)> 0.5\inf_{\mu\in\mathbb{R}} \sqrt{\mathbb{E}\left(\mu GY-1\right)_{-}^2}>0 .
	\label{eq:1}
\end{equation}
On the other hand,
\begin{equation}
	\min_{\eta\leq \alpha}  \inf_{\mu\in\mathbb{R}} \overline{\Ell}(\alpha,\mu)\geq \inf_{\substack{t_2\geq 0\\ t_1\in\mathbb{R}}} \eta \Big(\sqrt{\mathbb{E}\left(H+t_1GY-\frac{1}{t_2}\right)_{-}^2}-\sqrt{\kappa}\Big)\geq  \eta\,\big(\inf_{t_1\in\mathbb{R}}\sqrt{\mathbb{E}\left(H+t_1GY\right)_{-}^2}-\sqrt{\kappa}\big) = \eta \,(\sqrt{g(\kappa)} - \sqrt{\kappa}).
	\label{eq:2}
\end{equation}
Note that the right-hand side above is strictly positive because of \eqref{eq:cond247}.

Thus, combining \eqref{eq:1} and \eqref{eq:2}, we conclude with the desired, i.e. 
if \eqref{eq:cond247} holds, then \eqref{eq:condEq} and consequently \eqref{eq:unboundedness} hold.

\vp
\subsubsection{Part II: Proof of \eqref{eq:proof_PT_step1}}\label{sec:PartII}
We will now prove that if \begin{equation}
\kappa > \inf_{t\in\mathbb{R}} \mathbb{E}\left(H+tGY\right)_{-}^2 =: g(\kappa),
	\label{eq:condC}
\end{equation} then, the PO in \eqref{eq:SVM_app} is feasible (eqv. $\Phi^{(n)}$ is bounded) with probability 1. To prove this, we will apply Theorem \ref{th:feasibility}(ii). Specifically, in view of \eqref{eq:bnd_cond} it will suffice  to  show that  under \eqref{eq:condC} the AO is feasible with probability 1, for sufficiently large $n$. 

First, we will show that for any $\delta>0$  the following set is non-empty: 
\bea\label{eq:nonempty}
\mathcal{B}_\delta:=\left\{(\alpha,\mu)\in\mathbb{R}_{+}\times \mathbb{R} \ \ | \ \ \overline{\Ell}(\alpha,\mu)\leq -\delta\right\}\,\neq\,\emptyset.
\eea
To see this, let $t^\star$ be such that:
$$
t^\star=\arg\min_{t\in \mathbb{R}}\mathbb{E}(H+tGY)_{-}^2 \,.
$$
Clearly, $t^\star$ is finite since $\lim_{t\to \pm\infty} \mathbb{E}(H+tGY)_{-}^2=\infty$. Thus, setting $\mu=t^\star \alpha$ and using \eqref{eq:condC}, we find that
$$
\lim_{\alpha\to\infty} \overline{\Ell}(\alpha,t^\star\alpha)=\lim_{\alpha\to\infty} \alpha \Big[\sqrt{\mathbb{E}\left(H+t^\star GY-\frac{1}{\alpha}\right)_{-}^2}-\sqrt{\kappa}\Big]=-\infty\,.
$$
This and continuity of $\overline{\Ell}(\alpha,\mu)$ prove \eqref{eq:nonempty}.

Having shown \eqref{eq:nonempty}, let 
\bea \label{eq:C_delta}
C_\delta:=\min_{(\alpha,\mu)\in\mathcal{B}_\delta} \sqrt{\alpha^2+\mu^2} < \infty.
\eea

Next, we prove \eqref{eq:bnd_cond} by showing that the AO problem is upper bounded by $C_\delta^2$ with probability 1. Specifically, we work with the sufficient condition \eqref{eq:bnd_cond_easy}. Towards that end, choose $k_0$ be an integer strictly greater than $C_\delta$ and let\footnote{Recall from  \eqref{eq:2o} that $\phi_{k_0}^{(n)}$ is defined as $\phi_{k_0}^{(n)}=\displaystyle{\sup_{m>0}}~\displaystyle{\min_{(\alpha,\mu)\in\Cc_k} \max_{0\leq\theta\leq m}\theta \,\Ell_n(\alpha,\mu) +\alpha^2+\mu^2}$. Thus, inequality ``$\leq$" holds in \eqref{eq:less_A0} by the min-max theorem; moreover inequality suffices here for our purposes as explained in Remark \ref{rem:feas}. Nevertheless, since the objective function of the scalarized AO in \eqref{eq:less_A0} is convex-concave, equality holds in \eqref{eq:less_A0} by applying the Sion's min-max theorem \cite{Sion}.}
\bea
\phi_{k_0}^{(n)}&=\displaystyle{\min_{\substack{\mu,\alpha\geq 0\\ \mu^2+\alpha^2\leq k_0^2}}\sup_{\theta\geq 0}\,\theta L_n(\alpha,\mu) +\mu^2+\alpha^2} \label{eq:less_A0} \\
& =  \displaystyle{\min_{\substack{(\alpha,\mu)\in\Cc_{k_0}\\ L_n(\alpha,\mu)\leq 0}} \mu^2+\alpha^2},\label{eq:less_A}
\eea
where we defined
$$\mathcal{C}_{k_0}:=\left\{(\alpha,\mu)\in \mathbb{R}_{+}\times \mathbb{R}~|~\alpha^2+\mu^2\leq k_0^2\right\}.$$
Moreover, by uniform convergence of the AO in \eqref{eq:uni} we know that for all sufficiently large $n$: $L_n(\alpha,\mu)\leq \overline{\Ell}(\alpha,\mu)+\delta$, uniformly for all $(\alpha,\mu)\in\Cc_{k_0}$.
%
Hence, recalling \eqref{eq:nonempty}, it holds that $$\left\{(\alpha,\mu)~|~(\alpha,\mu)\in \mathcal{C}_{k_0}\cap \mathcal{B}_\delta \right\}\subset\left\{(\alpha,\mu)~|~(\alpha,\mu)\in \mathcal{C}_{k_0} \text{ and } L_n(\alpha,\mu)\leq 0 \right\}.$$ 
And so, continuing from \eqref{eq:less_A}, we have shown that for sufficiently large $n$:
\bea\label{eq:UB_step1}
\phi_{k_0}^{(n)}\displaystyle{\leq \min_{(\mu,\alpha)\in \mathcal{C}_{k_0}\cap\mathcal{B}_\delta} \mu^2+\alpha^2 \leq C_\delta^2},
\eea
where the last inequality follows by definition of $C_\delta$ in \eqref{eq:C_delta}.
This shows \eqref{eq:bnd_cond_easy} and the proof is complete by applying Theorem \ref{th:feasibility}(ii).
\vp

\subsection{Proof of Proposition \ref{propo:SVM}}\label{sec:SVM_prop}




\subsubsection{Preliminaries and strategy}\label{sec:pre_SVM} Throughout this section we assume that \eqref{eq:condC} holds. We will prove Proposition \ref{propo:SVM} by applying Theorem \ref{th:characterization}.  Specifically, let $$k_0^2= \lceil(q^\star)^2\rceil+1,$$
where $q^\star$ is as defined in Proposition \ref{propo:SVM}. In view of Theorem \ref{th:characterization}, we need to prove the following
 \begin{subequations}\label{eq:cond}
\begin{align}
&\text{For any}~ \eps>0:\quad\mathbb{P}\left[\cup_{k=k_0}^\infty\Big\{\phi_{k}^{(n)} \leq (q^\star)^2-\epsilon\Big\}, ~~\text{i.o.}\right]=0,\label{eq:cond1}\\
	&\text{For any}~ \eps>0:\quad\mathbb{P}\left[\Big\{\phi_{k_0}^{(n)} \geq (q^\star)^2+\epsilon\Big\}, ~~\text{i.o.}\right]=0\label{eq:cond2},
\end{align}
\end{subequations}
where $\phi_{k}^{(n)}$ is (see \eqref{eq:less_A0} and Remark \ref{rem:minsup}) given by:
\bea
\phi_{k}^{(n)}&=\displaystyle{\min_{\substack{(\alpha,\mu)\in\Cc_{k}\\ L_n(\alpha,\mu)\leq 0}} \mu^2+\alpha^2},\label{eq:less_A2}
\eea
with
$\mathcal{C}_{k}:=\left\{(\alpha,\mu)\in \mathbb{R}_{+}\times \mathbb{R}~|~\alpha^2+\mu^2\leq k^2\right\}.$ At this point, it might be unclear how the value of $q^\star$ relates to the AO in \eqref{eq:less_A2}. Lemma \ref{lem:detAO} below explains this; see Section \ref{sec:detAO} for the proof. Specifically recall from \eqref{eq:uni} that $L_n(\alpha,\mu)$ in \eqref{eq:less_A2} converges to $\overline{\Ell}(\alpha,\mu)$ defined in \eqref{eq:oell}.

\begin{lem}\label{lem:detAO}
Recall the definition of $\overline{\Ell}(\alpha,\mu)$ in \eqref{eq:oell} and let $q^\star$ be defined as in Proposition \ref{propo:SVM}. Further assume the separability condition \eqref{eq:condC}. The following statement is true, for all $k\geq k_0=\sqrt{\lceil(q^\star)^2\rceil+1}$,
\bea
\min_{ \substack{ (\alpha,\mu)\in\Cc_k \\ \overline{\Ell}(\alpha,\mu) \leq 0 }} \alpha^2+\mu^2 = (q^\star)^2.\label{eq:qstarwhy}
\eea
\end{lem}

In addition to \eqref{eq:cond}, we need to consider an appropriate ``perturbed" AO as suggested by \eqref{eq:condition3}. For this reason, fix $\widetilde{\ksi}>0$ and define the ``perturbed" version of the AO problem (cf. \eqref{eq:similar1}) as follows:
\bea
\widetilde\phi_{k,m}^{(n)}=\min_{\substack{(\alpha,\mu)\in\Cc_k \\ (\alpha,\mu)\not\in\Sc}} m\, \max\left(L_n(\alpha,\mu) ,0\right) +\mu^2+\alpha^2, \label{eq:similar}
\eea
where we denote:
\bea\label{eq:Sc}
\Sc &:= \{ (\alpha,\mu)~|~ \alpha\in\Sc_\alpha \text{ and } \mu\in\Sc_\mu \}.\\
\Sc_\mu & = \{ \mu~|~ |\mu-\mu^\star|<\ksit \},~~~ \mu^\star :=  q^\star\rho^\star,\nn\\
 \Sc_\alpha &= \{ \mu~|~ |\alpha-\alpha^\star|<\ksit \},~~~  \alpha^\star := q^\star\sqrt{1-(\rho^\star)^2},\nn
\eea
and $\rho^\star$ is defined in Proposition \ref{propo:SVM}. In order to apply Theorem \ref{th:characterization}, we will prove in addition to \eqref{eq:cond} that
\bea
\text{There exists}~ \zeta>0:\quad\mathbb{P}\left[\cup_{k=k_0}^\infty\Big\{\widetilde\phi_{k}^{(n)} \leq (q^\star)^2+\zeta\Big\}, ~~\text{i.o.}\right]=0,\label{eq:cond3}
\eea
where, recall from \eqref{eq:2o} that 
\bea
\widetilde\phi_k^{(n)} &=\sup_{m \geq 0}  \widetilde\phi^{(n)}_{k,m} = \lim_{m \rightarrow \infty}  \widetilde\phi^{(n)}_{k,m}. \nn
\eea

Before proceeding let us explain our choice of the perturbation set in \eqref{eq:Sc}. Recall from the analysis of the AO that the variable $\alpha$ takes the role of $\|\tilde{\betab}\|_2$. Hence, the set $\Sc$ in \eqref{eq:Sc} corresponds to the following set in terms of $\betab$:
\bea\label{eq:Sbad}
\Scb:= \{ \betab=[\mu, \tilde{\betab}^T]^T\in\R^p~|~ \mu \in \Sc_\mu \text{ and } \|\tilde{\betab}\|_2 \in \Sc_\alpha \}.
\eea
Thus, if we show \eqref{eq:cond1}, \eqref{eq:cond2} and \eqref{eq:cond3}, then Theorem \ref{th:characterization} will imply that the solution $\betabh$ of the PO $\Phi^{(n)}$ in \eqref{eq:SVM_app} satisfies $\betabh\in\Scb$ with probability one. To see why this leads to the asymptotic formulae of the proposition argue as follows. For small enough $\widetilde\ksi$ it follows by definition \eqref{eq:Sbad} that
$$\betab\in\Scb~~\Longrightarrow~~ \betab_1\approx\mu^\star\text{ and }\|\tilde\betab\|_2\approx\alpha^\star.$$
Therefore,  for $\betab\in\Scb$ (and small enough $\ksit$):
$$\frac{\langle \betab , \betab_0 \rangle}{\|\betab\|_2\|\etab_0\|_2} = \frac{ s \betab_1 }{r\sqrt{\betab_1^2 + \|\tilde \betab\|_2^2}} \approx  \frac{ s \mu^\star }{r\sqrt{(\mu^\star)^2 + (\alpha^\star)^2}} = \rho^\star\frac{s}{r}$$
and
\begin{align}
\Rc(\betab) &= \mathbb{P}(\,y\,\w^T\betab<0~) = \mathbb{P}\left(~\,\Rad{\ell^\prime(\,\w^T\betab_0+\z^T\gammab_0)} \cdot (\betab_1\,w_i + \tilde{\betab}^T\vb_i)<0 \, \right)\nn\\
&=\mathbb{P}\left(~\,\Rad{\ell^\prime(\,w_i \,s+\z^T\gammab_0)} \cdot (\betab_1\,w_i + \tilde{\betab}^T\vb_i)<0 \, \right)\nn\\
&= \mathbb{P}\left(~\,\Rad{\ell^\prime(\,s\,G +\sigma\,Z)} \cdot (\betab_1\,G + \|\tilde\betab\|_2 H)<0 \, \right) = \mathbb{P}\left(~\,Y \cdot (\betab_1\,G + \|\tilde\betab\|_2 H)<0 \, \right)\nn\\
&= \mathbb{P}\left(~\,\betab_1\,GY + \|\tilde\betab\|_2 H<0 \, \right)\\
&\approx  \mathbb{P}\left(~\, \mu^\star\,GY + \alpha^\star H<0 \, \right) = \mathbb{P}\left(~\, \rho^\star\,Y\,G  + \sqrt{1-(\rho^\star)^2} H<0 \, \right), \label{eq:egw2}
\end{align}
where as always $G,H,Z$ and $Y$ follow our notation in \eqref{eq:HGZ} and the second to last line follows since $HY\sim \Nn(0,1)$. 

\vp
Hence, it remains to prove  \eqref{eq:cond1}, \eqref{eq:cond2} and \eqref{eq:cond3}.

\subsubsection{Proof of Lemma \ref{lem:detAO}}\label{sec:detAO}

First, recalling the definition of $\overline{\Ell}(\alpha,\mu)$ in \eqref{eq:oell} note the following implications
\bea
\overline{\Ell}(\alpha,\mu)\leq 0 \, &\Rightarrow \, \alpha^2\kappa \geq \E\big( \alpha H + \mu GY - 1 \big)_-^2\nn\\
&\Rightarrow \, \frac{\alpha^2}{\alpha^2+\mu^2}\kappa \geq \E\Big( \frac{\alpha}{\sqrt{\alpha^2+\mu^2}} H + \frac{\mu}{{\sqrt{\alpha^2+\mu^2}}} GY - \frac{1}{\sqrt{\alpha^2+\mu^2}} \Big)_-^2,\label{eq:Impl2}
\eea
where, in the second line we have used the fact that $(0,0)\not\in\{(\alpha,\mu)~|~\overline{\Ell}(\alpha,\mu)\leq 0\}$. To proceed, define new variables
\bea\label{eq:repam}
\rho:=\frac{\mu}{\sqrt{\alpha^2+\mu^2}}\in[-1,1]\quad\text{and}\quad q := \sqrt{\alpha^2+\mu^2}\geq 0\,.
\eea
With this new notation, note that  the expression in \eqref{eq:Impl2} can be written as
$$
(1-\rho^2) \kappa \geq \E\big(\rho GY + \sqrt{1-\rho^2} H - 1/q \big)_-^2 \quad\Leftrightarrow\quad \eta(q,\rho) \leq 0,
$$
where we recall the definition of the mapping $\eta(q,\rho)$ in \eqref{eq:eta_func}. With this note and using \eqref{eq:Impl2}, we have that:
\bea\label{eq:ineta}
\min_{ \substack{ (\alpha,\mu)\in\Cc_k \\ \overline{\Ell}(\alpha,\mu) \leq 0 }} \alpha^2+\mu^2 ~=~ \min\left\{ q^2~\Big|~ 0\leq q \leq k  \text{ and } \min_{\rho\in[-1,1]} \eta(q,\rho) \leq 0 \right\} \,.
\eea
In Lemma \ref{lem:eta_properties}, it is shown that $q\mapsto\min_{\rho\in[-1,1]} \eta(q,\rho)$ is strictly decreasing. Recall the definition of $q^\star$ in Proposition \ref{propo:SVM} as the unique minimum of the equation $\min_{\rho\in[-1,1]} \eta(q,\rho) = 0$; see Appendix \ref{app:eta} for a proof that such a unique minimum exists under the separability assumption \eqref{eq:condC} of the lemma. Combining the above, we conclude that 
\bea
\nn \min_{\rho\in[-1,1]} \eta(q,\rho) \leq 0 \, \Rightarrow \, \min_{\rho\in[-1,1]} \eta(q,\rho) \leq  \min_{\rho\in[-1,1]} \eta(q^\star,\rho)  \Rightarrow \, q \geq q^\star.
\eea
Thus, 
\bea
 \min\left\{ q^2~\Big|~ 0\leq q \text{ and } \min_{\rho\in[-1,1]} \eta(q,\rho) \leq 0 \right\} = (q^\star)^2.
\eea
Using this and the assumption $k^2> (q^\star)^2$ proves that the RHS in \eqref{eq:ineta} is equal to $(q^\star)^2$, which concludes the proof.

\subsubsection{Proof of \eqref{eq:cond1}}~
Fix any $\eps>0$, and $k\geq k_0=\sqrt{\lceil(q^\star)^2\rceil+1}$. Recall the notation in \eqref{eq:less_A2}. 

\vp
We start by proving that the following statement holds almost surely
\bea\label{eq:question_LB}
\phi_k^{(n)} = \displaystyle{\min_{\substack{(\alpha,\mu)\in\Cc_{k_0}\\ \Ell_n(\alpha,\mu)\leq 0}} \alpha^2+\mu^2} =: \phi_{k_0}^{(n)}.
\eea
Before that we argue that the minimization in the RHS above is feasible. The argument is identical to what was done in Section \ref{sec:PartII}. Specifically, by uniform convergence of $\Ell_n(\alpha,\mu)$ to $\overline{\Ell}(\alpha,\mu)$ (cf. \eqref{eq:uni}), for any $\delta>0$ and sufficiently large $n$:
$$
\{(\alpha,\mu)~|~(\alpha,\mu)\in\Cc_{k_0} \text{ and } \overline{\Ell}(\alpha,\mu)\leq -\delta\} ~  \subseteq ~ \{(\alpha,\mu)~|~(\alpha,\mu)\in\Cc_{k_0} \text{ and } \Ell_n(\alpha,\mu)\leq 0\}.
$$
Using assumption \eqref{eq:condC}, we have shown in \eqref{eq:nonempty} that the set on the LHS above is non-empty. This implies non-emptyness of the feasibility set in \eqref{eq:question_LB}. 

\noindent We are now ready to prove \eqref{eq:question_LB}, i.e. $\phi_k^{(n)}=\phi_{k_0}^{(n)}$.  Clearly, the statement holds with ``$\leq$" instead of equality. On the other hand, suppose that the minimizer $(\tilde\alpha,\tilde\mu)$ of $\phi^{(n)}_{k}$ is not feasible for $\phi^{(n)}_{k_0}$, i.e. $(\tilde\alpha,\tilde\mu)\notin\Cc_{k_0}$. Then, by definition of the sets $\Cc_k$, it must be that $\phi^{(n)}_{k}=\tilde\alpha^2+ \tilde\mu^2 \geq k_0^2$. But, $k_0^2 \geq \phi^{(n)}_{k_0}$. Thus, $\phi^{(n)}_{k}\geq \phi^{(n)}_{k_0}$, which shows \eqref{eq:question_LB}.

\vp
Next, continuing from \eqref{eq:question_LB}, we will use 
%
uniform convergence of $\Ell_n(\alpha,\mu)$ to $\overline{\Ell}(\alpha,\mu)$ (cf. \eqref{eq:uni}) to show that for sufficiently large $n$:
\bea\label{eq:det_LB_AO}
\phi_{k_0}^{(n)} = \displaystyle{\min_{\substack{(\alpha,\mu)\in\Cc_{k_0}\\ \Ell_n(\alpha,\mu)\leq 0}} \alpha^2+\mu^2}  \geq \Big\{ \min_{ \substack{ (\alpha,\mu)\in\Cc_{k_0} \\ \overline{\Ell}(\alpha,\mu) \leq 0 }} \alpha^2+\mu^2 \Big\} \,-\, \eps.
\eea
 Before proving \eqref{eq:det_LB_AO}, let us see how it leads to the desired \eqref{eq:cond1}. When put together with \eqref{eq:question_LB}, \eqref{eq:det_LB_AO} shows that for sufficiently large $n$ (independent of $k$):
 $$
\phi_{k}^{(n)}  \geq  \Big\{ \min_{ \substack{ (\alpha,\mu)\in\Cc_{k_0} \\ \overline{\Ell}(\alpha,\mu) \leq 0 }} \alpha^2+\mu^2 \Big\} \,-\, \eps.
 $$
From this and \eqref{eq:qstarwhy} of Lemma \ref{lem:detAO} we conclude that  for sufficiently large $n$:
\bea
\phi_{k}^{(n)} \geq (q^\star)^2 - \eps.\nn
\eea
Since this holds for sufficiently large $n$ independent of $k\geq k_0$, we arrive at \eqref{eq:cond1}, as desired.

\vp
\noindent\underline{Proof of \eqref{eq:det_LB_AO}.} 
From \eqref{eq:uni}, the function $(\alpha, \mu)\mapsto \Ell_n(\alpha,\mu)$ converges  to $\overline{\Ell}(\alpha,\mu)$ defined in \eqref{eq:oell} uniformly over the compact set $\Cc_{k_0}$. Concretely, for any $\delta$,  for all sufficiently large $n$:
$
\Ell_n(\alpha,\mu)\geq \overline{\Ell}(\alpha,\mu)-\delta, \forall(\alpha,\mu)\in\mathcal{C}_{k_0}.
$
Thus,
\bea\label{eq:help21}
\phi_{k_0}^{(n)}\geq \displaystyle{\min_{\substack{(\alpha,\mu)\in\Cc_{k_0}\\ \overline{\Ell}(\alpha,\mu)\leq \delta}} \alpha^2+\mu^2}.
\eea
We may now conclude \eqref{eq:det_LB_AO} from \eqref{eq:help21}, by first applying Lemma \ref{lem:help2}  to see that $$\displaystyle{\min_{\substack{(\alpha,\mu)\in\Cc_{k_0}\\ \overline{\Ell}(\alpha,\mu)\leq 0}} \alpha^2+\mu^2} = \sup_{\zeta\geq0}\displaystyle{\min_{\substack{(\alpha,\mu)\in\Cc_{k_0}\\ \overline{\Ell}(\alpha,\mu)\leq \zeta}} \alpha^2+\mu^2},$$ and then invoking the $\epsilon$-definition of supremum. 


\subsubsection{Proof of  \eqref{eq:cond2}}\label{sec:cond2}
~Fix any $\eps>0$. It suffices to prove that for all sufficiently large $n$:
\bea\label{eq:det_UB_AO}
\phi_{k_0}^{(n)} \leq \Big\{ \min_{ \substack{ (\alpha,\mu)\in\Cc_{k_0} \\ \overline{\Ell}(\alpha,\mu) \leq 0 }} \alpha^2+\mu^2 \Big\} \,+\, {\eps}.
\eea
To see why this is sufficient for \eqref{eq:cond2} to hold, recall from \eqref{eq:qstarwhy} of Lemma \ref{lem:detAO} that the RHS above is equal to $(q^\star)^2+\eps$.

\vp
\noindent\underline{Proof of \eqref{eq:det_UB_AO}.}
From \eqref{eq:uni}, the function $(\alpha, \mu)\mapsto \Ell_n(\alpha,\mu)$ converges  to $\overline{\Ell}(\alpha,\mu)$ defined in \eqref{eq:oell} uniformly over the compact set $\Cc_{k_0}$. Concretely, for any $\delta>0$, for all sufficiently large $n$:
$
\Ell_n(\alpha,\mu)\leq \overline{\Ell}(\alpha,\mu)+\delta,~ \forall(\alpha,\mu)\in\mathcal{C}_{k_0}.
$
Thus,
\bea\label{eq:help21}
\phi_{k_0}^{(n)}\leq  \displaystyle{\min_{\substack{(\alpha,\mu)\in\Cc_{k_0}\\ \overline{\Ell}(\alpha,\mu)\leq -\delta}} \alpha^2+\mu^2}.
\eea
We may now deduce \eqref{eq:det_UB_AO} from \eqref{eq:help21}, by first applying Lemma \ref{lem:help2}  to see that $$\displaystyle{\min_{\substack{(\alpha,\mu)\in\Cc_{k_0}\\ \overline{\Ell}(\alpha,\mu)\leq 0}} \alpha^2+\mu^2} = \inf_{\zeta\geq0}\displaystyle{\min_{\substack{(\alpha,\mu)\in\Cc_{k_0}\\ \overline{\Ell}(\alpha,\mu)\leq -\zeta}} \alpha^2+\mu^2},$$ and then invoking the $\epsilon$-definition of infimum.

\subsubsection{Proof of \eqref{eq:cond3}}\label{sec:cond3}
The proof has two parts. 

\vp
In the first part, we  show that for any $\zeta>0$ and for sufficiently large $n$ (independent of $k$) it holds that:
\bea\label{eq:pertstep1}
\widetilde\phi_{k}^{(n)} \geq \Big\{\displaystyle{\min_{\substack{(\alpha,\mu)\in\Cc_{k_0}, (\alpha,\mu)\not\in \Sc\\ \overline{\Ell}(\alpha,\mu)\leq 0 }} \alpha^2+\mu^2} \Big\} -\zeta =:\widetilde{q}^2 -\zeta,
\eea
where recall the definition of the set $\Sc$ in \eqref{eq:Sc}.
Recall from Section \ref{sec:pre_SVM} that  
$$
\widetilde\phi_{k}^{(n)} = \lim_{m\rightarrow\infty} \displaystyle{\min_{\substack{(\alpha,\mu)\in\Cc_{k_0} \\ (\alpha,\mu)\not\in \Sc}}} \max_{0\leq\theta \leq m} \theta\, {\Ell}_n(\alpha,\mu) +\alpha^2 + \mu^2
$$
Note that the objective function above is convex in $(\alpha,\mu)$. Hence, we proceed as argued in Remark \ref{rem:minsup} to show that
\bea
\widetilde\phi_{k}^{(n)} = \displaystyle{\min_{\substack{(\alpha,\mu)\in\Cc_{k_0} \\ (\alpha,\mu)\not\in \Sc}}}  \sup_{\theta\geq 0} \theta\, {\Ell}_n(\alpha,\mu) +\alpha^2 + \mu^2 = \displaystyle{\min_{\substack{(\alpha,\mu)\in\Cc_{k_0}, (\alpha,\mu)\not\in \Sc \\ {\Ell}_n(\alpha,\mu) \leq 0}}}  \alpha^2 + \mu^2.\label{eq:startingfrom}
\eea
We may now prove \eqref{eq:pertstep1} starting from \eqref{eq:startingfrom} by using uniform convergence \eqref{eq:uni}. The proof of this step is identical to the proof of \eqref{eq:question_LB} and \eqref{eq:det_LB_AO} and is omitted for brevity.

\vp
In the second part of the proof, we show that there exists $\zeta>0$ such that 
\bea
\qt^2 \geq (q^\star)^2+2\zeta.\label{eq:desired_pert_112}
\eea
This will complete the proof of \eqref{eq:cond3}. Indeed, starting from \eqref{eq:pertstep1} and using $\zeta$ such that \eqref{eq:desired_pert_112} holds, we find that for sufficiently large $n$:
$$
\widetilde\phi_{k}^{(n)} \geq \widetilde{q}^2- \zeta  \geq (q^\star)^2 + \zeta, $$
as desired.

It remains to prove \eqref{eq:desired_pert_112}. First, following the same re-parametrization as in \eqref{eq:repam} (see Section \ref{sec:detAO} for identical derivations) we can express $\widetilde{q}^2$ as follows:

\bea\label{eq:det_LB_AO_perturbed_1}
\widetilde{q}^2 = \min\left\{ q^2~\Big|~ 0\leq q \leq k , \rho\in[-1,1] \text{ and } (q,\rho)\notin \Sc_q\times\Sc_\rho \text{ and } \eta(q,\rho) \leq 0 \right\},
\eea
where
$\Sc_\rho:=\{ \rho~|~|\rho-\rho^\star|<\ksi\}$, $\Sc_q:=\{ q~|~|q-q^\star|<\ksi\}$, and $\ksi>0$ is small enough constant chosen such that 
$
(\alpha,\mu)\in\Sc \,\Rightarrow\, (q,\rho)\in \Sc_q\times\Sc_\rho.
$
To proceed, we consider two cases as follows.

\vp
\noindent\underline{Case 1: $\rho\in[-1,1]$.} First, consider the case in which $\widetilde q$ is as follows:
$$
{\widetilde q}^2= \min\left\{ q^2~\Big|~ 0\leq q \leq k , ~q\notin \Sc_q \text{ and } \min_{\rho\in[-1,1]}\eta(q,\rho) \leq 0 \right\}.
$$
We will use the following facts for the function  $\widetilde\eta(q):= \min_{\rho\in[-1,1]}\eta(q,\rho)$: (i) it is decreasing; (ii) it has a unique zero $q^\star$. See Lemma \ref{lem:eta_properties} in Appendix \ref{app:eta} for a proof of these claims. From these and the constraint $|q-q^\star|\geq \ksi$, it is clear that ${\widetilde q}=q^\star + \ksi$. Hence, 
$
{\widetilde q}^2 > (q^\star)^2 + \ksi^2,
$
and \eqref{eq:desired_pert_112} holds after choosing $\zeta = \ksi^2/2>0$.

\vp
\noindent\underline{Case 2: $|\rho-\rho^\star|>\ksi$.} Second, consider the case in which $\widetilde q$ is as follows:
$$
{\widetilde q}^2= \min\Big\{ q^2~\Big|~ 0\leq q \leq k , \min_{\substack{\rho\in[-1,1] \\ |\rho-\rho^\star|>\ksi}}\eta(q,\rho) \leq 0 \Big\}.
$$
 For the sake of contradiction to \eqref{eq:desired_pert_112}, assume that ${\widetilde q}\leq q^\star$. By Lemma \ref{lem:eta_properties} in Appendix \ref{app:eta} the function $q\mapsto \min_{\substack{\rho\in[-1,1] \\ |\rho-\rho^\star|>\ksi}}\eta(q,\rho)$ is strictly decreasing. From this and definition of $\widetilde q$, it follows that $\min_{\substack{\rho\in[-1,1] \\ |\rho-\rho^\star|>\ksi}}\eta(q^\star,\rho) \leq 0 = \eta(q^\star,\rho^\star)$. But, again from Appendix \ref{app:eta} $\rho\mapsto\eta(q^\star,\rho) $ has a unique minimizer $\rho^\star$ in $[-1,1]$. Hence, the above is a contradiction.

\subsection{Proof sketch for GM model}\label{sec:sketch_GM}

For brevity, we only show how to formulate the PO and the corresponding AO under the GM model. The rest of the proof follows mutandis-mutatis the content of Sections \ref{sec:proof_PT} and \ref{sec:SVM_prop}.

Recall that under the GM model, the feature vectors ${\bf x}_i$ are given by:
$
{\bf x}_i=y_i\boldsymbol{\eta}_0+{\bf z}_i
$
where ${\bf z}_i\sim\mathcal{N}({\bf 0},{\bf I}_d)$ and ${y}_i=\pm1$.  Thus, $$\w_i = y_i \betab_0 + \vb_i,~\vb_i\sim\Nn(0,1),$$ where recall that $\w_i:=\x_i(1:p)$, $\vb_i:=\z_i(1:p)$ and $\betab_0=\etab_0(1:p)$.
Replacing ${\bf w}_i=y_i\boldsymbol{\beta}_0+{\bf v}_i$ in \eqref{eq:origprob}, the SVM problem is equivalent to 
$$
\min_{\boldsymbol{\beta}}\max_{u_i\leq 0} ~\|\boldsymbol{\beta}\|_2^2+\frac{1}{\sqrt{n}}\sum_{i=1}^n u_i\boldsymbol{\beta}_0^{T}\boldsymbol{\beta}+\frac{1}{\sqrt{n}}\sum_{i=1}^n u_iy_i{\bf v}_i^{T}\boldsymbol{\beta} -\frac{1}{\sqrt{n}}\sum_{i=1}^n u_i.
$$
Different from the logistic model, here, $y_i$ and ${\bf v}_i$ are independent. Call $\tilde{\bf v}_i=y_i{\bf v}_i$. Then $\tilde{\bf v}_i$ has the same distribution as ${\bf v}_i$ by rotational invariance of the Gaussian distribution. Let ${\bf V}=\left[\tilde{\bf v}_1,\cdots,\tilde{\bf v}_n\right]^{T}$ and ${\bf u}=\left[u_1,\cdots,u_n\right]^{T}$. Then, in matrix form, the above problem becomes:
$$
\min_{\boldsymbol{\beta}}\max_{u_i\leq 0}~ \frac{1}{\sqrt{n}}{\bf u}^{T}{\bf V}\boldsymbol{\beta}-\frac{1}{\sqrt{n}}{\bf u}^{T}{\bf 1} + \|\boldsymbol{\beta}\|_2^2+\frac{{\bf 1}^{T}{\bf u}}{\sqrt{n}}\boldsymbol{\beta}_0^{T}\boldsymbol{\beta}.
$$
In this, we clearly recognize a PO problem with which we associate the following AO problem:
$$
\min_{\boldsymbol{\beta}}\max_{u_i\leq 0}~ \frac{1}{\sqrt{n}}\|\boldsymbol{\beta}\|_2 \g^{T}{\bf u}+\frac{1}{\sqrt{n}} \|{\bf u}\|_2\h^{T}\boldsymbol{\beta}+\frac{{\bf 1}^{T}{\bf u}}{\sqrt{n}}\boldsymbol{\beta}_0^{T}\boldsymbol{\beta}-\frac{1}{\sqrt{n}}{\bf u}^{T}{\bf 1} + \|\betab\|_2^2.
$$
Next, we sketch how to simplify this vector optimization to a scalar one. First decompose $\boldsymbol{\beta}$ as:
$
\boldsymbol{\beta}=\alpha_1\boldsymbol{\beta}_0+\alpha_2\boldsymbol{\beta}_{\perp},
$
where $\alpha_2\geq 0$ and $\boldsymbol{\beta}_{\perp}$ is orthogonal to $\boldsymbol{\beta}_0$ and $\|\boldsymbol{\beta}_{\perp}\|=1$. Similarly, let $\h=(\h^T\betab_0)\frac{\betab_0}{\|\betab_0\|_2} + P_\perp \h$, where $P_\perp$ is the projection operator to a subspace orthogonal to $\betab_0$.  
Optimizing over  $\boldsymbol{\beta}_{\perp}$ is straightforward: $\boldsymbol{\beta}_{\perp}=\frac{P_\perp\h}{\|P_\perp\h\|}$ and further using Lemma \ref{lem:Abla_help1}  the AO  becomes:
$$
\min_{\alpha_1,\alpha_2\geq 0} \max_{{\bf u}\leq 0} ~\alpha_1^2\|\boldsymbol{\beta}_0\|_2^2+\alpha_2^2 +\frac{1}{\sqrt{n}}{\bf u}^{T}\left(\sqrt{\alpha_1^2\|\boldsymbol{\beta}_0\|_2^2+\alpha_2^2}\,\g+{\alpha_1\|\betab_0\|_2^2}\,{\bf 1}-{{\bf 1}}\right)-\alpha_2{\|{\bf u}\|_2}\,\frac{\|P_\perp\h\|}{{\sqrt{n}}} - \alpha_1\|\ub\|_2\frac{\h^T\betab_0}{\sqrt{n}}.
$$
Next, let $\theta={\|{\bf u}\|_2}$ and set $q^2=\alpha_1^2\|\boldsymbol{\beta}_0\|_2^2+\alpha_2^2$. Optimizing over the direction of ${\bf u}$, the problem further simplifies to the following:
$$
\min_{\substack{q\geq 0\\
\alpha_1^2\|\boldsymbol{\beta}_0\|^2\leq q^2}}\max_{\theta\geq 0}~ q^2+\frac{\theta}{\sqrt{n}}\, \Big\|\left(q \g+\alpha_1\|\boldsymbol{\beta}_0\|^2{\bf 1}-{\bf 1}\right)_{-}\Big\|_2-\theta\,\sqrt{q^2-\alpha_1^2\|\betab_0\|^2}\,\frac{\|P_\perp\h\|}{\sqrt{n}} - \theta\alpha_1\frac{\h^T\betab_0}{\sqrt{n}}.
$$
Further defining $\rho=\frac{\alpha_1\|\betab_0\|_2}{q}$ gives:
\bea
\min_{\substack{q\geq 0\\ |\rho|\leq 1}}\max_{\theta\geq 0}~ q^2+q\theta\left\{ \frac{\Big\|\big(\g+ \|\betab_0\|_2 \,\rho\, {\bf 1}-\frac{1}{q}{\bf 1}\big)_{-}\Big\|}{\sqrt{n}}-\frac{\|P_\perp\h\|}{\sqrt{n}}\sqrt{1-\rho^2} - \frac{\rho}{\|\betab_0\|}\frac{\h^T\betab_0}{\sqrt{n}}\right\} .
\label{eq:dfa}
\eea
At this point, recognize that the optimization above is very similar to the corresponding optimization for the logistic model in \eqref{eq:similar1} (under the mapping $\sqrt{\mu^2+\alpha^2}\leftrightarrow q$ and $\mu\leftrightarrow \rho\,q$ and using that $\frac{\h^T\betab_0}{\sqrt{n}}$).
Concretely, in a similar manner to Section \ref{sec:SVM_AO}, consider the sequence of functions in the bracket above 
 $$\Ell^\dagger_n(q,\rho):= q\left\{\left\|\left(\frac{1}{\sqrt{n}}\g+\frac{\|\betab_0\|_2 \rho}{\sqrt{n}}{\one}-\frac{1}{q}\frac{{\bf 1}}{\sqrt{n}}\right)_{-}\right\|_2 - \  \sqrt{1-\rho^2}\frac{\|P_\perp\h\|_2}{\sqrt{n}}  -\frac{\rho}{\|\betab_0\|}\frac{\h^T\betab_0}{\sqrt{n}}\right\}$$ for $q\geq 0$ and $|\rho|\leq 1$. It can be easily seen that  $\Ell_n$ is converges (pointwise) almost surely to:
\begin{align}\label{eq:oellGM}
{\overline{\Ell}^\dagger:(q,\rho)\mapsto q\left\{ \sqrt{\mathbb{E}\left(\,G+\rho\,s-{1}\big/{q}\right)_{-}^2}-\sqrt{1-\rho^2} \sqrt{\kappa} \right\},}
\end{align}
where $G\sim\Nn(0,1)$
and we used the facts that $\frac{\h^T\betab_0}{\sqrt{n}}\ras 0$, $\frac{\|P_\perp\h\|_2}{\sqrt{n}}\ras \sqrt{\kappa}$ and $\|\betab_0\|_2\ras s$. Using convexity and level-boundedness arguments similar to what was done for the logistic model, it can be further shown that the optimal cost of the optimization in \eqref{eq:dfa}  converges almost surely to the optimal cost of the same optimization but with $\overline{\Ell}_n^\dagger(q,\rho)$ substituted by $\overline{\Ell}^\dagger(q,\rho)$.  In \eqref{eq:oellGM} recognize the resemblance to the function $\eta$ defined in Proposition \ref{propo:SVM} (cf. \eqref{eq:eta_func_1}). From this point on, the proof repeats the arguments of Sections \ref{sec:proof_PT} and \ref{sec:SVM_prop} and we omit the details. Just for an illustration, argue as follows to see how the phase-transition threshold of Proposition \ref{propo:PT} appears naturally. The AO of the hard-margin SVM in \eqref{eq:dfa} is infeasible if $\Ell^\dagger_n(q,\rho)$ is positive for all values of $q\geq 0$ and $\rho\in[-1,1]$. In the asymptotic limit, this happens if  $\kappa$ is such that:
$$
\min_{-1\leq \rho\leq 1} \sqrt{\mathbb{E}\left(G+\rho \,s\right)_{-}^2}-\sqrt{\kappa}\sqrt{1-\rho^2}>0\,.
$$
This above condition is equivalent to:
$$
\kappa \leq \min_{-1\leq \rho\leq 1} \mathbb{E}\Big(\frac{1}{\sqrt{1-\rho^2}}G+\frac{\rho}{\sqrt{1-\rho^2}}s\Big)_{-}^2\,.
$$
or, setting $t=\frac{\rho}{1-\rho^2}$,
$$
\kappa \leq \min_{t\in\R} \mathbb{E}(\sqrt{1+t^2}\,H+t\,s)_{-}^2.
$$
In the above line recognize the function $g(\kappa)$ defined in \eqref{eq:g_GM}.



\section{On the solutions to the equation $g(\kappa)=\kappa$ in Proposition \ref{propo:PT}}\label{sec:g(kappa)_proof}

\noindent\textbf{Proof of Proposition \ref{propo:PT}: Part III.}~
In this section, we study the function
$
g(\kappa):=\min_{t\in\R}\E\left(H+t\,G\,Y\right)_-^2.
$
Recall that this is a function of $\kappa$ since the distribution of the random variable $Y$ (cf. \eqref{eq:HGZ}) is a function of the signal strength  $s=s(\kappa)$. 
The following lemma summarizes the main result. This, together with \eqref{eq:proof_PT_step1} and \eqref{eq:proof_PT_step2}, complete the proof of Proposition \ref{propo:PT}.

\begin{lem}\label{lem:Ong} Let $s=s(\kappa)$ be a strictly increasing function of $\kappa\in[0,\zeta]$ for some $\zeta\geq 1$. Then, the function $g:\R_+\rightarrow\R$ defined as follows:
\bea\label{eq:g_again}
g(\kappa):=\min_{t\in\R}\E\left(H+t\,G\,Y\right)_-^2,\qquad Y={\rm Rad}\left(f\Big(\,s(\kappa)G+\sqrt{r^2-s^2(\kappa)}\,Z\,\Big)\right),\quad H,G,Z\simiid\Nn(0,1),
\eea
is strictly decreasing. Therefore, the equation $\kappa=g(\kappa)$ admits a unique solution $\kappa_*\in(0,1/2)$ and one of the following two statements holds true for any $\kappa$:
\bea
\kappa > \kappa_\star &\Rightarrow  \kappa>g(\kappa) \label{eq:app_EQ1}\\
\kappa < \kappa_\star &\Rightarrow  \kappa<g(\kappa).\label{eq:app_EQ2}
\eea
\end{lem}
\begin{proof}
Consider the function $\go:[0,1]\rightarrow\R$ defined as follows:
\bea
\go(\eps):=\min_{t\in\R}\E\left(H+t\,G\cdot{\rm Rad}\Big(r \,\big( \eps G + \sqrt{1-\eps^2} Z \big)\Big)\right)_-^2,
\eea
where the expectation is over $G,H,Z\simiid\Nn(0,1)$. Note that 
\bea\label{eq:g2go}
g(\kappa) = \go\big( s(\kappa)/r \big).
\eea
Lemma \ref{lem:g_decreasing} below, shows that $\go(\eps)$ is strictly decreasing in $\eps$ and $\go(0)=1/2$ (see also Figure \ref{fig:geps} for a numerical illustration). Here, we use these facts to prove the desired.

Towards this end, let $\kappa_1<\kappa_2$. Since $s(\cdot)$ is increasing, $s_1:=s(\kappa_1)<s(\kappa_2)=:s_2$. Thus, by \eqref{eq:g2go} and  Lemma \ref{lem:g_decreasing}:
$$
\kappa_1<\kappa_2 ~\Rightarrow~ s(\kappa_1)<s(\kappa_2) ~\Rightarrow~ \go\big( s(\kappa_1)/r \big) > \go\big( s(\kappa_2)/r \big) ~\Rightarrow~ g(\kappa_1)>g(\kappa_2).
$$
Thus, the function $g$ is strictly decreasing. Consider now the equation $g(\kappa)=\kappa$ for $\kappa\in[0,\zeta).$ Clearly, since $g$ is decreasing, if a solution $\kappa_\star$ exists, then it is unique. We now show that such a solution exists in the interval $[0,1/2]$. From Lemma \ref{lem:g_decreasing} the function $g$ is continuous and $g(\kappa)\leq 1/2$ for all $\kappa\geq 0$. For the shake of contradiction assume that there is no $\kappa\in[0,1/2]$ such that $g(\kappa)=\kappa$. Then, it must be that $g(\kappa)>\kappa$ for all $\kappa\in[0,1/2]$. In particular,  $g(1/2)>1/2$, which contradicts the decreasing nature of $g$ and $g(0)\leq  \go(0)=1/2$. 

It remains to prove \eqref{eq:app_EQ1} and \eqref{eq:app_EQ2}. We only prove the first one here, as the second one follows from the exact same argument. By decreasing nature of $g$ and definition of $\kappa_\star$ we have the following implications:
$$
\kappa>\kappa_\star\,\Rightarrow\, g(\kappa) < g(\kappa_\star) = \kappa_\star  \,\Rightarrow\, g(\kappa) < \kappa.
$$
This completes the proof of the lemma.
\end{proof}

%
 

\noindent\textbf{Decreasing nature of $g(\kappa)$.}~ First, we introduce some handy notation. For a (random) variable $X$ we use the following shorthand:
$$Y_X:=\Rad{\rhod(X)}.$$ 

\begin{lem}\label{lem:g_decreasing}
Fix $r>0$. Let $G,H,Z\simiid\Nn(0,1)$ and consider the function $\go:[0,1]\rightarrow\R$ defined as follows
\bea
\go(\eps):=\min_{t\in\R}\E\left(H+t\,G\,Y_{r\,Q}\right)_-^2,
\eea
where the random variable $Q$ is defined as $Q=\eps G + \sqrt{1-\eps^2}Z$. Then, $\go(0)=1/2$ and $\go$ is decreasing in $[0,1]$.
\end{lem}
\begin{proof}
First, we rewrite $\go(\eps)$ in a more convenient form. Note that 
$$H+t\,G\,Y_{r\,Q}=H\,Y_{r\,Q}^2+t\,G\,Y_{r\,Q}=(H\,Y_{r\,Q}+t\,G)Y_{r\,Q}\sim(H+t\,G)Y_{r\,Q}\sim T_1Y_{rQ},$$
where 
$$
\begin{bmatrix}
T_1 \\ Q
\end{bmatrix} \sim
\Nn\left(
\begin{bmatrix}
0\\ 0
\end{bmatrix} \, ,\,
\begin{bmatrix}
1+t^2 & t\eps \\ t\eps & 1 
\end{bmatrix} 
\right).$$
 Here, we have used the fact that $Y_{r\, Q}\in\{\pm1\}$ and $Q$ is independent of $H$. 
By further decomposing $T_1$ on its projection on $Q$, i.e., $$T_1=t\eps Q + \sqrt{1+t^2-\eps^2t^2}R,$$ 
with $R\sim\Nn(0,1)$, $\E[RQ]=0$, we find that
$$
H+t\,G\,Y_{r\,Q}\sim t\,\eps\,QY_{rQ} +  \sqrt{1+t^2-\eps^2t^2}R.
$$
Thus, we have shown that
\bea\label{eq:go2h}
\go(\eps) = \min_{t\in\R}\E\left(t\,\eps\,QY_{rQ} +  \sqrt{1+t^2-\eps^2t^2}R\right)_-^2 =:  \min_{t\in\R} h_\eps(t).
\eea
 From this, it is easy to see that the function $\go(\eps)$ is continuous and that $\go(0)=1/2$. 
 
 Next, we show that the function $\go(\eps)$ is decreasing. Towards this goal,
%
%
let $\eps_1<\eps_2$. It will suffice to show that $g(\eps_1)\geq g(\eps_2)$. 
Denote $\delta:=\frac{\eps_1}{\eps_2}<1$.For any $t\in\R$, we have the following chain of inequalities
\bea\nn
h_{\eps_2}(t\delta) &=  \E\left( t\,\eps_1\, Q\,Y_{rQ} + \sqrt{1-t^2\eps_1^2+t^2\delta^2}\,R \right)_-^2 \\
&<\E\left( t\,\eps_1\, Q\,Y_{rQ} + \sqrt{1-t^2\eps_1^2+t^2}\,R \right)_-^2 =  h_{\eps_1}(t ),\label{eq:fixt}
\eea
where the inequality follows from Lemma \ref{lem:inc1} below and the fact that $\delta<1$. The desired claim follows by minimizing both sides of the inequality in \eqref{eq:fixt} with respect to $t$ and invoking \eqref{eq:go2h}.
\end{proof}

\begin{lem}\label{lem:inc1}
Fix $t_1\in\R$ and $t_2\geq 0,r\geq 0, t_3>0$. Consider the function  $f(x)= \E[(t_1\,Q\,Y_{rQ}+\sqrt{t_2+t_3 x}\,R)_-^2]$ where $Q, Y_{rQ}$ are defined as in Lemma \ref{lem:g_decreasing} and $R\sim\Nn(0,1)$. Then, the function $f$ is increasing in $[0,1]$.
\end{lem}
\begin{proof}
The claim follows by direct differentiation of $f$. Formally, let $F(Q,R;x):=(t_1\,Q\,Y_{rQ}+\sqrt{t_2+t_3 x}\,R)_-^2$. Then, $|\frac{\vartheta}{\vartheta x}F(Q,R;x)|:=\frac{t_3}{\sqrt{t_2+t_3x}}|R(t_1\,Q\,Y_{rQ}+\sqrt{t_2+t_3 x}\,R)_-|\leq \frac{t_3}{\sqrt{t_2}}(|t_1||R||Q|+\sqrt{t_2+t_3}R^2)=:M$ a.s. for all $x\in[0,1]$. Clearly, $ \E[M]<\infty$, thus, by dominated convergence theorem, we have
\bea\nn
f^\prime(x)=\frac{t_3}{\sqrt{t_2+t_3x}}\,\E\left[\,R(t_1\,Q\,Y_{rQ}+\sqrt{t_2+t_3 x}\,R)_-\right] = t_3\E\left[\ind{t_1\,Q\,Y_{rQ}+\sqrt{t_2+t_3 x},\,R<0}\right]> 0 \,, 
\eea
where the second equality follows by Gaussian integration by parts.
\end{proof}

\begin{figure}[h!]
  \begin{center}
    \includegraphics[width=.5\textwidth]{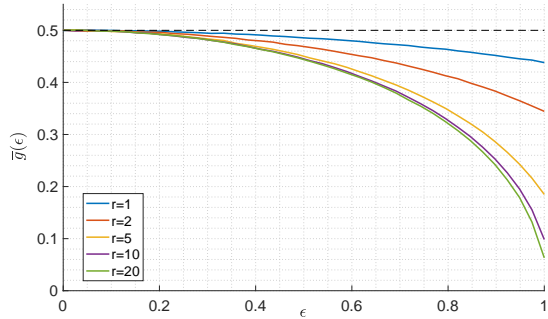}
  \end{center}
  \caption{Numerical illustration of the monotonicity of the threshold function $\go(\eps)$ of Lemma \ref{lem:g_decreasing} for several values of $r$.}\label{fig:threshold_monotone}
  \label{fig:geps}
\end{figure}

\section{Properties of the function $\eta(q,\rho)$ of Proposition \ref{propo:SVM} }\label{app:eta}

\begin{lem}\label{lem:eta_properties}
Consider the function $\eta:\R_+\times[-1,1]\rightarrow\R$
$$
\eta(q,\rho):=\min_{-1\leq \rho\leq 1} \E\left(\rho GS + H\sqrt{1-\rho^2}-1/q\right)_-^2 - (1-\rho^2)\,\kappa.
$$
Further define function  $\etao:\R_+\rightarrow\R$ as follows:
$$\etao(q):= \min_{-1\leq \rho\leq 1}\eta(q,\rho).$$
The following statements are true:

~~\noindent(i).~ The function $q\mapsto \eta(q,\rho)$ is strictly decreasing for all $\rho\in[-1,1]$.

~~\noindent(ii).~ The function $\etao$ is strictly decreasing.
\end{lem}

\begin{proof} We prove each one of the two statements separately.

\noindent(i).~ This holds because $x\mapsto(x)_-^2$ is strictly decreasing for $x<0$ and the measure of the random variable $\rho GY+ \sqrt{1-\rho^2} H$ is strictly positive on the real line. 

\noindent(ii).~ Fix $q_2>q_1>0$. Let $\rho_1\in[-1,1]$ be such that $\eta(q_1,\rho_1) = \etao(q_1)$. It holds,
$$
\etao(q_2) \leq  \eta(q_2,\rho_1) <  \eta(q_1,\rho_1) = \etao(q_1),
$$
where the second inequality follows from the first statement of the lemma.

\end{proof}

\noindent\textbf{Unique zero of $q\mapsto \min_{-1\leq \rho\leq 1}\eta(q,\rho)$:}~
We show here that the function $\etao:\R_+\rightarrow\R$:
$$\etao(q):= \min_{-1\leq \rho\leq 1}\eta(q,\rho) = \min_{-1\leq \rho\leq 1} \E\left(\rho GS + H\sqrt{1-\rho^2}-1/q\right)_-^2 - (1-\rho^2)\,\kappa,$$
 defined in \eqref{eq:eta_func} admits a unique zero $q^\star$ provided that . Recall that $\kappa>\kappa_\star$ implies $\kappa > g(\kappa)$ (see Lemma \ref{lem:Ong}).

First, we show that such a zero  exists. From the maximum Theorem \cite[Theorem 9.7]{Sundaram}, the function $\etao(q)$ is continuous. Moreover, we will show that
\bea
\lim_{q\rightarrow0}\etao(q) &= \infty,\label{eq:q_0}\\
\lim_{q\rightarrow\infty}\etao(q) &< 0.\label{eq:q_inf}
\eea
These combined prove existence of a solution to $\etao(q)=0$; thus it remains to prove \eqref{eq:q_0} and \eqref{eq:q_inf}. For the former, we argue as follows:
\begin{align*}
\lim_{q\rightarrow0} \etao(q) &\geq  \lim_{q\rightarrow0} \min_{-1\leq \rho\leq 1} \E\left(\rho GS + H\sqrt{1-\rho^2}-1/q\right)_-^2 \\
&\geq \lim_{q\rightarrow0} \min_{-1\leq \rho\leq 1} \E\Big[\left(\rho GS + H\sqrt{1-\rho^2}-1/q\right)_-^2 1_{\{\rho GS + H\sqrt{1-\rho^2}\leq 0\}}\Big]\\
&\geq \lim_{q\rightarrow0} \min_{-1\leq \rho\leq 1} \frac{1}{q^2} = \infty\,,
\end{align*}
where in the last inequality we have used the facts that $x\mapsto(x)_-^2$ is decreasing and that the event $\{\rho G S + H \sqrt{1-\rho^2}\leq 0\}$ has nonzero measure for all $\rho\in[-1,1]$. On the other hand, for \eqref{eq:q_inf} we follow the following chain of inequalities
\begin{align}
\lim_{q\rightarrow\infty} \etao(q) &= \lim_{1/q\rightarrow0^+} \etao(q) = \min_{-1\leq \rho\leq 1} \E\left(\rho GS + H\sqrt{1-\rho^2}\right)_-^2 - (1-\rho^2)\kappa \nn\\
&\leq \min_{-1<\rho< 1} \E\left(\rho GS + H\sqrt{1-\rho^2}\right)_-^2 - (1-\rho^2)\kappa\nn\\
&= \min_{-1< \rho< 1} (1-\rho^2)\, \left( \E\Big(\frac{\rho}{\sqrt{1-\rho^2}} GS + H\Big)_-^2 -\kappa \right)\nn\\
&\leq \min_{-1< \rho< 1} \E\Big(\frac{\rho}{\sqrt{1-\rho^2}} GS + H\Big)_-^2\nn \\
&\leq \min_{t\in\R} \E\Big(t GS + H\Big)_-^2 - \kappa = g(\kappa) - \kappa < 0,
\end{align}
where the second inequality in the first line follows by continuity of $\etao(q)$ and the last (strict) inequality follows from $\kappa > g(\kappa)$.

Next, we prove that such a zero must be unique. For that we assume that there exists $q_1^\star$ and $q_2^\star$ such that:
$$
\min_{-1\leq \rho\leq 1} \eta(\rho,q_1^\star)=\min_{-1\leq \rho\leq 1} \eta(\rho,q_2^\star).
$$
Denote by $\rho_1$ and $\rho_2$ the real values in the interval $[-1,1]$ such that $\eta(q_1^\star,\rho_1)=\min_{-1\leq \rho\leq 1}\eta(q_1^\star,\rho)=0$, and   $\eta(q_2^\star,\rho_2)=\min_{-1\leq \rho\leq 1}\eta(\rho,q_2^\star)=0$. Then, by optimality of $\rho_1$,
$$
0=\eta(q_1^\star,\rho_1)\leq \eta(q_1^\star,\rho_2).
$$
Since $0=\eta(q_2^\star,\rho_2)$ and $q\mapsto \eta(q,\rho_2)$ is decreasing (see Lemma \ref{lem:eta_properties}),  $q_1^\star\geq q_2^\star$. Similarly, we can use the same reasoning to prove that $q_2^\star\geq q_1^\star$. We thus necessarily have $q_1^\star=q_2^\star$, which proves the uniqueness of $q^\star$. 

\vspace{5pt}
{\noindent\textbf{Unique minimizer of $\rho\mapsto \eta(\rho,q^\star)$:}}
Let $\rho_1$ and $\rho_2$ be two minimizers of  $\rho\mapsto \eta(\rho,q^\star)$ . 
Define $\mu_1=q^\star \rho_1$, $\alpha_1=q^\star\sqrt{1-\rho_1^2}$, $\mu_2=q^\star\rho_2$ and $\alpha_2=q^\star \sqrt{1-\rho_2^2}$,  where $q^\star$ is the unique minimizer of $\overline{\eta}$. Then, in view of \eqref{eq:ineta}, $(\alpha_1,\mu_1)$ and $(\alpha_2,\mu_2)$ are two different minimizers of the following minimization problem:
\bea\nn
\min_{ \substack{ (\alpha,\mu)\in\Cc_{k_0} \\ \overline{\Ell}(\alpha,\mu) \leq 0 }} \alpha^2+\mu^2.
\eea
As $(\alpha,\mu)\mapsto \alpha^2+\mu^2$ is jointly strictly convex in its variables and $(\alpha,\mu)\mapsto \overline{\Ell}(\alpha,\mu)$ is convex, we have necessarily $\alpha_1=\alpha_2$ and $\mu_1=\mu_2$.  Hence, $\rho_1=\rho_2$. 

\section{Logistic regression}\label{sec:proof_ML}

We only present the proof for the logistic model \eqref{eq:yi_log} as there is nothing fundamentally changing for the GM model \eqref{eq:yi_GM}. Specifically, the PO can be identified following the exact same procedure as in Section \ref{sec:sketch_GM} and the analysis of the AO uses the same arguments as what is presented below.

\subsection{Proof of Proposition \ref{propo:ML}}
\noindent\textbf{Identifying the PO.}
The logistic-loss minimization in \eqref{eq:ML} is equivalent to: 
\bea\label{eq:mmbn}
\min_{\ub,\betab} \max_{\vb} \frac{1}{n} \sum_{i=1}^{n}\ell(u_i) - \frac{1}{n}\sum_{i=1}^{n}v_i u_i +\frac{1}{n}\sum_{i=1}^{n}v_i y_i \w_i^T \betab,
\eea
where, recall from \eqref{eq:y_pf_SVM}, that the responses $y_i$ depend on the feature vectors $\w_i$ used for training as follows: $y_i \sim \Rad{f(\betab_0^T\,\w_i+\sigma z_i)},~z_i\sim\Nn(0,1).$
By rotational invariance of the Gaussian distribution of the feature vectors,  we assume without loss of generality that $\betab_0 = [s,0,...,0]^T$. Further let us define 
 \bea\nn
 \w_i = [w_i,\tilde{\w}_i^T]^T\quad\text{ and }\quad\betab = [\mu, \tilde{\betab}^T]^T,
 \eea 
such that $w_i$ and $\mu$ are the first entries of $\w_i$ and $\betab$, respectively. In this new notation
\bea\label{eq:y_pf}
y_i \sim \Rad{f(s\,w_i+\sigma z_i)},~~~z_i\sim\Nn(0,1),
\eea
%
%
%
and the minimization of the logistic loss in \eqref{eq:mmbn}  becomes
$$
\min_{{\bf u},\tilde{\betab},\mu} \max_{\bf v} \frac{1}{n}\sum_{i=1}^n \ell(u_i)-\frac{1}{\sqrt{n}}\sum_{i=1}^n v_i u_i+\frac{1}{\sqrt{n}}\sum_{i=1}^n v_i y_i \tilde{\bf w}_i^{T}\tilde{\betab} +\frac{1}{\sqrt{n}}\sum_{i=1}^n v_i y_i w_i \mu\,.
$$
Equivalently, in matrix form:
\begin{equation}
{\Phi}_{\mathcal{L}}^{(n)}:=\min_{\tilde{\betab},\mu,\ub} \max_{\bf v} \frac{1}{\sqrt{n}}{\bf v}^{T}{\bf D}_y\tilde{\bf W}\tilde{\betab} + \frac{1}{\sqrt{n}}\mu {\bf v}^{T}{\bf D}_y \boldsymbol{\chi} +\frac{1}{\sqrt{n}}{\bf v}^{T}{\bf u} +\frac{1}{n}\sum_{i=1}^n\ell(u_i),
	\label{eq:prop}
\end{equation}
where ${\bf D}_y:={\rm diag}(y_1,\cdots,y_n)$, $\boldsymbol{\chi}=\left[w_1,\cdots,w_n\right]^{T}$ and $\tilde{\bf W}$ is a $n\times (p-1)$ matrix with rows $\tilde{\bf w}_i^{T}$, $i=1,\cdots,n$.  
Problem \eqref{eq:prop} is brought to the form required by the CGMT with the exception that the optimization sets of the variables $(\betab,\mu)$ and ${\bf v}$ are not compact. Thus, we apply the extended versions of the CGMT in Appendix \ref{sec:CGMT_new}. Specifically, in order to use Corollary \ref{cor:characterization}, we apply the recipe prescribed in Remark \ref{rem:recipeCor} as follows.

For $\Bc>0$ large enough constant (to be specified later), consider the set $$\mathcal{S}^{(n)}_{\mathcal{B}}=\left\{(\tilde{\betab},\mu)\in \mathbb{R}^{p-1}\times \mathbb{R} \ \ | \ \ \|\tilde{\betab}\|^2+\mu^2\leq \mathcal{B}^2\right\},$$ and the ``bounded version" of \eqref{eq:prop} given by:
\begin{equation}
{\Phi}_{\mathcal{L,B}}^{(n)}:=\min_{(\tilde{\betab},\mu) \in \mathcal{S}^{(n)}_{\mathcal{B}} , \ub} ~\max_{\bf v} ~{\frac{1}{\sqrt{n}}{\bf v}^{T}{\bf D}_y\tilde{\bf W}\tilde{\betab} + \frac{1}{\sqrt{n}}\mu {\bf v}^{T}{\bf D}_y \boldsymbol{\chi} +\frac{1}{\sqrt{n}}{\bf v}^{T}{\bf u} +\frac{1}{n}\sum_{i=1}^n\ell(u_i)},.
\label{eq:prop_bounded}
\end{equation}
To connect \eqref{eq:prop_bounded} to \eqref{eq:prop} we rely on Lemma \ref{lemma:5}, which ensures that if there exists $\alpha^\star$ and $\mu^\star$ such that for sufficiently large $\mathcal{B}$, the optimizers $\tilde{\betab}_{\mathcal{B}}$ and $\hat{\mu}_{\mathcal{B}}$  of \eqref{eq:prop_bounded} satisfies  $\|\tilde{\betab}_{\mathcal{B}}\|\to \alpha^\star$ and $\hat{\mu}_\Bc\to \mu^\star$, then any minimizer $\hat{\betab}=[\hat{\mu},\tilde{\betab}^T]^T$ \eqref{eq:prop} satisfy  $\|\tilde{\betab}\|\to \alpha^\star$ and $\hat{\mu}\to \mu^\star$, as well. 
In view of this and Remark \ref{rem:recipeCor}, we focus onwards on analyzing the bounded PO problem \eqref{eq:prop_bounded} as prescribed by Corollary \ref{cor:characterization}.

\vp
\noindent\textbf{Forming the AO.}
In view of \eqref{eq:AOpr}, we associate the PO with a sequence of bounded AO problems as follows:
\begin{equation}
{\phi}_{\mathcal{L,B},\Gamma}^{(n)}:=\min_{{(\mu,\tilde{\betab})\in\mathcal{S}^{(n)}_{\mathcal{B}}}, \ub} ~\max_{\substack{{\bf v}\\ \|{\bf v}\|\leq \Gamma}}~ \ \frac{1}{\sqrt{n}} \|\tilde{\betab}\|_2{\bf g}^{T}{\bf D}_y{\bf v} +\frac{1}{\sqrt{n}}\|{\bf D}_y{\bf v}\|_2 {\bf h}^{T}\tilde{\betab} -\frac{1}{\sqrt{n}}\mu {\bf v}^{T}{\bf D}_y\boldsymbol{\chi} + \frac{1}{n}{\bf v}^{T}{\bf u} +\frac{1}{\sqrt{n}}\sum_{i=1}^n\ell(u_i)\,,
	\label{eq:AO}
\end{equation}
where ${\bf g}\sim \mathcal{N}({\bf 0},{\bf I}_n)$  and ${\bf h}\sim\mathcal{N}({\bf 0},{\bf I}_{p-1})$. 
Having identified the sequence of AO problems, we proceed by simplifying them to problems involving only a few number of scalar variables. As will be seen later, this will facilitate inferring their asymptotic behavior as required by the conditions of Corollary \ref{cor:characterization}.\\

\vp
\noindent{\textbf{Scalarization of the AO.}} 
As ${y}_i=\pm 1$, ${\bf D}_y{\bf g}\sim \mathcal{N}({\bf 0},{\bf I}_n)$ and $\|{\bf D}_y{\bf v}\|_2=\|{\bf v}\|$. Denote by $\alpha:=\|\tilde{\betab}\|_2$. For any vector ${\bf v}$, we notice that the objective in   \eqref{eq:AO} is minimized when $\tilde{\betab}$ aligns with $-{\bf h}$. Based on this observation and using Lemma \ref{lem:Abla_help1}, \eqref{eq:AO} becomes:
$$
{\phi}_{\mathcal{L,B},\Gamma}^{(n)}=\min_{\substack{\ub,\alpha\geq 0 \\ \alpha^2+\mu^2\leq \mathcal{B}^2}} \max_{\substack{{\bf v} \\ \|{\bf v}\|\leq \Gamma }} \frac{1}{\sqrt{n}}\alpha {\bf g}^{T}{\bf h}-\frac{\alpha}{\sqrt{n}}\|{\bf v}\|_2\|{\bf h}\|_2  -\frac{1}{\sqrt{n}} \mu \boldsymbol{\chi}^{T}{\bf D}_{\bf y}{\bf v} + \frac{1}{\sqrt{n}}{\bf v}^{T}{\bf u} + \frac{1}{\sqrt{n}}\sum_{i=1}^{n}\ell(u_i)\,.
$$
For convenience, denote
$$
\mathcal{S}_{\mathcal{B}}=\left\{(\alpha,\mu)\in \mathbb{R}_+\times \mathbb{R} \ \ | \ \ \alpha^2+\mu^2\leq \mathcal{B}^2\right\}.
$$
To continue, let $\tilde{\gamma}=\|{\bf v}\|_2$ and optimize over the direction of ${\bf v}$ to find:
\begin{equation}
	{\phi}_{\mathcal{L,B},\Gamma}^{(n)}=\min_{(\alpha,\mu)\in\Sc_\Bc,\ub} ~\max_{0\leq\tilde{\gamma}\leq \Gamma} ~\frac{\tilde{\gamma}}{\sqrt{n}} \|\alpha{\bf g}-\mu {\bf D}_{\bf y}\boldsymbol{\chi}+{\bf u}\|_2-\frac{\alpha}{\sqrt{n}} \tilde{\gamma} \|{\bf h} \|_2+\frac{1}{n}\sum_{i=1}^n \ell(u_i)\,.
\label{eq:AOs}
\end{equation}
Note that the objective function above is convex in $\ub$ and concave in $\tilde\gamma$. Furthermore, $\tilde\gamma$ is optimized over a compact set. Thus, but Sion's min-max theorem we may flip the order of min-max between $\ub$ and $\tilde\gamma$. Moreover, we apply the following trick in order to turn the minimization over $\ub$ into a separable optimization over its entries. We use the fact that for all $x\in\mathbb{R}$, $\min_{\tilde{r}>0} \frac{\tilde{r}}{2}+\frac{x^2}{2\tilde{r} n}=\frac{x}{\sqrt{n}}$ and apply this to $x:=\|\alpha {\bf g}-\mu {\bf D}_{\bf y}\boldsymbol{\chi}+{\bf u}\|_2$. With these, we can equivalently express \eqref{eq:AOs} as:
\bea
{\phi}_{\mathcal{L,B},\Gamma}^{(n)}&=\min_{(\alpha,\mu)\in\Sc_\Bc}~\max_{0\leq\tilde{\gamma}\leq \Gamma}~\inf_{\tilde{r}\geq 0}~\frac{\tilde{\gamma}\tilde{r}}{2}-\frac{\alpha}{\sqrt{n}}\tilde{\gamma}\|{\bf h}\|_2+\min_{\ub}\big\{\frac{\tilde{\gamma}}{2\tilde{r} n} \|\alpha {\bf g}-\mu {\bf D}_{\bf y} \boldsymbol{\chi} + {\bf u}\|_2^2+\frac{1}{n}\sum_{i=1}^n\ell(u_i)\big\}\nn\\
&= \min_{(\alpha,\mu)\in\Sc_\Bc}~\max_{0\leq\tilde{\gamma}\leq \Gamma}~\inf_{\tilde{r}\geq 0}~\frac{\tilde{\gamma}\tilde{r}}{2}-\frac{\alpha}{\sqrt{n}}\tilde{\gamma}\|{\bf h}\|_2+ \frac{1}{n}\sum_{i=1}^n e_{\ell}\big(\alpha g_i+\mu y_iw_i;{\tilde r}/{\tilde\gamma}\big)\,,\label{eq:inter22}
\eea
where in the second line we have denoted the Moreau envelope of the logistic loss function as:
$$
e_\ell(x;\tau) := \min_u \frac{1}{2\tau} (x-u)^2 + \ell(u).
$$
The Moreau envelope is convex in its second argument (e.g., see \cite[Lem.~A.2]{Hossein2020}). Thus, the objective function in \eqref{eq:inter22} is convex in $\tilde r$. Moreover, the objective is concave with respect to $\tilde\gamma$ as the point-wise minimum (over $\ub$) of linear functions. From these and compactness of the feasibility set of $\tilde\gamma$ we  flip the order of max-min between $\tilde\gamma$ and $\tilde r$, to find that
\bea
{\phi}_{\mathcal{L,B},\Gamma}^{(n)}= \min_{(\alpha,\mu)\in\Sc_\Bc}~\inf_{\tilde{r}\geq 0}~\max_{0\leq\tilde{\gamma}\leq \Gamma}~\frac{\tilde{\gamma}\tilde{r}}{2}-\frac{\alpha}{\sqrt{n}}\tilde{\gamma}\|{\bf h}\|_2+ \frac{1}{n}\sum_{i=1}^n e_{\ell}\big(\alpha g_i+\mu y_iw_i;{\tilde r}/{\tilde\gamma}\big)\,.\label{eq:inter23}
\eea

At this point we recognize that the objective function in \eqref{eq:AOs} is (jointly) convex in $(\alpha,\mu,\tilde r)$; see \cite[Lem.~A.2]{Hossein2020} for a proof. Therefore, we can apply the ``meta-theorem" of Remark \ref{rem:minsup} according to which, it suffices to consider the ``unbounded" AO problem defined as the limit of ${\phi}_{\mathcal{L,B},\Gamma}^{(n)}$ as $\Gamma$ grows to infinity. We denote this by $\phi_{\mathcal{L,B}}^{(n)}$ and is given by:
\bea
{\phi}_{\mathcal{L,B}}^{(n)}= \min_{(\alpha,\mu)\in\Sc_\Bc}~\inf_{\tilde{r}\geq 0}~\sup_{\tilde{\gamma}\geq 0}~\frac{\tilde{\gamma}\tilde{r}}{2}-\frac{\alpha}{\sqrt{n}}\tilde{\gamma}\|{\bf h}\|_2+ \frac{1}{n}\sum_{i=1}^n e_{\ell}\big(\alpha g_i+\mu y_iw_i;{\tilde r}/{\tilde\gamma}\big)\,.\label{eq:AOsinf}
\eea
To complete the ``scalarization" of the AO we further simplify \eqref{eq:AOsinf} by performing the change of variable $\lambda:=\frac{\tilde\gamma\alpha}{\tilde{r}}$:
\bea
{\phi}_{\mathcal{L,B}}^{(n)}=\min_{(\alpha,\mu)\in\Sc_\Bc} \inf_{\tilde{r}\geq 0}\sup_{\lambda\geq 0} \frac{\lambda\tilde{r}^2}{2\alpha} -\frac{1}{\sqrt{n}}\lambda\tilde{r}\|{\bf h}\|_2 +\frac{1}{n}\sum_{i=1}^n e_{\ell}\big(\alpha g_i+\mu y_iw_i;{\alpha}/{\lambda}\big)\,.
\eea
Now, note that for any $\lambda>0$ and $\alpha>0$, the optimum $\tilde{r}$ is given by $\frac{\alpha}{\sqrt{n}} \|{\bf h}\|_2$. Hence using Lemma  \ref{lem:Abla_help1},} we can replace in the above optimization problem $\tilde{r}$ by its optimal  value. In doing so, we obtain:
\bea
{\phi}_{\mathcal{L,B}}^{(n)}=\min_{(\alpha,\mu)\in\Sc_\Bc}~ \sup_{\lambda\geq 0}~ -\frac{ \alpha \lambda}{2n} \|{\bf h}\|_2^2 +\frac{1}{n}\sum_{i=1}^n e_{\ell}\big(\alpha g_i+\mu y_iw_i;{\alpha}/{\lambda}\big)\,.
\eea
The optimization above  is jointly convex in $(\alpha,\mu)$ and concave in $\lambda$. Thus, we may flip the order of the min-max, which yields:
\begin{align*}
{\phi}_{\mathcal{L,B}}^{(n)}=\sup_{\lambda\geq 0} ~\min_{(\alpha,\mu)\in\Sc_\Bc}~  -\alpha \lambda \frac{\|{\bf h}\|_2^2}{2n} +   \frac{1}{n}\sum_{i=1}^n e_{\ell}(\alpha g_i+\mu y_iw_i;\frac{\alpha}{\lambda}) \,.
\end{align*}
\vp
\noindent \textbf{Asymptotic behavior of ${\phi}_{\mathcal{L,B}}^{(n)}$.} 
Define,
$$
\mathcal{R}_n(\alpha,\mu,\lambda)=-\alpha\lambda \frac{\|{\bf h}\|^2}{2n}+\frac{1}{n}\sum_{i=1}^n e_\ell(\alpha g_i+\mu y_iw_i;\frac{\alpha}{\lambda}),
  \ \ \ \ \textnormal{and} \ \ \ \ \mathcal{R}(\alpha,\mu,\lambda)=-\alpha\lambda \frac{\kappa}{2}+ \mathbb{E}\left[e_\ell(\alpha G+\mu YH;\frac{\alpha}{\lambda})\right].
$$
For any $\lambda>0$, the function $(\alpha,\mu)\mapsto \mathcal{R}_n(\alpha,\mu,\lambda)$ is jointly convex in $(\alpha,\mu)$ (e.g., \cite[Lem.~A2]{Hossein2020}) and converges to $(\alpha,\mu)\mapsto \mathcal{R}(\alpha,\mu,\lambda)$. 
But, convergence of convex functions is uniform over compact sets \cite[Cor.~II.1]{AG1982}. It thus converges uniformly over the set $\left\{(\alpha,\mu) \ | \ \alpha \geq 0,  \ \alpha^2+\mu^2\leq \mathcal{B}^2 \ \right\}$. Hence, for any $\lambda>0$, the function
\bea \label{eq:lamfun}
\lambda\mapsto \min_{(\alpha,\mu)\in\Sc_\Bc} \mathcal{R}_n(\alpha,\mu,\lambda)\,,
\eea
converges pointwise to
\bea \label{eq:lamfun_der}
\lambda \mapsto \min_{(\alpha,\mu)\in\Sc_\Bc} \mathcal{R}(\alpha,\mu,\lambda)\,.
\eea
Next, we argue that the convergence above is essentially uniform. To see this note that the functions in \eqref{eq:lamfun} and \eqref{eq:lamfun_der} are concave (as the pointwise minimum of concave functions). Moreover, the function is level-bounded, which can be shown as follows.
For $\mathcal{B}\geq \sqrt{2}$, it can be checked that
\begin{align}
	\min_{(\alpha,\mu)\in\Sc_\Bc} \mathcal{R}(\alpha,\mu,\lambda)&\leq -\lambda\frac{\kappa}{2}+\mathbb{E}\left[e_\ell(H+YG;\frac{1}{\lambda})\right]\overset{\lambda\to\infty}{\longrightarrow} -\infty,
\end{align}
where we have used the fact that $\mathbb{E}\left[e_\ell(H+YG;\frac{1}{\lambda})\right]\overset{\lambda\to\infty}{\longrightarrow} ~\E[\ell(H+YG)]<\infty$ (see \cite[Thm.~1.25]{RocVar}). This proves that:
$$
\lim_{\lambda\to\infty} \min_{(\alpha,\mu)\in\Sc_\Bc} \mathcal{R}(\alpha,\mu,\lambda)=-\infty, 
$$
and hence, the function \eqref{eq:lamfun_der} is level-bounded. We may now use \cite[Lem.~10]{Master} to conclude that
$$
\phi_{\mathcal{L,B}}^{(n)}=\sup_{\lambda\geq 0}  \min_{(\alpha,\mu)\in\Sc_\Bc} \  \mathcal{R}_n(\alpha,\mu,\lambda) ~\ras~ \overline{\phi}_{\mathcal{L,B}}:=\sup_{\lambda\geq 0}  \min_{(\alpha,\mu)\in\Sc_\Bc} \  \mathcal{R}(\alpha,\mu,\lambda) \,.
$$
Note that we can flip the order of the min-max in $\overline{\phi}_{\mathcal{L,B}}$ since $\mathcal{R}$ is convex in $(\alpha,\mu)$ and concave in $\lambda$. We thus also have:
\bea\label{eq:phio_LB}
\overline{\phi}_{\mathcal{L,B}}= \min_{(\alpha,\mu)\in\Sc_\Bc}\sup_{\lambda \geq 0} -\alpha\lambda\frac{\kappa}{2}+\mathbb{E}\left[e_{\ell}(\alpha H+\mu YG;\frac{\alpha}{\lambda})\right]\,.
\eea

\vp
\noindent \textbf{Checking the conditions of Corollary \ref{cor:characterization}.} Let $\overline{\phi}_{\mathcal{L}}$ be given by:
\bea\label{eq:phio_L}
\overline{\phi}_{\mathcal{L}}:=\min_{\substack{\alpha\geq 0\\ \mu}}\sup_{\lambda \geq 0}\, -\alpha\lambda\frac{\kappa}{2}+\mathbb{E}\left[e_{\ell}(\alpha H+\mu YG;\frac{\alpha}{\lambda})\right]\,.
\eea
Assume that $\overline{\phi}_{\mathcal{L}}$ has a unique minimizer $(\alpha^\star,\mu^\star)$, which will be proven later. Then, for $\mathcal{B}$ sufficiently large (e.g., any $\Bc > 2((\alpha^\star)^2+(\mu^\star)^2)$), 
$$
\overline{\phi}_{\mathcal{L}}=\overline{\phi}_{\mathcal{L,B}}. 
$$
Based on this and on the previous analysis, we have thus far shown that $\overline{\phi}_{\mathcal{L,B}}$ converges almost surely to $\overline{\phi}_{\mathcal{L}}$, i.e.,
\begin{subequations}
	\begin{align}
 &\text{For any}~ \eps>0:\quad\mathbb{P}\left[ \phi_{\mathcal{L,B}}^{(n)} \leq \overline{\phi}_{\mathcal{L}}-\epsilon, ~~\text{i.o.}\right]=0,\label{eq:cond1phi_L_au}\\
	 &\text{For any}~ \eps>0:\quad\mathbb{P}\left[ \phi_{\mathcal{L,B}}^{(n)} \geq \overline{\phi}_{\mathcal{L}}+\epsilon, ~~\text{i.o.}\right]=0\label{eq:cond2phi_L_au}.
	\end{align}
\end{subequations}
Using a ``deviation argument" and uniqueness of solutions of $\overline{\phi}_{\mathcal{L}}$, which will be shown next, it can be further shown that
	\begin{align}
 	 &\text{There exists}~ \zeta>0:\quad\mathbb{P}\left[ \phi_{\mathcal{L,B}}^{(n)} \leq \overline{\phi}_{\mathcal{L}}+\zeta, ~~\text{i.o.}\right]=0\label{eq:cond3phi_L_au}.
	\end{align}
\noindent The deviations argument is similar to the proof of SVM in Section \ref{sec:SVM_prop} and we omit the details. 

By reference to Remark \ref{rem:minsup} observe that \eqref{eq:cond1phi_L_au}, \eqref{eq:cond2phi_L_au} and \eqref{eq:cond3phi_L_au} correspond to  \eqref{eq:unbounded_AO_rel}, \eqref{eq:minsupAO} and \eqref{eq:unbounded_AO_rel2}, respectively.  Therefore, invoking Corollary \ref{cor:characterization}, we conclude that for $\mathcal{B}$ sufficiently large, $\tilde{\betab}_{\mathcal{B}}$ and $\hat{\mu}_{\mathcal{B}}$ satisfy $\|\tilde{\betab}_{\mathcal{B}}\|\to \alpha^\star$ and $\hat{\mu}_{\mathcal{B}}\to \mu^\star$. As mentioned, based on Lemma \ref{lemma:5}, this implies that any minimizer $\hat{\betab}=[\hat{\mu},\tilde\betab^T]^T$  of \eqref{eq:prop} satisfies
\bea\label{eq:rP_T}
\hat\mu\rP \mu^\star \quad\text{and}\quad \|\tilde{\betab}\|_2\rP \alpha^\star.
\eea
This proves the desired formulae for the cosine similarity and the risk similar to \eqref{eq:egw2} in Section \ref{sec:pre_SVM}.

\vp
\noindent \textbf{Uniqueness of the minimizer of $\overline{\phi}_{\mathcal{L}}$}.
Our goal here is to prove that if $\kappa<\kappa_\star$, then $\overline{\phi}_{\mathcal{L}}$ given by
$$
\overline{\phi}_{\mathcal{L}}= \min_{\substack{\alpha\geq 0\\ \mu}}\sup_{\lambda \geq 0} -\alpha\lambda\frac{\kappa}{2}+\mathbb{E}\left[e_{\ell}(\alpha H+\mu YG;\frac{\alpha}{\lambda})\right]\,,
$$
 admits a unique minimizer $\alpha^\star$ and $\mu^\star$. 
 We proceed with the following change of variable   $\Upsilon:=\frac{\mu}{\alpha}$  and write $\overline{\phi}_{\mathcal{L}}$ as:
$$
\overline{\phi}_{\mathcal{L}}= \inf_{\alpha\geq 0}\min_{\Upsilon}  \alpha\left\{\sup_{\lambda\geq 0} -\lambda \frac{\kappa}{2}+\mathbb{E}\left[\min_{u}\frac{\lambda}{2}(H+\Upsilon \  YG-u)^2+\frac{1}{\alpha}\ell(u\alpha)\right]\right\}=:
\alpha\,\mathcal{D}(\frac{1}{\alpha},\Upsilon,\lambda)\,,
$$
where 
\bea\nn
\mathcal{D}(\alpha,\Upsilon,\lambda)&=-\lambda\frac{\kappa}{2}+\mathbb{E}\left[\min_u\frac{\lambda}{2}(H+\Upsilon \ YG-u)^2+\alpha\ell(\frac{u}{\alpha})\right]\\
&=-\lambda\frac{\kappa}{2}+\mathbb{E}\left[e_{\alpha\ell\big(\frac{\cdot}{\alpha}\big)}(H+\Upsilon \ YG;1/\lambda)\right]
\,.\label{eq:D_3}
\eea

We start by proving that for any $\alpha$ and $\Upsilon$, the optimum for $\lambda$ cannot be achieved as $\lambda\to 0^+$.

\noindent Indeed, denote by $\widetilde{\rm prox}:={\rm prox}_{u\mapsto \frac{1}{\alpha}\ell(u \alpha)} (H+\Upsilon \ YG;\frac{1}{\lambda})$ the proximal operator of the function $u\mapsto \frac{1}{\alpha}\ell(u \alpha)$. By optimality of the prox operator \cite{RocVar}:
\bea\label{eq:proxx}
\widetilde{\rm prox}=-\frac{1}{\lambda} \ell'(\alpha \widetilde{\rm prox})+ H+\Upsilon \ YG.
\eea
It can be easily checked that the function $k:x\mapsto -\frac{1}{\lambda}\ell'(\alpha x)+ H+\Upsilon \  YG$ is decreasing. Hence, $x\leq 0 \implies k(x)\geq k(0)$. Moreover, it is seen from \eqref{eq:proxx} that $\widetilde{\rm prox}$ is the solution of $k(x)=x$. Notice also that in the event that $|H|+|G \Upsilon|\leq \frac{1}{2\lambda}$, we have that $k(0)=\frac{1}{2\lambda}+ H+\Upsilon YG\geq 0$. These facts combined lead to the conclusion that:
\bea\label{eq:prox>0}
|H|+ |\Upsilon G|\leq \frac{1}{2\lambda} \implies \widetilde{\rm prox}\geq 0.
\eea

\noindent Using this observation, the derivative of $\mathcal{D}$ with respect to $\lambda$ can be lower-bounded as follows:
\begin{align}
	\frac{\partial \mathcal{D}}{\partial \lambda}&=-\frac{\kappa}{2}+\mathbb{E}\left[\frac{1}{2}\big(H+\Upsilon \  YG-\widetilde{\rm prox}\big)^2\right]\nn\\
	&	\geq -\frac{\kappa}{2}+\mathbb{E}\left[\frac{1}{2}(H+\Upsilon \ YG-\widetilde{\rm prox})^2 1_{\{|H|+|\Upsilon G|\leq \frac{1}{2\lambda}\}} \right]\nn\\
	&\geq  -\frac{\kappa}{2}+\mathbb{E}\left[\min_{u\geq 0}\frac{1}{2}(H+\Upsilon \ YG-u)^2 1_{\{|H|+|\Upsilon G|\leq \frac{1}{2\lambda}\}} \right]\label{eq:prox>0_use}\\
	&=-\frac{\kappa}{2}+\mathbb{E}\left[\frac{1}{2}(H+\Upsilon \ YG)_{-}^2 1_{\{|H|+|\Upsilon G|\leq \frac{1}{2\lambda}\}} \right]\label{eq:lastline}	\,.
\end{align}
In the first equality above we have used \cite[Prop.~A4]{Hossein2020} for the derivative of the expected Moreau envelope function in the definition of the function $\Dc$ in \eqref{eq:D_3}. The second inequality in \eqref{eq:prox>0_use} follows from \eqref{eq:prox>0}. Finally, the equality in \eqref{eq:lastline} is because $\min_{u\geq 0} (a-u)^2 = (a)_-^2$ for any $a\in\R$.

\noindent Continuing from \eqref{eq:lastline}, as $\lambda \to 0$, the event $1_{\{|H|+\Upsilon |G|\leq \frac{1}{2\lambda}\}}$ occurs with probability one. Hence, 
\bea
\frac{\partial \mathcal{D}}{\partial \lambda}\Big|_{\lambda\to 0}  \, \geq  -\frac{\kappa}{2} + \mathbb{E}\Big[\big(H+\Upsilon YG\big)_{-}^2 \Big] 
&\geq \frac{1}{2}\Big(\min_{t\in\R}\mathbb{E}\big(H+t YG\big)_{-}^2  - \kappa \Big)\nn\\
&\geq 
\Big(\frac{1}{2}\Big)\Big( g(\kappa)  - \kappa \Big)>0,\label{eq:kappalast}
\eea
where in \eqref{eq:kappalast} we first recall the definition of $g(\kappa)$ in \eqref{eq:threshold_func} and further apply Lemma \ref{lem:Ong}, i.e., $\kappa<\kappa_\star \implies \kappa < g(\kappa)$. 
This proves that the supremum is not attained in the limit $\lambda\to 0$.

\vp
Next, we prove that there exists a unique $\alpha^\star$ minimizing $\alpha\mapsto \alpha\min_{\Upsilon}\sup_{\lambda\geq 0}\mathcal{D}(\frac{1}{\alpha},\Upsilon,\lambda)$. For this, it will suffice to prove the following two statements:
\begin{flalign}
&\text{(i)~~The function $\mathcal{H}:\alpha\mapsto \alpha\min_{\Upsilon}\sup_{\lambda\geq 0}\mathcal{D}(\frac{1}{\alpha},\Upsilon,\lambda)$ is strictly convex in $\alpha$.}&\label{eq:strict_i}
\\
&\text{(ii)~~$\lim_{\alpha\to\infty} \mathcal{H}(\alpha)=\infty.$}&\label{eq:inf_ii}
\end{flalign}

\noindent\underline{Proof of \eqref{eq:strict_i}:} It will suffice to prove that the function 
\bea\label{eq:Gc}
\Gc(\alpha,\Upsilon):=\sup_{\lambda\geq 0} \mathcal{{D}}(\alpha,\Upsilon,\lambda)
\eea
is (jointly) strictly convex in the variables $(\alpha,\Upsilon)$. Indeed, this would imply that $\alpha \mapsto \inf_{\Upsilon} \sup_{\lambda\geq 0}\mathcal{{D}}(\alpha,\Upsilon,\lambda)$ is strictly convex. In its turn, this would mean that its perspective function $\alpha \mapsto\alpha \min_\Upsilon \sup_{\lambda\geq 0}\mathcal{{D}}(\frac{1}{\alpha},\Upsilon,\lambda)$ is also strictly convex. Thus, in what follows, we prove that \eqref{eq:Gc} is strictly convex. Fix any $\la>0$ (it suffices to consider strictly positive $\la$ since we have already shown that the supremum in \eqref{eq:Gc} is not attained at $\la\rightarrow0^+$). Let us define
\bea\label{eq:Gcla}
\Gc_\la(\alpha,\Upsilon):=\mathbb{E}\left[\min_{u} \frac{\la}{2}(H+\Upsilon \ YG- u)^2+ \alpha\ell\big(\frac{u}{\alpha}\big)\right].
\eea
It suffices to prove that \eqref{eq:Gcla} is jointly convex in $(\alpha,\Upsilon)$; then, $\Gc$ would be strictly convex as the pointwise supremum of strictly convex functions \cite[Sec.~3.2.3]{boyd2009convex}. Let $(\alpha_1,\Upsilon_1) \neq (\alpha_2,\Upsilon_2)$, $\theta\in(0,1)$ and $\thetao=1-\theta$. For convenience, define the proximal operators
\bea\label{eq:proxGcla}
p_i(H,GY):=\prox{\alpha_i \ell\big(\frac{\cdot}{\alpha_i}\big)}{H + \Upsilon_i G Y}{1/\la} = \arg\min_{u}\Big\{\frac{\la}{2}(H+\Upsilon_i YG - u)^2 + \alpha_i\ell\Big(\frac{u}{\alpha_i}\Big)\Big\},~~i=1,2.
\eea
Finally, denote $\alpha_\theta:=\theta\alpha_1+\thetao\alpha_2$ and $\Upsilon_\theta:=\theta\Upsilon_1+\thetao\Upsilon_2$. With this notation, we have the following chain of inequalities,
\bea
&\Gc_\la(\alpha_\theta,\Upsilon_\theta) \leq \E\left[\, \frac{\la}{2}\left( H + \Upsilon_\theta YG - \Big( \theta p_1(H,GY) + \thetao p_2(H,GY) \Big) \right)^2 + {\alpha_\theta}\ell\left(\frac{\theta p_1(H,GY) + \thetao p_2(H,GY)}{\alpha_\theta}\right)\,\right] \nn\\
 &\leq \E\left[\,  \frac{\theta\la}{2}\left( H + \Upsilon_1 YG - p_1(H,GY) \right)^2 +  \frac{\thetao\la}{2}\left( H + \Upsilon_2 YG - p_2(H,GY) \right)^2 + {\alpha_\theta}\ell\left(\frac{\theta p_1(H,GY) + \thetao p_2(H,GY)}{\alpha_\theta}\right)\,\right] \label{eq:str001}\\
  &{=} \E\Big[\, \theta\cdot \frac{\la}{2}\left( H + \Upsilon_1 YG - p_1(H,GY) \right)^2 + \thetao \cdot \frac{\la}{2}\left( H + \Upsilon_2 YG - p_2(H,GY) \right)^2\Big] \nn\\&\qquad\qquad\qquad\qquad\qquad\qquad\qquad\qquad\qquad\qquad+ {\alpha_\theta} \E\Big[\ell\left( \frac{\theta\alpha_1}{\alpha_\theta}\,\frac{p_1(H,GY)}{\alpha_1} + \frac{\thetao\alpha_2}{\alpha_\theta}\,\frac{p_2(H,GY)}{\alpha_2}\right)\,\Big] \nn\\
    &< \E\Big[\, \theta\cdot \frac{\la}{2}\left( H + \Upsilon_1 YG - p_1(H,GY) \right)^2 + \thetao \cdot \frac{\la}{2}\left( H + \Upsilon_2 YG - p_2(H,GY) \right)^2\Big] \nn\\&\qquad\qquad\qquad\qquad\qquad\qquad\qquad\qquad\qquad\qquad+ \theta\alpha_1\E\Big[\ell\left(\frac{p_1(H,GY)}{\alpha_1}\right)\Big] + \thetao\alpha_2\E\Big[\ell\left(\frac{p_2(H,GY)}{\alpha_2}\right)\,\Big] \label{eq:str003}\\
    & = \theta \E\left[\frac{\la}{2}\big(H+\Upsilon_1 \ YG- p_1(H,GY)\big)^2+ \alpha_1\ell\Big(\frac{p_1(H,GY)}{\alpha_1}\Big)\right] \nn\\ &\qquad\qquad\qquad\quad\quad\qquad\qquad\qquad\qquad\qquad+ \thetao \E\left[\frac{\la}{2}\big(H+\Upsilon_2 \ YG- p_2(H,GY)\big)^2+ \alpha_2\ell\Big(\frac{p_2(H,GY)}{\alpha_2}\Big)\right]\nn\\
    & = \theta \Gc_\la(\alpha_1,\Upsilon_1) + \thetao \Gc_\la(\alpha_2,\Upsilon_2).\label{eq:str004}
\eea
Inequality \eqref{eq:str001} follows by convexity of the function $(x,u)\mapsto(H+xYG-u)^2$ for every $H,G,Y$. The strict inequality in \eqref{eq:str003} follows by Lemma \ref{lem:prox_diff}, which   uses strict convexity of $\ell$ and the fact that $p_1(H,GY)/\alpha_1 \neq p_2(H,GY)/\alpha_2$ for a set of non-zero measure over the distribution of the pair $H,GY$. In the last line \eqref{eq:str004} we have recalled the definition of the function $\Gc_\la$ and of the proximal operator in \eqref{eq:proxGcla}. This completes the proof of strict convexity of $\Gc_\la$, as desired.


\vp
\noindent\underline{Proof of \eqref{eq:inf_ii}:} We will now prove that as $\alpha\to\infty$, $\mathcal{H}(\alpha)\to\infty$. To see this, we will again invoke the fact that if $|H|+|\Upsilon G|\leq \frac{1}{2\lambda}$, then $\widetilde{\rm prox}\geq 0$ (see \eqref{eq:prox>0}). With this at hand, we lower bound $\mathcal{H}$ as follows:
\begin{align}
	\mathcal{H}(\alpha)&\geq \alpha\min_{\Upsilon}\left(-\frac{\kappa}{2\sqrt{\alpha}}+\mathbb{E}\left[\min_{u} \left(\frac{1}{2\sqrt{\alpha}}(H+\Upsilon \ YG-u)^2 + \frac{1}{\alpha}\ell(u\alpha)\right){\bf 1}_{\{|H|+  |G \Upsilon|\leq \frac{\sqrt{\alpha}}{2}\}}\right]\right)\nn\\
	&	\geq \alpha\min_{\Upsilon}\left(-\frac{\kappa}{2\sqrt{\alpha}}+\mathbb{E}\left[\min_{u\geq 0} \left(\frac{1}{2\sqrt{\alpha}}( H+\Upsilon \ YG-u)^2 + \frac{1}{\alpha}\ell(u\alpha)\right){\bf 1}_{\{|H|+|\Upsilon G|\leq \frac{\sqrt{\alpha}}{2}\}}\right]\right)\nn\\
	&\geq \sqrt{\alpha} \,\Big( -\frac{1}{2}\kappa + \mathbb{E}\Big[\min_{\Upsilon}\frac{1}{2}(H+\Upsilon YG)_{-}^2\ \  1_{|H|+|\Upsilon G|\leq \frac{\sqrt{\alpha}}{2}}\Big] \Big).\label{eq:Halinf}
\end{align}
The first inequality follows by setting $\la=\frac{1}{\sqrt{\alpha}}$. The second inequality uses \eqref{eq:prox>0}. The third inequality follows since $\ell(x)\geq 0,\forall x\in\R$ and $\min_{u\geq 0} (a-u)^2 = (a)_-^2$ for any $a\in\R$. To continue, note that if $\kappa< \kappa^\star$, the right-hand side of \eqref{eq:Halinf} tends to infinity as $\alpha\to\infty$; see \eqref{eq:kappalast}. 
Hence, the function $\mathcal{H}$ has a unique minimizer which we denote by $\alpha^\star$. 

\vp
We conclude the proof of uniqueness by showing that the function $\Upsilon\mapsto \sup_{\lambda\geq 0}\mathcal{D}(\frac{1}{\alpha^\star},\Upsilon,\lambda)$ has a unique minimizer. Since, the function is strictly convex in $\Upsilon$, it suffices to show that it is level-bounded, i.e., that  $$\lim_{|\Upsilon|\to\infty} \sup_{\lambda\geq 0}\mathcal{D}(\frac{1}{\alpha^\star},\Upsilon,\lambda)=\infty.$$ For that, we lower bound it as follows:
\begin{align}
	\sup_{\lambda>0} \mathcal{D}(\frac{1}{\alpha^\star},\Upsilon,\lambda)&\geq \alpha^\star\left(-\frac{\kappa}{2|\Upsilon|^{\frac{3}{2}}}+\mathbb{E}\left[\min_u \left(\frac{1}{2|\Upsilon|^{\frac{3}{2}}}\left(H + \Upsilon \ YG-u\right)^2+\frac{1}{\alpha^\star}\ell(u\alpha^\star)\right){\bf 1}_{\{|H|+|\Upsilon G|\leq \frac{1}{2}|\Upsilon|^{\frac{3}{2}}\}}\right]\right)\nn\\
	&\geq \alpha^\star\left(-\frac{\kappa}{2|\Upsilon|^{\frac{3}{2}}}+\mathbb{E}\left[\frac{1}{2|\Upsilon|^{\frac{3}{2}}}(H+\Upsilon \ YG)_{-}^2\right]\right) \label{eq:s}
\end{align}
where  the first inequality is obtained by setting $\lambda=\frac{1}{|\Upsilon|^{\frac{3}{2}}}$ and the second follows from \eqref{eq:prox>0}. Given that $\alpha^\star\neq 0$,   the limit of the right-hand side of   \eqref{eq:s} tends to infinity as $|\Upsilon|\to\infty$ which yields the desired result.

\subsection{Technical lemmata for uniqueness}
Lemmata \ref{lem:prox_diff} and \ref{lem:key_prox} are useful in proving uniqueness of the minimizer of $\overline{\phi}_{\mathcal{L}}$. They are adapted with small modifications from \cite{Hossein2020}. Specifically, see \cite[Sec.~A.6.2]{Hossein2020}.
\begin{lem}\label{lem:prox_diff} Fix arbitrary pairs $\vb_i=(\alpha_i,\Upsilon_i),~i=1,2$ such that $\vb_1\neq \vb_2$ and let random variables $H$ and $X=GY$ defined as in \eqref{eq:HGZ}. Further denote
\bea\label{eq:prox_short}
\proxri{i}{H}{X} := \prox{\alpha_i\ell\big(\frac{\cdot}{\alpha_i}\big)}{H+ \Upsilon_i X}{\la},~~~i=1,2.
\eea
Then, there exists a ball $\Sc\subset\R^2$ of nonzero measure, i.e. $\Pro\left( (H,X)\in\Sc \right)>0$, such that 
$\proxri{1}{h}{x}\neq \proxri{2}{h}{x}$, for all $(h,x)\in\Sc$. Consequently,  for any $\theta\in(0,1)$ and $\thetao=1-\theta$, the following strict inequality holds,
\bea\label{eq:ell_prox_strict}
\E\Big[ \ell\left( \theta \frac{\proxri{1}{H}{X}}{\alpha_1} + \thetao \frac{\proxri{2}{H}{X}}{\alpha_2} \right) \Big] < \theta\, \E\Big[ \ell\left( \frac{\proxri{1}{H}{X}}{\alpha_1} \right) \Big] +  \thetao\, \E\Big[ \ell\left( \frac{\proxri{2}{H}{X}}{\alpha_2} \right) \Big].
\eea
\end{lem}
\begin{proof}
Note that \eqref{eq:ell_prox_strict} holds trivially with ``$<"$ replaced by ``$\leq"$ due to the convexity of $\ell$. To prove that the inequality is strict, it suffices, by strict convexity of $\ell$, that there exists subset $\Sc\subset\R^2$ that satisfies the following two properties:
\begin{enumerate}
\item $\proxri{1}{h}{x}\neq \proxri{2}{h}{x}$, for all $(h,x)\in\Sc$.
\item $\Pro\left( (H,X)\in\Sc \right)>0$.
\end{enumerate}
Consider the following function $f:\R^2\rightarrow\R$: 
\bea
f(h,x):= \frac{\proxri{1}{h}{x}}{\alpha_1}- \frac{\proxri{2}{h}{x}}{\alpha_2}.
\eea
By lemma \ref{lem:key_prox}, there exists $(h_0,x_0)$ such that 
\bea\label{eq:x0y0}
 f(h_0,x_0) \neq 0.
 \eea 
Moreover, by continuity of the proximal operator (cf. \cite[Prop.~A1(a)]{Hossein2020}), it follows that $f$ is continuous. From this and \eqref{eq:x0y0}, we conclude that for sufficiently small $\zeta>0$ there exists a $\zeta$-ball $\Sc$ centered at $(h_0,x_0)$, such that property 1 holds. Property $2$ is also guaranteed to hold for $\Sc$, since both $H,X$ have strictly positive densities and are independent.
\end{proof}

\begin{lem}\label{lem:key_prox}
Let $\al_1,\al_2>0$ and $\la>0$. Then, the following statement is true
\bea
(\alpha_1,\Upsilon_1) \neq (\alpha_2,\Upsilon_2) \quad\Longrightarrow\quad \exists (h,x)\in\R^2: \alpha_2\cdot\prox{\alpha_1\ell\big(\frac{\cdot}{\alpha_1}\big)}{ h + \Upsilon_1 x}{\la} \neq \alpha_1\cdot\prox{\alpha_2\ell\big(\frac{\cdot}{\alpha_2}\big)}{ h + \Upsilon_2 x}{\la}.
\eea
\end{lem}
\begin{proof}
We prove the claim by contradiction, but first, let us set up some useful notation. Define
$$
\proxri{\al,\Upsilon}{h}{x} := \prox{\alpha\ell\big(\frac{\cdot}{\alpha}\big)}{ h + \Upsilon x}{\la},
$$
and 
$$
\ellprox{\al,\Upsilon}{h}{x} := \ellp\left(\frac{\prox{\alpha\ell\big(\frac{\cdot}{\alpha}\big)}{h+ \Upsilon z}{\la}}{\alpha}\right).
$$
By optimality of the proximal operator (cf. \cite[Prop.~A.1(c)]{Hossein2020}), the following is true:
\bea\label{eq:L_eqv}
\ellprox{\al,\Upsilon}{h}{x}= \frac{1}{\la}\left( h + \Upsilon x - \proxri{\al,\Upsilon}{h}{x} \right).
\eea
For the sake of contradiction, assume that the claim of the lemma is false. Then, 
\bea\label{eq:prox_eq}
\frac{\proxri{\al_1,\Upsilon_1}{h}{x}}{\al_1} =  \frac{\proxri{\al_2,\Upsilon_2}{h}{x}}{\al_2}, \quad\forall (h,x)\in\R^2.
\eea
From this, it also holds that
\bea\label{eq:der_eq}
\ellprox{\al_1,\Upsilon_1}{h}{x} =  \ellprox{\al_2,\Upsilon_2}{h}{x},\quad\forall (h,x)\in\R^2.
\eea
Recalling \eqref{eq:L_eqv} and applying \eqref{eq:prox_eq}, we derive the following from \eqref{eq:der_eq}:
\bea\label{eq:use1}
\Big(\frac{\alpha_2-\alpha_1}{\alpha_1}\Big)\,\proxri{\al_1,\Upsilon_1}{h}{x} = (\Upsilon_2-\Upsilon_1) x ,\quad\forall(h,x)\in\R^2.
\eea

We consider the following two cases separately.

\vp
\noindent\underline{Case 1:~$\al_1=\al_2$}\,:  From \eqref{eq:use1}, it would then follow that $\Upsilon_1=\Upsilon_2$, which is a contradiction to the assumption $(\al_1,\Upsilon_1)\neq (\al_2,\Upsilon_2)$ and completes the proof for this case.
\\

\noindent\underline{Case 2:~$\al_1\neq\al_2$}\,: Continuing from \eqref{eq:use1} we can compute that for all $(x,z)\in\R^2$
\bea\label{eq:use2}
\la\cdot\ellp\Big(\frac{\proxri{\al_1,\Upsilon_1}{h}{x}}{\al_1}\Big) &=   h+\Upsilon_1 x-\proxri{\al_1,\Upsilon_1}{h}{x}\nn \\
&= h  + \frac{\alpha_2\Upsilon_1-\alpha_1\Upsilon_2}{\alpha_2-\alpha_1} x.
\eea
By replacing $\proxri{\al_1,\Upsilon_1}{h}{x}$ from \eqref{eq:use1} we derive that:
\bea\label{eq:use3}
\ellp\Big( \frac{\Upsilon_2-\Upsilon_1}{\alpha_2/\alpha_1 - 1} x\Big) =  h +  \frac{\alpha_2\Upsilon_1-\alpha_1\Upsilon_2}{\alpha_2-\alpha_1} x,\quad\forall (h,x)\in\R^2.
\eea
By replacing $h=x=0$ in \eqref{eq:use3} we find that $\ell '(0) = 0$. This contradicts the assumption of the lemma and completes the proof.
\end{proof}

\section{Useful technical lemmata}\label{app:tech}

\begin{lem}[Lemma 8 from \cite{svm_abla}]\label{lem:Abla_help1}
Let $d_1$ and $d_2$ be two strictly positive integers. Let $X\times Y$ be two non-empty sets in $\mathbb{R}_{d_1}\times \mathbb{R}_{d_2}$. Let $F:X\times Y\to \mathbb{R}$ be a given real-valued  function. Assume there exists $\tilde{X}\subset X$ such that for all ${\bf x}\in X$ there exists $\tilde{\bf x}\in \tilde{X}$ such that:
\begin{equation}
\forall {\bf y}\in Y, \ \ F({\bf x},{\bf y})\geq F(\tilde{\bf x},{\bf y}).
\label{eq:prop}
\end{equation}
Then
$$
\min_{{\bf x}\in X} \max_{{\bf y}\in Y} F({\bf x},{\bf y}) = \min_{\tilde{\bf x}\in \tilde{X}} \max_{{\bf y}\in Y} F(\tilde{\bf x},{\bf y})
$$
\label{lem:prop}
In particular, if $\tilde{\bf X}=\{\tilde{\bf x}\}$, then:
$$
\min_{{\bf x}\in X} \max_{{\bf y}\in Y} F({\bf x},{\bf y}) = \max_{{\bf y}\in Y} F(\tilde{\bf x},{\bf y})
$$
\end{lem}

\begin{lem}[Lemma 9 from \cite{svm_abla}]\label{lem:help2}
Let ${d}\in\mathbb{N}^\star$. Let ${S}_x$ be a compact non-empty set in $\mathbb{R}^{d}$. Let $f$ and $c$ be two continuous functions over $S_x$ such that the set $\left\{c(x)\leq 0\right\}$ is non-empty. Then:
$$
\min_{\substack{{\bf x}\in{S}_x \\ c({\bf x})\leq 0}}  f({\bf x})\\
=\sup_{\delta \geq 0} \min_{\substack{{\bf x}\in{S}_x \\ c({\bf x})\leq \delta}}  f({\bf x})= \inf_{\delta>0} \min_{\substack{{\bf x}\in{S}_x \\ c({\bf x})\leq -\delta}}  f({\bf x}) 
$$ 
\end{lem}

\begin{lem}[Lemma 5 in \cite{Master}]
	\label{lemma:5}
	Consider the  primary optimization problem in \eqref{eq:pr} 
	\begin{equation}
		\hat{\bf w}:=\arg\min_{{{\bf w}\in\mathcal{S}_{{\bf w}}}} \sup_{{\bf u}\in \mathcal{S}_{\bf u}} X_{{\bf w},{\bf u}}\,, \label{eq:or}
	\end{equation}
and its bounded version
	\begin{equation}
	\hat{\bf w}_{\mathcal{B}}:=\arg\min_{\substack{{\bf w }\in\mathcal{S}_{{\bf w}}\\ \|{\bf w}\|\leq \mathcal{B}}} \sup_{{\bf u}\in \mathcal{S}_{\bf u}} X_{{\bf w},{\bf u}} \,.
		\label{eq:bounded_pr}
	\end{equation}
	Assume that for any optimal solution $\hat{\bf w}_{\mathcal{B}}$ of \eqref{eq:bounded_pr}, it holds $\|\hat{\bf w}_{\mathcal{B}}\| \ras \theta^\star$ for some $\theta^\star>0$. Then any minimizer of \eqref{eq:or} satisfies $\|\hat{\bf w}\|\ras \theta^\star$.  
\end{lem}

\section{Additional numerical results}\label{sec:num_more}

\begin{figure*}[h!]
\begin{center}
\includegraphics[width=.5\textwidth]{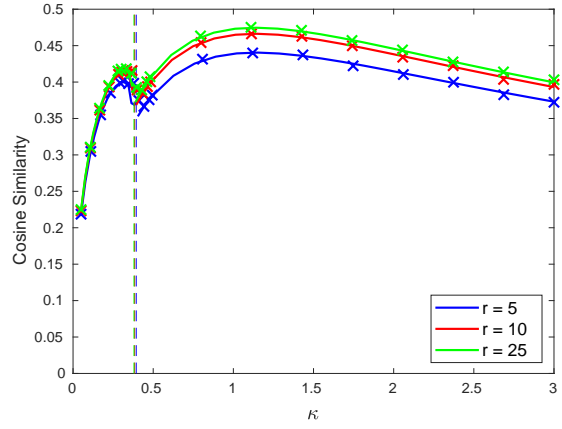}
\end{center}
\caption{Cosine similarity 
as a function of the overparameterization ratio $\kappa$ for binary logistic regression under the polynomial feature selection rule. The setting here is the same as in Figure \ref{fig:exp4str}. See also Section \ref{sec:curves}}
\label{fig:app1}
\end{figure*}


\begin{figure}[h!]
	\begin{center}
		\includegraphics[width=.49\textwidth]{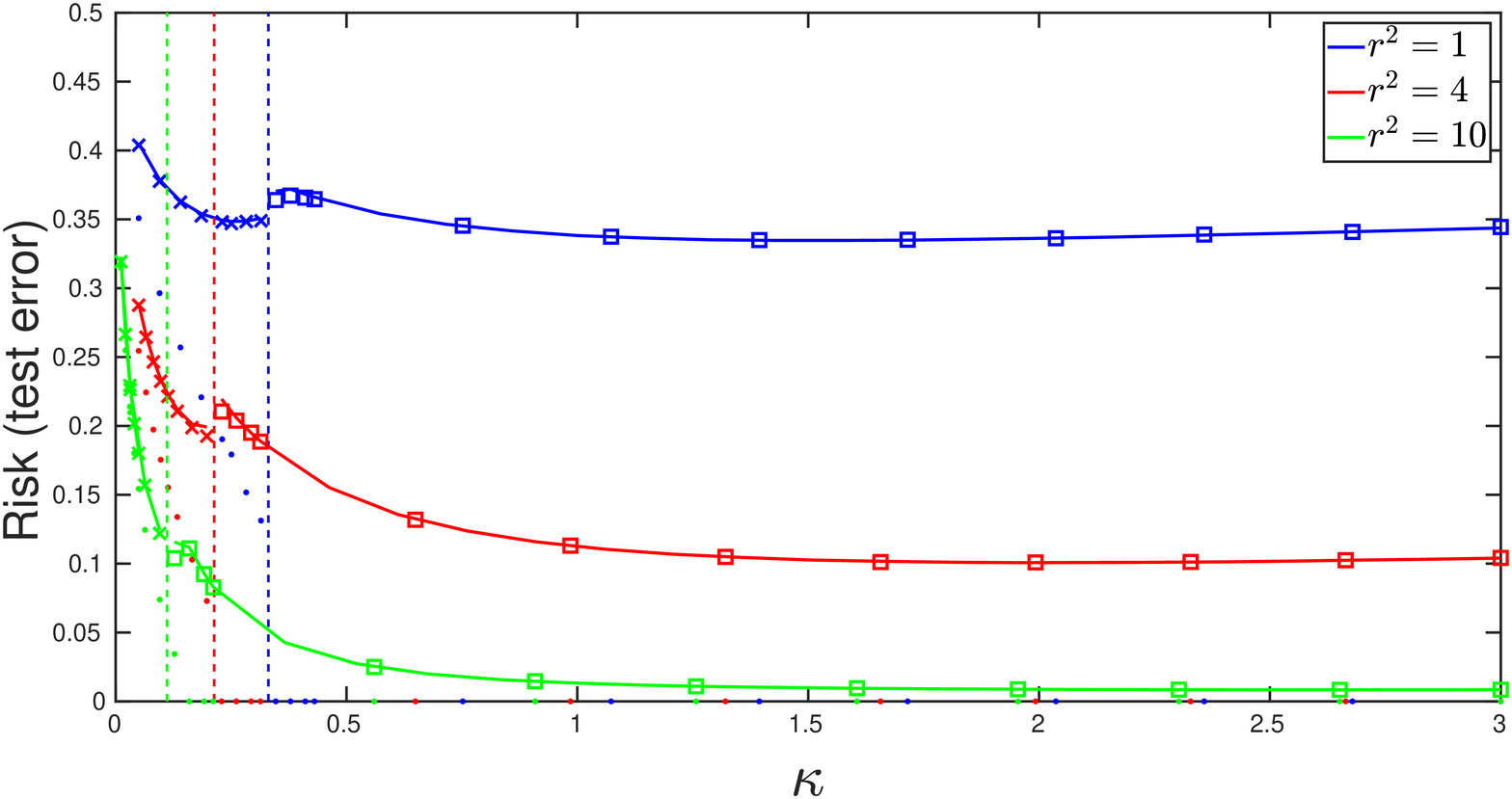}
		\includegraphics[width=.49\textwidth]{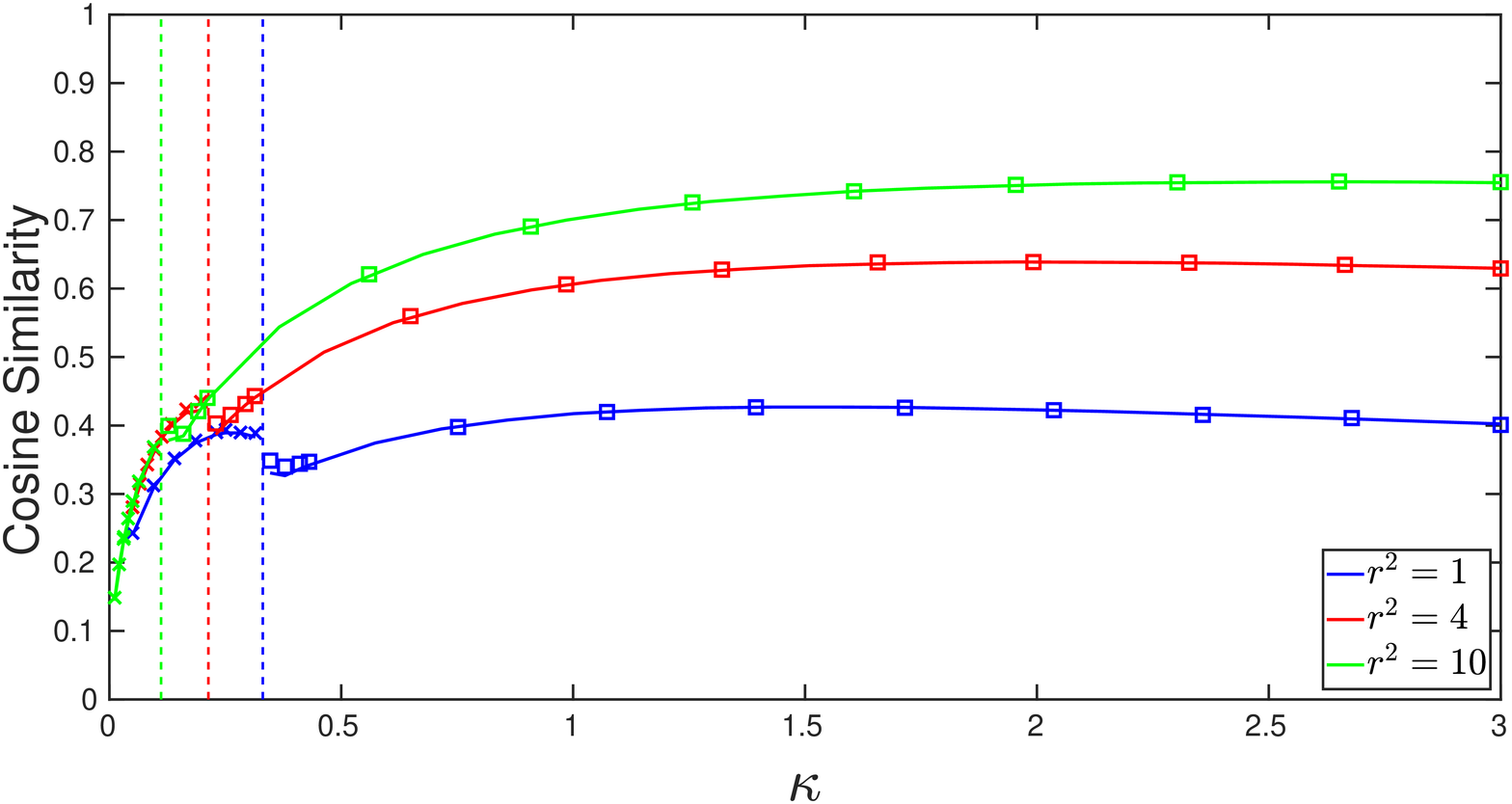}
	\end{center}
	\caption{Cosine similarity (Left) and risk (Right) curves for GM exponential model with $\gamma = 2$.
	}\label{fig:GMexp4r}
\end{figure}

\begin{figure}[h!]
	\begin{center}
		\includegraphics[width=1\textwidth]{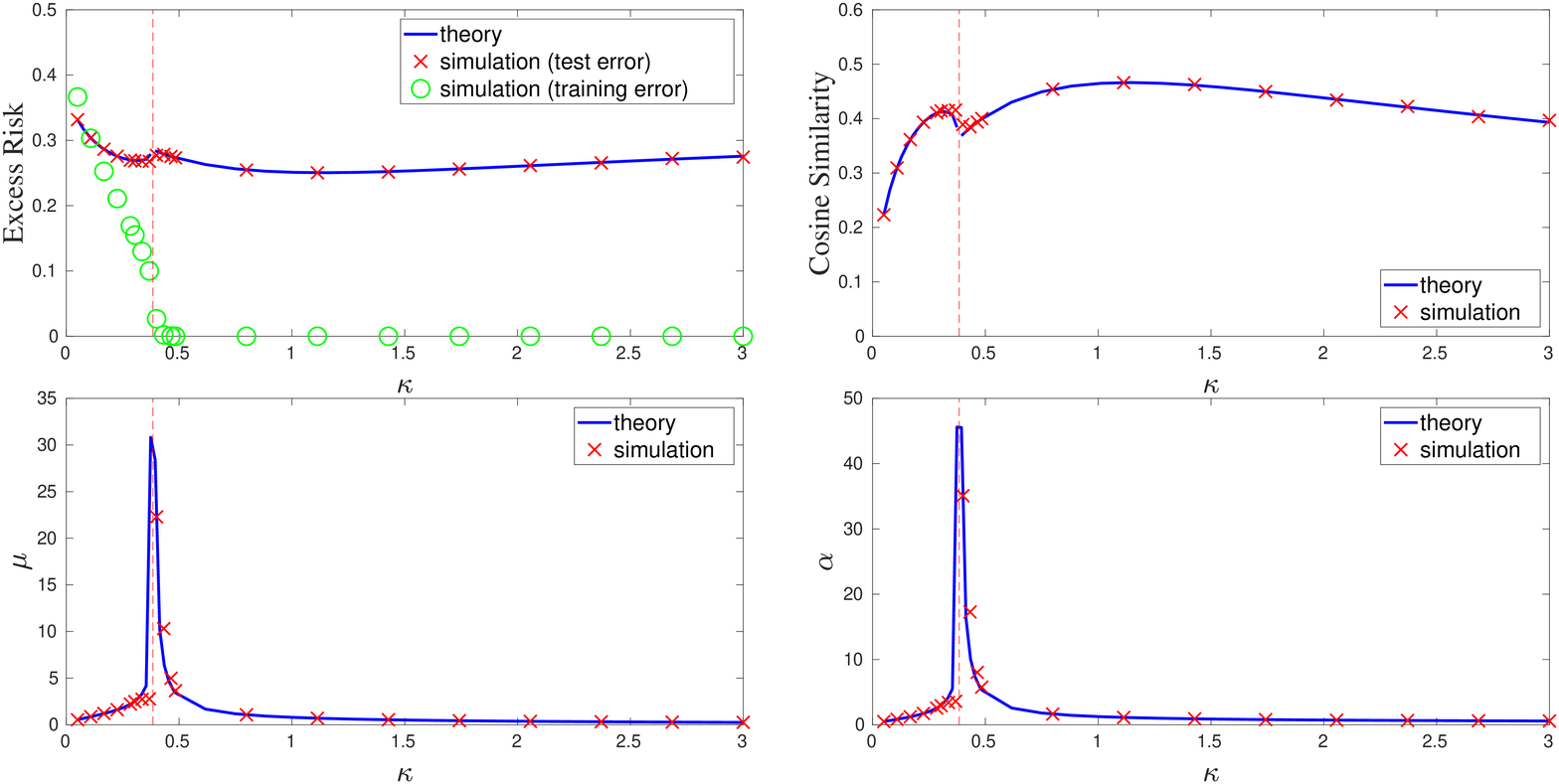}
	\end{center}
	\caption{Simulation results ($\times$) vs theoretical predictions (---) for the logistic model for signal strength $r = 10$. Without loss of generality, we set $\betab_0=[s,0,\ldots,0]$ and also decompose $\betabh = [\betah_1, \widetilde{\betab}_*^T]^T$. \textbf{(Upper left)} Absolute risk $\Rc(\betabh)$; \textbf{(Upper right)} cosine similarity $\Cc(\betabh)$; \textbf{(Lower left)}: solution $\mu$ to \eqref{eq:ML_gen}  (---) and averages of $\betah_1$ ($\times$); \textbf{(Upper right)} solution $\alpha$ to \eqref{eq:ML_gen}  (---) and averages of $\|\widetilde{\betab}_*\|_2$ ($\times$). The bottom two figures experimentally validate \eqref{eq:rP_T}. }\label{fig:exp4r10}
\end{figure}

\begin{figure}[h!]
	\begin{center}
		\includegraphics[width=1\textwidth]{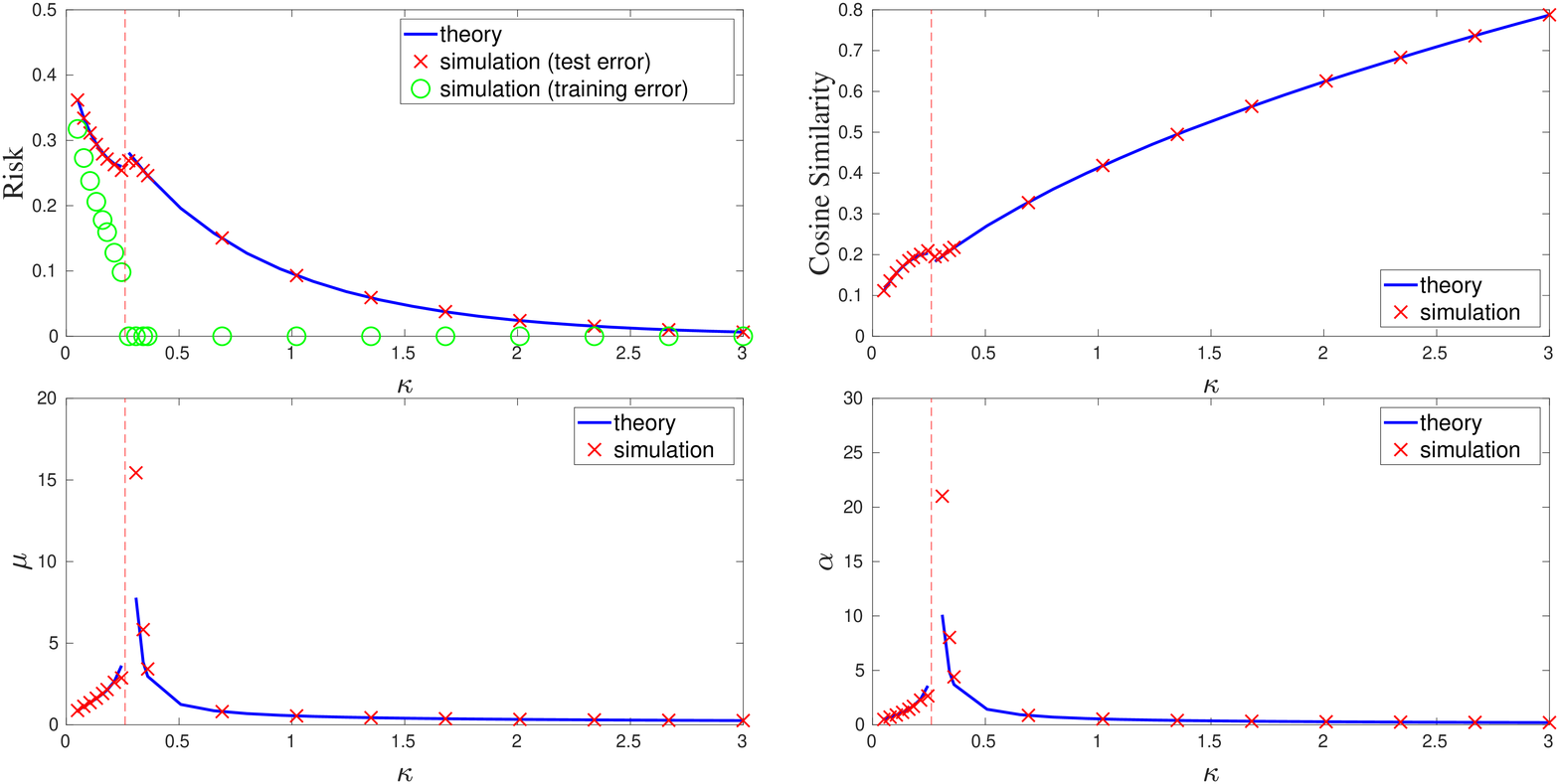}
	\end{center}
	\caption{Simulation results ($\times$) vs theoretical predictions (---) for the GM model for signal strength $r^2 = 10$. Without loss of generality, we set $\betab_0=[s,0,\ldots,0]$ and also decompose $\betabh = [\betah_1, \widetilde{\betab}_*^T]^T$. \textbf{(Upper left)} Absolute risk $\Rc(\betabh)$; \textbf{(Upper right)} cosine similarity $\Cc(\betabh)$; \textbf{(Lower left)}: solution $\mu$ to \eqref{eq:ML_GM}  (---) and averages of $\betah_1$ ($\times$); \textbf{(Upper right)} solution $\alpha$ to \eqref{eq:ML_GM}  (---) and averages of $\|\widetilde{\betab}_*\|_2$ ($\times$). 
	}\label{fig:poly4r10}
\end{figure}

%
%

\begin{figure}[h!]
	\begin{center}
		\includegraphics[width=1\textwidth]{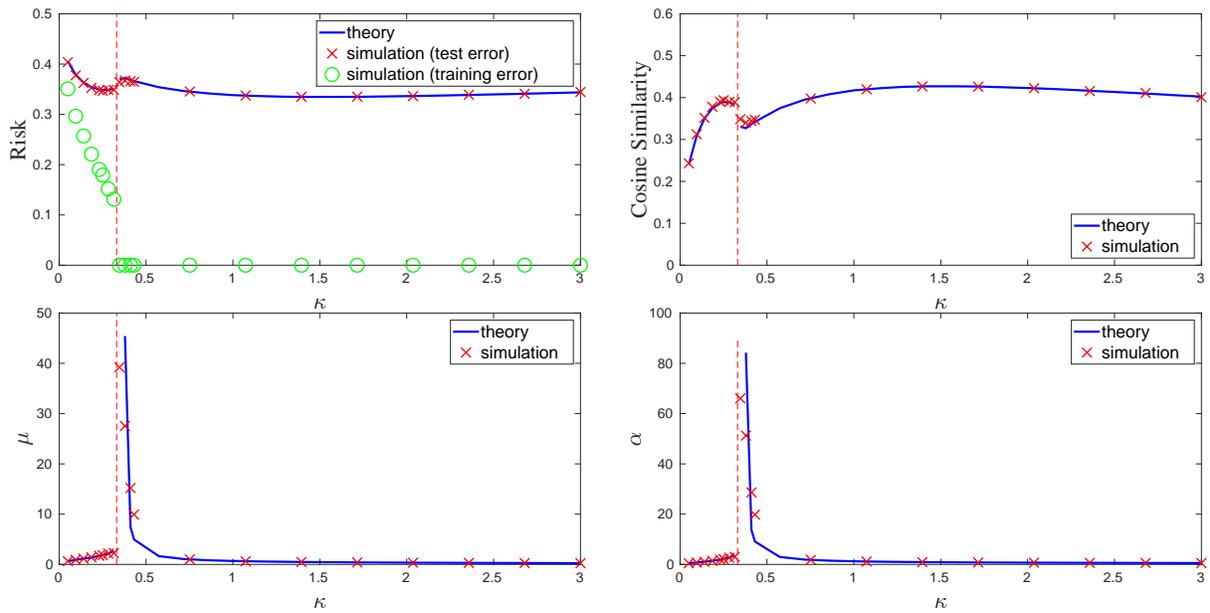}
	\end{center}
	\caption{Detailed simulation results for GM exponential model where $r = 1$.
	}\label{fig:GMexpr1}
\end{figure}

%

%
%
%
%
%

\end{document}